\documentclass[11pt]{article}

\usepackage[T1]{fontenc}
\usepackage{lmodern}
\usepackage{microtype}

\usepackage{amsmath,amssymb,amsthm,mathtools}

\usepackage{booktabs,array,longtable,tabularx}

\usepackage{graphicx}
\usepackage{float}
\usepackage{placeins}
\usepackage{needspace}
\usepackage{etoolbox}

\usepackage{authblk}
\usepackage{xcolor}
\usepackage{caption}
\usepackage{enumitem}

\usepackage{natbib}

\usepackage[
    letterpaper,
    top=2cm,
    bottom=2cm,
    left=3cm,
    right=3cm
]{geometry}

\usepackage[colorlinks=true,allcolors=blue]{hyperref}

\newtheorem{theorem}{Theorem}[section]

\newtheorem{lemma}[theorem]{Lemma}
\newtheorem{proposition}[theorem]{Proposition}
\newtheorem{definition}{Definition}[section]
\newtheorem{assumption}{Assumption}[section]

\numberwithin{equation}{section}

\providecommand{\TPR}{\operatorname{TPR}}

\providecommand{\AUC}{\operatorname{AUROC}}

\BeforeBeginEnvironment{longtable}{\Needspace{10\baselineskip}}

\title{Robust Detection of LLM-Generated Text under Contamination}
\author{Jiaxun Li, Saptarshi Chakraborty, Ambuj Tewari}
\affil{Department of Statistics, University of Michigan}
\affil{\texttt{\{jasonli, saptarsc, tewaria\}@umich.edu}}

\hypersetup{
  pdfauthor={},
  pdftitle={Robust Detection of Human-LLM Mixed Text}
}

\begin{document}
\maketitle

\begin{abstract}
We study the detection of LLM-generated text under editing and
contamination. Modeling human and machine text as finite-order
Markov processes with Huber contamination, we characterize an exact
boundary for reliable detection under our assumptions. Detection
is impossible when contamination is sufficiently large relative
to clean-source separation. Below this boundary, a collection of
clipped likelihood-ratio tests achieves vanishing worst-case errors.
This construction motivates clipping as a simple modification of
existing statistical detectors. For a broad class of additive scores,
we identify conditions under which the clipped test is consistent
while the raw test's worst-case power tends to zero. We evaluate
seven detectors across three datasets and three generation models,
and on the RAID benchmark. Clipping improves robustness in both
studies, with gains varying across detectors and contamination
settings. For example, at a target false-positive rate of 5\%,
clipping improves the log-likelihood--log-rank ratio (LRR) detector's
true-positive rate by a median of 8.3 percentage points in the
controlled study and 2.1 and 4.3 points in rate- and attack-specific
RAID evaluations, respectively.
\end{abstract}

\section{Introduction}
\label{sec:introduction}

Large language models (LLMs) produce increasingly fluent text, making it difficult to distinguish machine-generated text from human writing. Reliable detection matters
for content provenance and academic integrity. Yet text encountered
in practice may contain both machine-generated content and human
contributions. Editing and other modifications can substantially
weaken detectors that perform well on clean machine text
\citep{wu-etal-2025-survey}. This raises two questions: when does
detection remain possible under contamination, and how can existing
detectors be made more robust?

We model human and machine text as finite-order Markov processes
\citep{shannon1948mathematical,bengio2003neural,zekri2024large}.
A class of human transition kernels captures variation across
writers, while a fixed kernel represents the target LLM.
We model modifications through Huber contamination \citep{huber}.
At each position, the token distribution mixes the target LLM
distribution with an arbitrary distribution. This allows contamination to adapt as the
text develops.

Under our assumptions, we characterize the exact boundary between
possible and impossible detection. The boundary compares the
contamination level with a measure of clean-source separation.
When contamination is sufficiently strong, the contaminated machine
can reproduce an admissible human source exactly. No detector can
then reliably distinguish the two. Otherwise, we construct a test combining lower-clipped log-likelihood ratios for representative human sources. Both its worst-case false-positive and false-negative probabilities vanish as text length grows.

This construction motivates improving existing statistical
detectors through clipping \citep{huber,huber1973minimax}.
Many detectors aggregate token-level scores into document-level.
A few modified tokens with extreme scores can offset
the evidence from the rest of the text. We apply clipping
to limit these individual contributions. For a general class of additive scores, we identify conditions in which both worst-case errors of the clipped test vanish, while the raw test’s worst-case power tends to zero whenever its worst-case false-positive probability vanishes.

We evaluate seven statistical detectors in two complementary studies.
The controlled study crosses three datasets with three generators
and replaces machine-generated passages with human text.
At 20\% replacement and a target false-positive rate of 5\%,
clipping improves mean detection rates for all seven methods under
both random and tail replacement. For example, Binoculars improves by 27 percentage points under random replacement.
The RAID benchmark \citep{dugan2024raid} evaluates a broader range
of attacks across multiple domains and generators.
We study clipping tailored to the modification rate or attack type. For example, 
at a target false-positive rate of 5\%, attack-specific clipping
increases rank's detection rate by 30 percentage points under case
changes and LRR's by 10 points under paraphrasing.
The gains vary across detectors and attacks. These experiments extend beyond the Markov and Huber contamination
models used in our theory. The observed gains provide complementary
evidence that clipping can improve robustness under more varied
text modifications.

\subsection{Related Work}
\label{subsec:related-work}

\paragraph{Detection methods and clipping.}
LLM detectors include trained classifiers
\citep{ippolito-etal-2020-automatic}, watermark-based methods
\citep{Li_2025}, and statistical methods.
The latter use token probabilities and ranks
\citep{gehrmann-etal-2019-gltr,su-etal-2023-detectllm},
probability curvature
\citep{mitchell2023detectgptzeroshotmachinegeneratedtext,bao2024fastdetectgpt},
or comparisons between models \citep{hans2024binoculars}.
Closest to our approach, \citet{wong2026knnproxy} clip token
log-likelihoods to reduce outliers caused by proxy-model mismatch.
We establish contamination-dependent power guarantees for a general
class of additive scores, including conditions under which clipping
succeeds while the raw test fails.
\paragraph{Editing and robustness.}
Prior studies evaluate editing, paraphrasing, and other attacks
\citep{wang-etal-2024-stumbling}, mixed human--machine authorship
\citep{zhang-etal-2024-llm}, and robustness across generators and
domains \citep{dugan2024raid}.
\citet{li2026robustwatermarks} model human edits through mixtures
and develop a truncated test for watermark detection.
Our analysis instead concerns contamination of the text source
without an embedded watermark. Controlled replacement and RAID
experiments assess clipping beyond the theoretical setting.

\paragraph{Detection limits and robust testing.}
Detection limits have been studied through source separation \citep{sadasivan2025aigeneratedtextreliablydetected} and text length \citep{chakraborty2024possibilities}. These studies connect detection difficulty to both distributional similarity and the amount of observed text. We characterize the detection boundary under contamination and a composite human model. Classical robust testing motivates clipping through least-favorable distributions and censored likelihood ratios \citep{huber,huber1973minimax}; recent work extends robust testing to composite hypotheses and adaptive contamination, with validity at data-dependent stopping times \citep{saha2025huberrobustlikelihoodratiotests}. Our contribution is to extend these ideas to Markov text sources and general additive scores, whose robustness is not implied by the classical optimality of clipped likelihood-ratio tests.

\section{Preliminaries}
\label{sec:preliminaries}

\subsection{Source and contamination model}
\label{subsec:markov-sources-contamination}

We model clean human and LLM-generated text as finite-order Markov processes.
Our results apply to any fixed order \(K\geq1\) and no
particular numerical scale for \(K\) is required. Finite-context Markov
sources are classical in language modeling and have also been used to analyze
autoregressive LLM generation
\citep{shannon1948mathematical,bengio2003neural,zekri2024large}.

Let \(\mathcal X\) be a finite vocabulary and let
\(\mathcal S:=\mathcal X^K\). We call
\(S_i:=(X_{i-K+1},\ldots,X_i)\in\mathcal S\) the history state at position
\(i\). Under the order-\(K\) model, this state contains all information from
the preceding text needed to determine the next-token distribution. The
prompt fixes the initial history state \(S_0\). For
\(s=(x_1,\ldots,x_K)\), write
\(T(s,x):=(x_2,\ldots,x_K,x)\), so that
\(S_i=T(S_{i-1},X_i)\). Let \(\mathcal F_i:=\sigma(S_0,X_1,\ldots,X_i)\) denote the information
available after observing \(X_i\), with
\(\mathcal F_0:=\sigma(S_0)\). We write \(\Delta(\mathcal X)\) for the
probability simplex on \(\mathcal X\).

The human hypothesis is composite. Let
$
\mathcal P_0
\subseteq
\left\{
P_0:
P_0(\cdot\mid s)\in\Delta(\mathcal X)
\text{ for every }s\in\mathcal S
\right\}
$
be a class of admissible human transition kernels. Each \(P_0\in\mathcal P_0\) specifies a next-token distribution \(P_0(\cdot\mid s)\) at every history state \(s\). Different kernels may represent different writers or writing styles.

We test human-generated text against text produced by one particular target
LLM. The target LLM has a fixed transition kernel \(P_1\),
with \(P_1(\cdot\mid s)\in\Delta(\mathcal X)\) at each history.
Both clean sources are time-homogeneous Markov processes of order \(K\).

We model modifications through history-dependent Huber contamination
\citep{huber,huber1973minimax}. For \(0\le\epsilon<1\), we test
\[
\begin{aligned}
H_0:\quad&
X_i\mid\mathcal F_{i-1}
\sim
P_0(\cdot\mid S_{i-1})
\quad
\text{for some }P_0\in\mathcal P_0,\\
H_1:\quad&
X_i\mid\mathcal F_{i-1}
\sim
(1-\epsilon)P_1(\cdot\mid S_{i-1})
+
\epsilon B_i(\cdot\mid\mathcal F_{i-1}).
\end{aligned}
\]
Here \(B_i(\cdot\mid\mathcal F_{i-1})\in\Delta(\mathcal X)\)
may depend arbitrarily on the preceding text and need not belong
to the human class.  For later use, define the statewise contamination neighborhood
\[
\mathcal Q_{1,\epsilon}(s)
:=
\left\{
(1-\epsilon)P_1(\cdot\mid s)+\epsilon B:
B\in\Delta(\mathcal X)
\right\}.
\]

This is a one-sided, history-dependent version of Huber's
contamination model
\citep{huber,huber1973minimax}, related to recent work on adaptive
robust testing \citep{saha2025huberrobustlikelihoodratiotests}.
Huber contamination have been studied in
covariance estimation, online learning, and change-point detection
\citep{chen2018robustcovariance,chen2021online,bhatt2022offline}.
We next impose two regularity assumptions. 

\begin{assumption}
\label{ass:composite-markov-model}
The human-kernel class \(\mathcal P_0\) is nonempty and compact. For every \(s\in\mathcal S\), there is a nonempty set
\(\mathcal A(s)\subseteq\mathcal X\) such that
\[
\operatorname{supp}P_0(\cdot\mid s)
=
\operatorname{supp}P_1(\cdot\mid s)
=
\mathcal A(s)
\qquad
\text{for every }P_0\in\mathcal P_0.
\]
\end{assumption}

Assumption~\ref{ass:composite-markov-model} ensures that, at any given history, every token
possible under one clean source is also possible under the other,
so a single token cannot identify the source with certainty.
\begin{assumption}
\label{ass:common-support-ergodicity}
Let
\(\mathcal G=(\mathcal S,\mathcal E)\) be the directed graph with edge set
\[
\mathcal E
:=
\left\{
\bigl(s,T(s,x)\bigr):
s\in\mathcal S,\ x\in\mathcal A(s)
\right\}
\]
We assume that \(\mathcal G\) is strongly connected and aperiodic.
\end{assumption}

Assumption~\ref{ass:common-support-ergodicity} ensure that every history has positive long-run frequency under each human kernel. Appendix~\ref{app:uniform-ergodicity} defines strong connectivity and
aperiodicity formally and proves the long-run stability needed in
Section~\ref{sec:minimax-exponent}. An intuitive picture of connectivity is that a discussion can move between topics through a shared “bridge” topic. Philosophy offers one such illustration: discussions of science, art, or everyday decisions can turn to philosophical questions and those questions can lead back to the original subjects.  Related positivity and ergodicity conditions also appear in theoretical analyses of language models and generated-text detection \citep{zekri2024large,varshney2020limits}. With these conditions being specified, we next define the criterion for successful detection.

\subsection{Testing criterion}
\label{subsec:testing-criterion}

For a human kernel \(P_0\in\mathcal P_0\), let \(P_0^{(n)}\) denote the
induced distribution of \(X^n=(X_1,\ldots,X_n)\). Let
\(\mathbf B=(B_i)_{i\geq1}\) denote a contamination strategy, where each
\(B_i(\cdot\mid\mathcal F_{i-1})\in\Delta(\mathcal X)\) may depend on the
 history, and let \(Q_{1,\mathbf B}^{(n)}\) denote the resulting distribution
of \(X^n\) under the alternative.

Consider a test that accepts the human hypothesis when \(X^n\in A_n\)
and the machine hypothesis otherwise.
Its worst-case false-positive and false-negative probabilities are
\[
\alpha_n^{\mathrm{rb}}(A_n)
:=\sup_{P_0\in\mathcal P_0}P_0^{(n)}(A_n^c),
\qquad
\beta_n^{\mathrm{rb}}(A_n)
:=\sup_{\mathbf B}Q_{1,\mathbf B}^{(n)}(A_n),
\]
where the second supremum ranges over admissible contamination strategies. We quantify detection performance using the robust minimax Stein exponent
\citep{chernoff1952measure,hoeffding1965asymptotically,han2003information}:
\[
E_{\mathrm{rb}}
:=
\sup_{\{A_n\}:\,\alpha_n^{\mathrm{rb}}(A_n)\to0}
\liminf_{n\to\infty}
\frac1n
\log\frac1{\beta_n^{\mathrm{rb}}(A_n)}.
\]
Informally, \(E_{\mathrm{rb}}\) describes how quickly longer texts help us detect
contaminated LLM output while keeping false alarms vanishingly rare. A larger \(E_{\mathrm{rb}}\) means that the detection problem is easier.

\section{Possibility and impossibility of robust detection}
\label{sec:minimax-exponent}

Contamination can make machine text indistinguishable from human text. If a contamination strategy makes the machine source follow the same distribution as one admissible human source, no test can reliably tell them apart. We formalize this intution by using a measure of separation between the clean human and machine sources. Define
\begin{equation}
D_0:=\min_{P_0\in\mathcal P_0}
\max_{\substack{s\in\mathcal S\\x\in\mathcal A(s)}}
\log\frac{p_1(x\mid s)}{p_0(x\mid s)}.
\label{eq:clean-kernel-separation}
\end{equation}
It measures the smallest worst-case relative excess of machine probabilities
over human probabilities. It is a directional separation measure and equals zero exactly when
\(P_1\in\mathcal P_0\).

Equivalently, \(D_0\) is the smallest number that there is a human kernel \(P_0^\star\in\mathcal P_0\) satisfying
$$
p_0^\star(x\mid s)\geq e^{-D_0}p_1(x\mid s)
\qquad\text{for every }s\in\mathcal S,\ x\in\mathcal X.
$$
Intuitively, \(1-e^{-D_0}\) is the smallest replacement probability
that allows some edits to make machine-generated text
indistinguishable from one human source.
A larger \(D_0\) means more editing is needed.

\begin{theorem}
\label{thm:clean-centered-two-sided-markov}
Under Assumptions~\ref{ass:composite-markov-model} and~\ref{ass:common-support-ergodicity}, fix \(0\le\epsilon<1\).
\begin{enumerate}
\item If \(D_0\le-\log(1-\epsilon)\), then \(\alpha_n^{\mathrm{rb}}(A_n)+\beta_n^{\mathrm{rb}}(A_n)\ge1\) for every test and every \(n\), so \(E_{\mathrm{rb}}=0\).
\item If \(D_0>-\log(1-\epsilon)\), there exist tests \(A_n^\star\) and \(c>0\) such that \(\alpha_n^{\mathrm{rb}}(A_n^\star)\to0\) and \(\beta_n^{\mathrm{rb}}(A_n^\star)\le e^{-cn}\) for all sufficiently large \(n\), so \(E_{\mathrm{rb}}>0\).
\end{enumerate}
\end{theorem}
Figure~\ref{fig:detection-boundary} illustrates the detection boundary.
When contamination is sufficiently large, the contaminated machine source can reproduce one admissible
human source exactly. Otherwise, the two sources remain distinguishable. In the proof, we
construct a test that combines a collection of clipped likelihood-ratio
tests to distinguish between them.

\begin{figure}[htbp]
\centering
\includegraphics[width=0.5\linewidth]{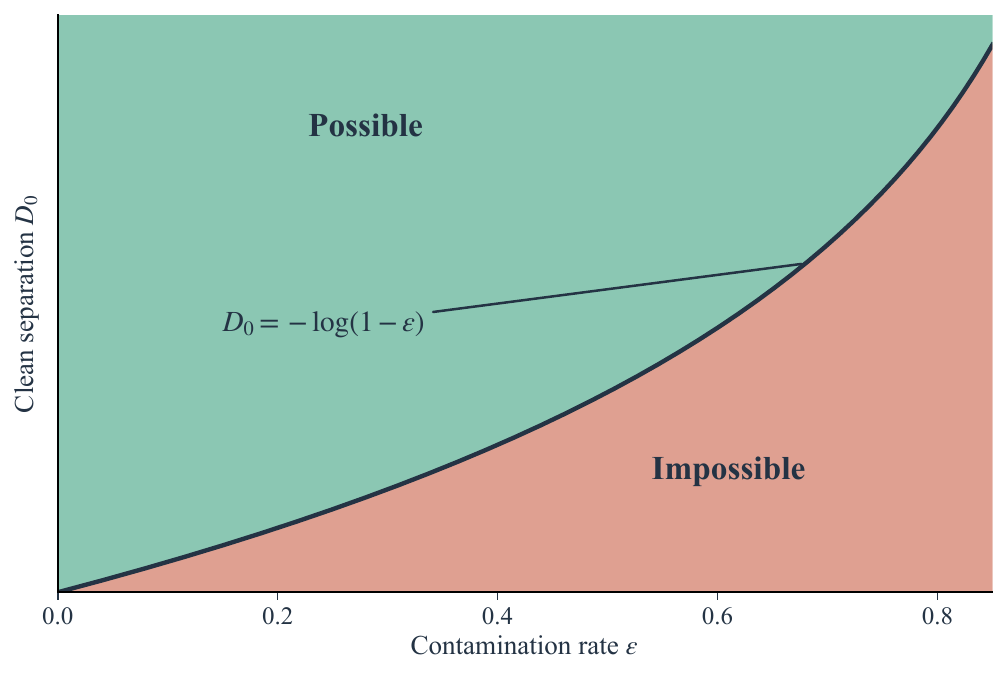}
\caption{Possibility and impossibility of robust detection.}
\label{fig:detection-boundary}
\end{figure}

For each candidate human kernel \(P_0\), consider the token score
\begin{equation}
Z_{P_0}(x\mid s)
=
\max\left\{
\log\frac{(1-\epsilon)p_1(x\mid s)}{p_0(x\mid s)},
\log c_{P_0}(s)
\right\},
\qquad x\in\mathcal A(s),
\label{eq:clipped-likelihood-score}
\end{equation}
where \(c_{P_0}(s)>0\) normalizes the corresponding probability
distribution. The test accepts the human hypothesis if the average
score is sufficiently negative for at least one representative human
kernel. Appendix~\ref{app:proof-theorem-3-1} gives the details and proofs.

Implementing this test, however, requires knowledge of the human-kernel class
and the contamination level, which are generally unavailable in
practice. Nevertheless, its structure suggests a practical strategy:
apply clipping to existing token-level detection scores. We study
this approach next.

\section{Robustifying existing detectors by clipping}
\label{sec:general-clipped-tests}

Many statistical LLM detectors combine token-level evidence into a score for
the entire text. Examples include likelihood- and rank-based methods
\citep{gehrmann-etal-2019-gltr}, Binoculars \citep{hans2024binoculars}, and
LRR \citep{su-etal-2023-detectllm}. We focus on decision rules that, at a
fixed threshold, can be expressed through an average of token scores. After
fixing the sign so that larger values favor \(H_1\), write
\[
T_n(\phi):=\frac1n\sum_{i=1}^n\phi(X_i\mid S_{i-1}).
\]

Clipping is a classical approach to robust hypothesis testing under
Huber's contamination model \citep{huber,huber1973minimax}.
The construction in Section~\ref{sec:minimax-exponent} shows that a
collection of lower-clipped likelihood-ratio tests achieves the
detection boundary under our assumptions. We now study whether
clipping can also improve existing detection scores.

When larger scores favor \(H_1\), unusually negative token scores
can outweigh the evidence from the rest of the text. Such scores
may reflect contamination rather than human authorship.
To limit their influence, define
$
\phi^{(a)}(x\mid s)
:=
\max\{\phi(x\mid s),a\},
x\in\mathcal A(s).
$
The clipped detector replaces \(\phi\) by \(\phi^{(a)}\) in the
token-score average. 
If any observed token lies outside the common clean support, the test rejects the human hypothesis.
Figure~\ref{fig:clipping-illustration} illustrates how clipping
raises unusually low token scores while leaving the others unchanged.
Although clipping limits extreme evidence against machine generation,
it can also remove genuine evidence of human authorship.
We identify conditions under which this tradeoff gives the
clipped test an asymptotic power advantage over the raw test.

\begin{figure}[t]
\centering
\includegraphics[width=0.8\linewidth]{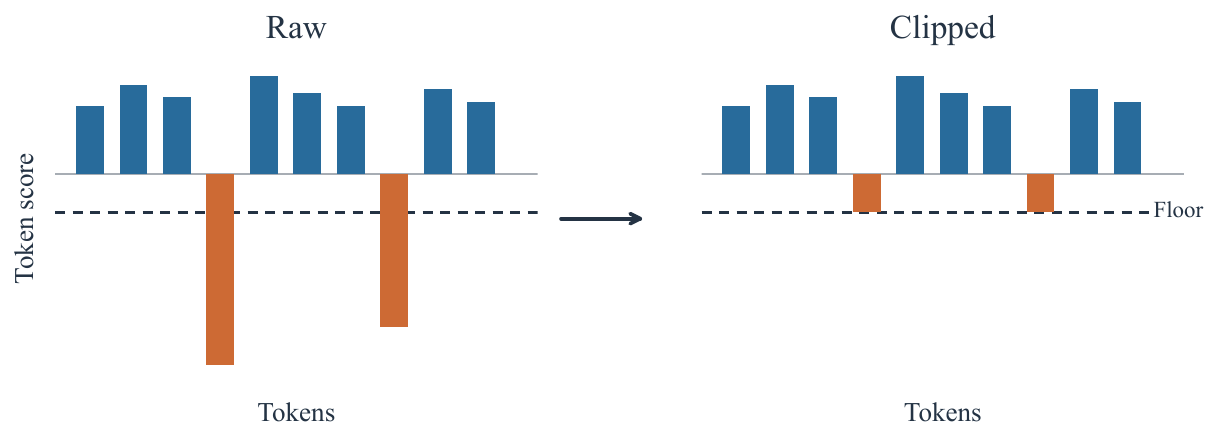}
\caption{Lower clipping of token scores. Scores below the floor
are raised to it; all other scores remain unchanged.
Larger scores favor machine-generated text.}
\label{fig:clipping-illustration}
\end{figure}

Let \(\phi\) be real-valued on the common clean support, with
\(\underline\phi:=\min_{s,\,x\in\mathcal A(s)}\phi(x\mid s)\) and
\(\overline\phi:=\max_{s,\,x\in\mathcal A(s)}\phi(x\mid s)\).
Write
\(\mu_0:=\sup_{P_0\in\mathcal P_0}
\sum_s\pi_{P_0}(s)\mathbb E_{P_0(\cdot\mid s)}[\phi(X\mid s)]\)
for the largest stationary human mean, where \(\pi_{P_0}\)
is the stationary history distribution.
We next state assumptions on machine-score variation,
human--machine separation, and the effect of clipping on human scores.

\begin{assumption}
\label{ass:score-variation-tail}
There exist \(\mu_1\in\mathbb R\) and \(\delta,\eta\ge0\)
such that, for every \(s\in\mathcal S\),
\[
\left|\mathbb E_{P_1(\cdot\mid s)}[\phi(X\mid s)]-\mu_1\right|
\le\delta,
\qquad
\min_{x\in\mathcal A(s)}\phi(x\mid s)\le\underline\phi+\eta.
\]
\end{assumption}

\begin{assumption}
\label{ass:human-score-separation}
The largest stationary human mean satisfies
$
\underline\phi+\eta<\mu_0<\mu_1-\delta.
$
\end{assumption}

Assumption~\ref{ass:score-variation-tail} says that clean machine scores are stable on average, while low-scoring tokens remain possible. Assumption~\ref{ass:human-score-separation} ensures that machine text scores higher on average than human text. Appendix~\ref{app:simulation} illustrates these conditions through simulations with several detectors.

To quantify the cost of clipping, choose a width
$
0<r<\min\{\overline\phi-\underline\phi,
\mu_1-\delta-\underline\phi\},
$
and define
\begin{equation}
\rho_0:=\sup_{P_0\in\mathcal P_0}\sum_s\pi_{P_0}(s)
P_0\{\phi(X\mid s)<\underline\phi+r\mid s\}\in[0,1).
\label{eq:human-clipping-mass}
\end{equation}
At \(a=\underline\phi+r\), \(\rho_0\) is the largest long-run
fraction of human tokens affected by clipping. 
Define
\[
\epsilon_-:=
\frac{\mu_1+\delta-\mu_0}
{\mu_1+\delta-\underline\phi-\eta},
\qquad
\epsilon_+:=\min\left\{1,
\frac{\mu_1-\delta-\mu_0-\rho_0r}
{\mu_1-\delta-\underline\phi-r}\right\}.
\]

The thresholds describe two detection regimes: below \(\epsilon_+\), clipping preserves consistent detection; above \(\epsilon_-\), the raw test loses its worst-case power. When these regimes overlap, clipping gives a strict asymptotic power advantage.

\begin{theorem}
\label{thm:one-sided-clipping-improvement}
Under Assumptions~\ref{ass:composite-markov-model},
\ref{ass:common-support-ergodicity},
\ref{ass:score-variation-tail}, and~\ref{ass:human-score-separation},
fix \(0\le\epsilon<1\).
\begin{enumerate}
\item If \(\epsilon<\epsilon_+\), set \(a=\underline\phi+r\).
There exists a fixed threshold \(\theta_a\) such that
\(A_n^{(a)}:=\{T_n(\phi^{(a)})\le\theta_a\}\) satisfies
\(\alpha_n^{\mathrm{rb}}(A_n^{(a)})\to0\) and
\(\beta_n^{\mathrm{rb}}(A_n^{(a)})\to0\).
\item If \(\epsilon>\epsilon_-\), every raw test
\(A_n^{\mathrm{raw}}:=\{T_n(\phi)\le\theta_n\}\) satisfying
\(\alpha_n^{\mathrm{rb}}(A_n^{\mathrm{raw}})\to0\) also satisfies
\(\beta_n^{\mathrm{rb}}(A_n^{\mathrm{raw}})\to1\).
\end{enumerate}
\end{theorem}

The clipping bound \(a=\underline\phi+r\) is one convenient choice. Other clipping levels also ensure consistency. Appendix~\ref{app:conditions} describes these alternatives.

Consequently, when \(\epsilon_-<\epsilon_+\), clipping restores worst-case power tending to one throughout \((\epsilon_-,\epsilon_+)\), where the raw test’s worst-case power tends to zero provided its worst-case false-positive probability vanishes. The proof of Theorem~\ref{thm:one-sided-clipping-improvement} appears in Appendix~\ref{app:proof-clipping}. Appendix~\ref{app:conditions} discusses when the score conditions yield a nonempty contamination interval, and Appendix~\ref{app:simulation} illustrates their feasibility for several detectors in finite-state simulations.

The theorem identifies a balance between protection against contamination and loss of human evidence. Setting the clipping floor to \(\underline\phi+r\) increases the guaranteed contribution from contamination by \(\epsilon r\), while increasing the largest stationary human mean by at most \(\rho_0r\). Clipping therefore improves the guaranteed mean separation when \(\epsilon>\rho_0\). These guarantees concern additive scores satisfying the stated assumptions. Section~\ref{sec:experiments}  evaluates clipping under practical text modifications, including detectors and clipping procedures as empirical extensions beyond this theorem.

\section{Experiments}
\label{sec:experiments}

We conduct two experiments to evaluate the clipping method motivated by our theory.
The controlled study in Section~\ref{subsec:controlled_contamination} spans
three datasets and three language models, replacing passages of generated text
with human text at specified locations and rates. The RAID evaluation in
Section~\ref{subsec:raid} uses a benchmark of human, machine-generated, and
attacked texts across diverse domains, generators, and attack mechanisms
\citep{dugan2024raid,kandula2025alert,edikala2025leidos,zhu2025reliably}.

Our theory provides guarantees uniformly over all history-dependent contamination strategies within the conditional Huber model, without specifying a particular attack mechanism. The experiments assess whether the resulting clipping principle remains effective under practical text modifications beyond this model. In both studies, human transition distributions are unknown and edits to completed text need not follow the conditional contamination mechanism; RAID additionally includes generators that differ from the scoring model. We observe robustness gains in both studies, with benefits varying across detectors and modification types, supporting clipping as a theoretically motivated and empirically useful approach.

\subsection{Detectors and evaluations}
\label{subsec:detector_and_evaluation}
\paragraph{Detectors.}
We evaluate seven statistical detectors: log likelihood, rank, log rank, entropy, LRR \citep{su-etal-2023-detectllm}, entropy gap based on \citet{radvand2025trainingfree}, and Binoculars \citep{hans2024binoculars}. Table~\ref{tab:primary-methods} summarizes their scores. Here, \(p_i=p_i(x_i)\) is the probability assigned to the observed token, \(r_i=1+\#\{v:p_i(v)>p_i(x_i)\}\) is its rank, and \(h_i=-\sum_v p_i(v)\log p_i(v)\) is predictive entropy. An overbar denotes the average over scored positions.

LRR combines token surprise with log rank, while entropy gap compares observed surprise with predictive entropy. Binoculars compares observed surprise under a performer model with the cross-entropy \(h_i^\times=-\sum_v p_i^{\mathrm O}(v)\log p_i^{\mathrm P}(v)\) between an observer and the performer. We use the original Falcon-7B observer and Falcon-7B-Instruct performer.

\paragraph{Clipping.}
We orient all scores so that larger values favor machine-generated text. LRR and Binoculars use their original directions; the remaining directions are learned from clean tuning data. Table~\ref{tab:primary-methods} summarizes the directions and clipping operations.

For additive detectors, we clip the oriented token scores from below before averaging. For LRR, we cap the token-level contributions to both the numerator and denominator; for Binoculars, we cap only the numerator contributions. We then average the contributions, form the ratios, and apply their respective directions. Appendix~\ref{app:protocol} provides implementation details.

\begin{table}[!htb]
\centering\small
\caption{Detector statistics and clipping operations. All bounds are
applied to token-level quantities before averaging.}
\label{tab:primary-methods}
\renewcommand{\arraystretch}{1.2}
\setlength{\tabcolsep}{6pt}
\begin{tabularx}{\linewidth}{@{}l c X@{}}
\toprule
Detector & Raw statistic & Orientation and clipping \\
\midrule
Log likelihood & \(\overline{\log p_i}\)
    & Keep sign; clip token scores from below \\
Rank & \(\overline{r_i}\)
    & Reverse sign; clip oriented scores from below \\
Log rank & \(\overline{\log r_i}\)
    & Reverse sign; clip oriented scores from below \\
LRR & \(\overline{-\log p_i}/\overline{\log r_i}\)
    & Keep ratio's sign; clip both contributions from above \\
Entropy & \(\overline{h_i}\)
    & Orient using tuning data; clip oriented scores from below \\
Entropy gap & \(\overline{-\log p_i-h_i}\)
    & Reverse sign; clip oriented scores from below \\
Binoculars & \(\overline{-\log p_i}/\overline{h_i^\times}\)
    & Reverse numerator's sign; only clip numerator from below \\
\bottomrule
\end{tabularx}
\end{table}

\paragraph{Evaluation.}
Both studies use separate tuning, calibration, and test sets, keeping
each human text and all corresponding machine-generated versions together.
Tuning selects score directions and clipping bounds: the controlled study
uses weighted AUROC, which summarizes discrimination across thresholds;
RAID uses weighted TPR, the fraction of machine-origin texts detected,
with cross-fitted thresholds.

Independent human calibration data set separate raw and clipped thresholds
at the target FPR, the fraction of human texts mislabeled as machine-generated.
Both rules use the same test texts. Pointwise 95\% confidence intervals use
2,000 paired source-cluster bootstrap resamples with fitted parameters fixed.
Appendix~\ref{app:protocol} provides details.

\subsection{Controlled contamination study}
\label{subsec:controlled_contamination}
\paragraph{Setup}
We pair three datasets---SQuAD contexts
\citep{rajpurkar2016squad}, XSum news articles
\citep{narayan2018xsum}, and WritingPrompts stories
\citep{fan2018writingprompts}---with Granite~3.3~8B Base,
Mistral Small~24B Base, and Qwen~2.5~32B, yielding nine
dataset--generator combinations. Single-model detectors use the
corresponding generator as scorer; Binoculars retains the original
Falcon-7B observer and Falcon-7B-Instruct performer.

Each combination contains 3,000 source groups: 500 for tuning,
500 for calibration, and 2,000 for testing. An approximately
30-token human prefix prompts a machine continuation targeting
220 tokens. At each positive replacement rate
(5, 10, 20, 30, 40, and 50\%), we create three random variants
by replacing randomly located portions with human passages,
and one tail variant by replacing the ending with human passages
selected for high average surprise under the generator.
Clean texts provide the 0\% baseline.  For each detector and dataset--generator combination, we tune one
clipping specification using a weighted AUROC objective and apply
it across all contamination rates.
Appendices~\ref{app:data} and~\ref{app:protocol} provide
construction and fitting details.

\paragraph{Results}
Table~\ref{tab:primary-summary} shows the trade-off between clean
detection and robustness. Clipping decreases mean clean TPR for
five of seven detectors, but improves mean TPR for every detector
at 20\% replacement under both constructions. Under random
replacement, Binoculars increases from 58.8\% to 85.7\%
(26.9 percentage points), followed by LRR with a 26.2-point gain.
Under tail replacement, rank gains the most, rising from 0.3\%
to 34.3\%. Entropy benefits least, improving by 1.4 and 2.0 points,
respectively.

\begin{table}[!htb]
\centering
\caption{Controlled-study TPR at target 5\% FPR, averaged equally
across nine settings. Raw: TPR (\%); gain: clipped minus raw
(percentage points). Replacement rates are requested.}
\label{tab:primary-summary}
\begin{tabular}{lrrrrrr}
\toprule
& \multicolumn{2}{c}{0\% replacement}
& \multicolumn{4}{c}{20\% replacement} \\
\cmidrule(lr){2-3}\cmidrule(lr){4-7}
& \multicolumn{2}{c}{Clean}
& \multicolumn{2}{c}{Random}
& \multicolumn{2}{c}{Tail} \\
\cmidrule(lr){2-3}\cmidrule(lr){4-5}\cmidrule(lr){6-7}
Detector & Raw & Gain & Raw & Gain & Raw & Gain \\
\midrule
Log likelihood & 62.6 & -2.9 & 9.2  & +7.6  & 21.0 & +7.7 \\
Rank           & 70.7 & -5.3 & 0.1  & +21.4 & 0.3  & +34.0 \\
Log rank       & 67.0 & -3.1 & 12.1 & +9.5  & 24.8 & +9.1 \\
LRR            & 71.2 & -3.5 & 22.3 & +26.2 & 30.0 & +21.4 \\
Entropy        & 23.4 & +0.9 & 21.3 & +1.4  & 18.4 & +2.0 \\
Entropy gap    & 91.6 & -4.1 & 61.5 & +23.1 & 71.3 & +11.5 \\
Binoculars     & 92.0 & +0.5 & 58.8 & +26.9 & 65.3 & +15.7 \\
\bottomrule
\end{tabular}
\end{table}

Figure~\ref{fig:primary-binoculars} examines Binoculars across
all nine settings. Although performance varies across datasets
and generators, clipping increases the area under the
TPR--contamination curve under both constructions in every setting.
Mean test FPR remains below 5\% for raw (4.48\%) and clipped
(4.82\%) scores. Appendix~\ref{app:primary} reports complete
results for all detectors.

\begin{figure}
\centering
\includegraphics[width=0.7\linewidth]{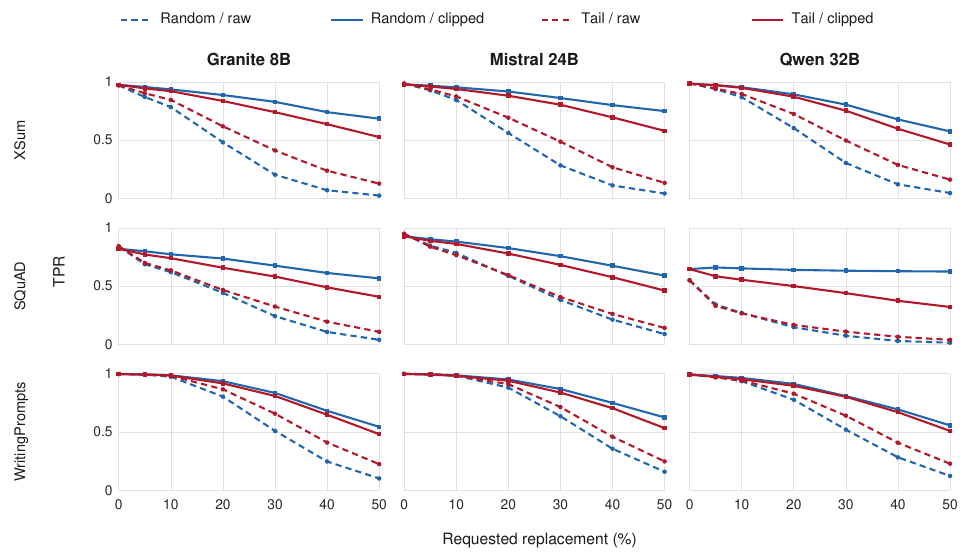}
\caption{Binoculars TPR across nine settings versus requested
replacement rate. Raw and clipped thresholds are separately
calibrated at target 5\% FPR.}
\label{fig:primary-binoculars}
\end{figure}

\subsection{RAID benchmark}
\label{subsec:raid}
\paragraph{Setup}
We use RAID \citep{dugan2024raid}, a benchmark for evaluating
machine-generated text detection across diverse domains, generators,
and attacks. It contains human texts, machine-generated counterparts,
and attacked versions spanning eight domains and eleven attack
mechanisms, including paraphrasing, word substitution, spelling
changes, and so on.

For each human source, we randomly select one machine counterpart
with all eleven attacked versions. After excluding 500 pilot sources,
the sample contains 12,871 source groups and 167,323 documents:
5,148 groups for tuning, 2,574 for calibration, and 5,149 for testing.
We use Falcon-7B for all the six single-model scores, while Binoculars uses
the original Falcon-7B observer and Falcon-7B-Instruct performer.
We score at most 512 output tokens.

We measure modification using the normalized token edit rate
\(\rho\): the minimum number of insertions, deletions, and substitutions
between clean and attacked texts, divided by the longer sequence length.
Unlike the controlled experiment study, we train a different clipping bound for different contamination rates and attack types.
All results target 5\% FPR. Appendices~\ref{app:data}
and~\ref{app:stat-clipping} provide further details.

\paragraph{Results}
Table~\ref{tab:raid-rate-specific} compares seven detectors across
three rate intervals. Rank gains the most, with an average increase of around 20 percentage points. LRR, log rank,
entropy gap, and Binoculars improve in all three intervals,
whereas log likelihood improves only in the lowest interval
and entropy changes little.

\begin{table}[!htb]
\centering
\caption{Rate-specific RAID TPR at target 5\% FPR.
Raw: TPR (\%); gain: clipped minus raw (percentage points).
Test counts: 19,894 / 4,682 / 7,137.}
\label{tab:raid-rate-specific}
\setlength{\tabcolsep}{6pt}
\begin{tabular}{@{}lrrrrrr@{}}
\toprule
& \multicolumn{2}{c}{\(0<\rho\leq0.05\)}
& \multicolumn{2}{c}{\(0.05<\rho\leq0.1\)}
& \multicolumn{2}{c}{\(0.1<\rho\leq0.2\)} \\
\cmidrule(lr){2-3}\cmidrule(lr){4-5}\cmidrule(lr){6-7}
Detector & Raw & Gain & Raw & Gain & Raw & Gain \\
\midrule
Log likelihood & 49.17 & +0.32  & 55.62 & +0.00  & 30.47 & +0.00 \\
Rank           & 29.48 & +25.23 & 35.05 & +26.40 & 24.87 & +15.96 \\
Log rank       & 53.24 & +0.74  & 59.38 & +0.68  & 35.79 & +0.28 \\
LRR            & 59.51 & +2.28  & 64.48 & +4.21  & 47.60 & +0.73 \\
Entropy        & 26.11 & +0.09  & 35.07 & +0.43  & 18.50 & +0.01 \\
Entropy gap    & 66.55 & +1.60  & 71.64 & +1.13  & 46.77 & +1.02 \\
Binoculars     & 74.27 & +0.22  & 82.79 & +0.28  & 60.88 & +0.87 \\
\bottomrule
\end{tabular}
\end{table}

Table~\ref{tab:raid-attack-specific} examines three representative
attacks: case changes alter capitalization, paraphrasing rewrites
the text, and alternative spelling substitutes spelling variants.
All seven detectors improve under case changes, including gains
of 3.42 points for entropy gap and 2.21 points for log rank.
Under paraphrasing, gains are concentrated in LRR and rank,
while entropy gap and Binoculars remain unchanged.
Alternative spelling benefits six detectors, including entropy,
entropy gap, and Binoculars, while log likelihood remains unchanged.
These results show that the benefit of clipping depends on both
the detector and the attack.

\begin{table}[!htb]
\centering
\caption{Attack-specific RAID TPR at target 5\% FPR,
with 5,149 texts per attack. Raw: TPR (\%); gain:
clipped minus raw (percentage points).}
\label{tab:raid-attack-specific}
\setlength{\tabcolsep}{6pt}
\begin{tabular}{@{}lrrrrrr@{}}
\toprule
& \multicolumn{2}{c}{Case change}
& \multicolumn{2}{c}{Paraphrase}
& \multicolumn{2}{c}{Alternative spelling} \\
\cmidrule(lr){2-3}\cmidrule(lr){4-5}\cmidrule(lr){6-7}
Detector & Raw & Gain & Raw & Gain & Raw & Gain \\
\midrule
Log likelihood & 44.96 & +1.15  & 27.91 & +0.00 & 59.22 & +0.00 \\
Rank           & 22.02 & +31.37 & 27.89 & +7.79 & 39.50 & +23.83 \\
Log rank       & 50.15 & +2.21  & 31.87 & -0.10 & 62.11 & +0.19 \\
LRR            & 59.16 & +3.52  & 38.80 & +9.30 & 65.04 & +4.87 \\
Entropy        & 24.20 & +0.43  & 11.87 & +0.10 & 33.52 & +0.54 \\
Entropy gap    & 65.43 & +3.42  & 75.78 & +0.00 & 72.54 & +0.85 \\
Binoculars     & 74.07 & +0.43  & 80.83 & +0.00 & 78.71 & +0.31 \\
\bottomrule
\end{tabular}
\end{table}

\section{Conclusion}

We studied when LLM-generated text remains detectable after editing
and contamination. Under a finite-order Markov model with
history-dependent Huber contamination, we characterized an exact
boundary determined by contamination and clean-source separation.
Above or at this boundary, contamination can reproduce an admissible
human source, making reliable detection impossible. Below it,
a collection of lower-clipped likelihood-ratio tests achieves
vanishing worst-case errors. This construction connects the
possibility of robust detection to clipping individual token scores.

Motivated by this result, we studied clipping for general additive
detection scores. Under explicit conditions, clipping preserves
consistent detection over a contamination interval in which the raw
test's worst-case power tends to zero. Experiments with seven detectors
across three datasets and three generators, together with evaluation
on RAID, show that clipping can also improve robustness beyond the
theoretical contamination model. We also find the gains depend on the detector
and modification mechanism, and may come at the cost of reduced
clean-text detection.

Several limitations remain. The theoretical guarantees are asymptotic and rely on finite-state assumptions, leaving the text lengths needed for reliable detection unresolved. The boundary-achieving test also requires knowledge of the human class and contamination level. Future work could develop finite-sample guarantees and study how these bounds transfer to unseen modification mechanisms. Our results provide a theoretical foundation and empirical support for clipping as a simple modification that can strengthen existing detectors.

\bibliographystyle{plainnat}
\bibliography{references}

%\clearpage
\appendix
\section*{Appendices}

\section{Uniform ergodicity of the human-kernel family}
\label{app:uniform-ergodicity}
We collect the finite-state Markov-chain definitions needed below; see, for
example, \citet{levin2017markov}.  Recall the common directed graph
\(\mathcal G=(\mathcal S,\mathcal E)\) from
Assumptions~\ref{ass:composite-markov-model} and~\ref{ass:common-support-ergodicity}.

A human token kernel \(P_0\in\mathcal P_0\) induces a transition matrix on
\(\mathcal S\), denoted by \(M_{P_0}\), with entries
\[
M_{P_0}(s,t)
:=
\sum_{\substack{x\in\mathcal A(s):\\T(s,x)=t}}
p_0(x\mid s).
\]
Thus \(M_{P_0}(s,t)\) is the probability that the next state is \(t\)
when the current state is \(s\).  Under a fixed human kernel \(P_0\),
\((S_i)_{i\geq0}\) is a time-homogeneous Markov chain satisfying
\[
S_i\mid\mathcal F_{i-1}
\sim M_{P_0}(S_{i-1},\cdot).
\]
For this transition matrix, the entry \(M_{P_0}^n(s,t)\) of the \(n\)-th matrix power is the
\(n\)-step transition probability from \(s\) to \(t\).  The common-support
condition gives
\[
M_{P_0}(s,t)>0
\quad\Longleftrightarrow\quad
(s,t)\in\mathcal E.
\]
A directed walk of length \(n\geq0\) from \(s\) to \(t\) is a
sequence \(s=s_0,s_1,\ldots,s_n=t\) with
\((s_{j-1},s_j)\in\mathcal E\) for \(j=1,\ldots,n\). States and edges
may repeat; a walk is \emph{closed} if it starts and ends at the same state.
Thus \(M_{P_0}^n(s,t)>0\) exactly when the graph contains a directed walk
of length \(n\) from \(s\) to \(t\).

\begin{definition}[Strong connectivity and irreducibility]
\label{def:strong-connectivity-irreducibility}
A directed graph is \emph{strongly connected} if every state can be reached
from every other by a directed walk. A finite-state Markov chain is
\emph{irreducible} if, for every pair \(s,t\), there exists \(n\geq1\)
such that \(M_{P_0}^n(s,t)>0\).
\end{definition}

Strong connectivity is a property of the graph, whereas irreducibility is
a property of the chain. They are equivalent for a chain and its support
graph: Let \(\mathcal G=(\mathcal S,\mathcal E)\), where \(\mathcal E=\{(s,t)\in\mathcal S^2:M_{P_0}(s,t)>0\}\). Then \(\mathcal G\) is strongly connected if and only if \(M_{P_0}\) is irreducible for every \(P_0\in\mathcal P_0\).

\begin{definition}[Period and aperiodicity]
\label{def:period-aperiodicity}
For an irreducible chain, the period of a state \(s\) is
\[
\gcd\left\{
n\geq1:M_{P_0}^n(s,s)>0
\right\}.
\]
All states in an irreducible chain have the same period
\citep[Lemma~1.6]{levin2017markov}.
The chain and its support graph are \emph{aperiodic} if this common period
equals one.
\end{definition}

Strong connectivity concerns where the process can go; aperiodicity concerns
when it can return. Since every \(M_{P_0}\) has support graph \(\mathcal G\),
the graph conditions apply to every admissible human kernel.

\begin{definition}[Stationarity and uniform geometric ergodicity]
\label{def:stationarity-uniform-geometric-ergodicity}
A probability distribution \(\pi_{P_0}\) on \(\mathcal S\) is \emph{stationary} if
\[
\pi_{P_0}M_{P_0}=\pi_{P_0}.
\]
 The family
\(\{M_{P_0}:P_0\in\mathcal P_0\}\) is \emph{uniformly geometrically
ergodic} if each member has a unique stationary distribution and there
exist constants
\(0<C<\infty\) and \(\rho\in(0,1)\) such that
\begin{equation}
\label{eq:uniform-geometric-ergodicity}
\sup_{P_0\in\mathcal P_0}
\sup_{s\in\mathcal S}
\left\|
M_{P_0}^n(s,\cdot)-\pi_{P_0}
\right\|_{\mathrm{TV}}
\leq
C\rho^n
\qquad
\text{for every }n\geq0.
\end{equation}
\end{definition}

This bound gives convergence to stationarity at a geometric rate, uniformly
over the human kernels and initial states.

The following lemma establishes convergence uniformly over human kernels and initial history states. We assume\(P_0^{(n)}\) is evaluated with the corresponding initial state \(S_0\).

\begin{lemma}[Uniform ergodicity and uniform empirical averages]
\label{lem:uniform-ergodicity}
Under Assumptions~\ref{ass:composite-markov-model} and~\ref{ass:common-support-ergodicity}, every
\(P_0\in\mathcal P_0\) induces
an irreducible and aperiodic state chain with a unique full-support stationary
distribution \(\pi_{P_0}\).  The family is uniformly geometrically ergodic.
Moreover, for every bounded \(g:\mathcal S\to\mathbb R\) and every
\(\eta>0\),
\begin{equation}
\label{eq:uniform-empirical-state-average}
\sup_{S_0\in\mathcal S}
\sup_{P_0\in\mathcal P_0}
P_0^{(n)}
\left(
\left|
\frac1n\sum_{i=1}^n g(S_{i-1})
-
\sum_{s\in\mathcal S}\pi_{P_0}(s)g(s)
\right|
>\eta
\right)
\longrightarrow0.
\end{equation}
\end{lemma}

\begin{proof}
We first obtain a common number of steps for reaching any history state,
then use it to control convergence and empirical averages.

By Assumption~\ref{ass:composite-markov-model}, the finite state and token
spaces give positive lower bounds
\[
m_0:=\min_{P_0\in\mathcal P_0}
\min_{s\in\mathcal S,\,x\in\mathcal A(s)}p_0(x\mid s)>0,
\qquad
m_1:=\min_{s\in\mathcal S,\,x\in\mathcal A(s)}p_1(x\mid s)>0.
\]

Furthermore, Assumption~\ref{ass:common-support-ergodicity} imply every \(M_{P_0}\) is irreducible and aperiodic because its support graph is
\(\mathcal G\). Apply \citet[Proposition~1.7]{levin2017markov} to any one
human kernel: there exists an integer \(L\geq1\) for which all \(L\)-step
transition probabilities are positive. Thus every pair of states is
connected by a directed walk of length \(L\). Since all human kernels
have the same allowed transitions, the same \(L\) works for the entire
family. Each transition along such a walk has probability at least \(m_0\),
so the probability of following that walk is at least \(m_0^L\). Therefore
\[
M_{P_0}^L(s,t)\geq m_0^L
\qquad
\text{for every }s,t\in\mathcal S,\ P_0\in\mathcal P_0.
\]
By \citet[Corollary~1.17]{levin2017markov}, each finite irreducible chain
has a unique stationary distribution \(\pi_{P_0}\). It has full support,
since stationarity and the preceding bound give
\[
\pi_{P_0}(t)
=
\sum_{s\in\mathcal S}\pi_{P_0}(s)M_{P_0}^L(s,t)
\geq m_0^L>0.
\]

\emph{Uniform geometric convergence.}
We show directly that the distance to stationarity decreases by a common
factor every \(L\) steps. For \(n\geq0\), the probabilities
\(M_{P_0}^n(s,\cdot)\) and \(\pi_{P_0}\) both sum to one. Hence
\[
\begin{aligned}
&M_{P_0}^{n+L}(s,t)-\pi_{P_0}(t)\\
&\quad=
\sum_{u\in\mathcal S}
\bigl(M_{P_0}^n(s,u)-\pi_{P_0}(u)\bigr)
\bigl(M_{P_0}^L(u,t)-m_0^L\bigr).
\end{aligned}
\]
Subtracting \(m_0^L\) does not change this sum, because the first factor
sums to zero over \(u\). The second factor is nonnegative and sums to
\(1-|\mathcal S|m_0^L\) over \(t\). Taking absolute values, summing over
\(t\) gives
\[
\begin{aligned}
&\left\|M_{P_0}^{n+L}(s,\cdot)-\pi_{P_0}\right\|_{\mathrm{TV}}\\
&\quad\leq
\bigl(1-|\mathcal S|m_0^L\bigr)
\left\|M_{P_0}^{n}(s,\cdot)-\pi_{P_0}\right\|_{\mathrm{TV}}.
\end{aligned}
\]
Each row of
\(M_{P_0}^L\) has \(|\mathcal S|\) entries, each at least \(m_0^L\), and
sums to one. Thus
\[
1=\sum_{t\in\mathcal S}M_{P_0}^L(s,t)
\geq |\mathcal S|m_0^L>0,
\]
so \(0\leq1-|\mathcal S|m_0^L<1\).

The total variation distance between any two probability distributions is
at most one. Write \(n=kL+r\), where
\(k\geq0\) and \(0\leq r<L\). Starting at time \(r\) with this bound
of one and applying the preceding contraction bound \(k\) times gives, for
\(k\geq1\),
\[
\left\|M_{P_0}^{n}(s,\cdot)-\pi_{P_0}\right\|_{\mathrm{TV}}
\leq \bigl(1-|\mathcal S|m_0^L\bigr)^k.
\]
For \(k=0\), the distance remains bounded by one. To express this \(L\)-step bound in the form \(C\rho^n\), choose
\[
\rho:=\max\left\{\bigl(1-|\mathcal S|m_0^L\bigr)^{1/L},\frac12\right\},
\qquad
C:=\rho^{-(L-1)}.
\]
Taking the maximum with \(1/2\) ensures \(\rho>0\). These choices give \(0<\rho<1\), \(C\geq1\), and, for
\(k\geq1\),
\[
\bigl(1-|\mathcal S|m_0^L\bigr)^k
\leq\rho^{kL}
=\rho^{n-r}
\leq\rho^{-(L-1)}\rho^n
=C\rho^n.
\]
Since neither constant depends on the human kernel or initial state, we have, for every \(n\geq0\),
\[
\sup_{P_0\in\mathcal P_0}\sup_{s\in\mathcal S}
\left\|M_{P_0}^n(s,\cdot)-\pi_{P_0}\right\|_{\mathrm{TV}}
\leq C\rho^n.
\]
This is \eqref{eq:uniform-geometric-ergodicity}.

\emph{Empirical averages.}
Fix a bounded \(g\), and write
\(\|g\|_\infty:=\max_{s\in\mathcal S}|g(s)|\).
For \(j\geq i\), the Markov property and the preceding bound imply
\[
\begin{aligned}
&\left|
\mathbb E_{P_0}[g(S_j)\mid S_i]
-\sum_{t\in\mathcal S}\pi_{P_0}(t)g(t)
\right|\\
&\quad=
\left|
\sum_{t\in\mathcal S}
\bigl(M_{P_0}^{j-i}(S_i,t)-\pi_{P_0}(t)\bigr)g(t)
\right|
\leq 2\|g\|_\infty C\rho^{j-i}.
\end{aligned}
\]
Also,
\(\left|g(S_i)-\sum_t\pi_{P_0}(t)g(t)\right|\leq2\|g\|_\infty\).
Conditioning on \(S_i\) thus gives
\[
\left|
\mathbb E_{P_0}\left[
\left(g(S_i)-\sum_t\pi_{P_0}(t)g(t)\right)
\left(g(S_j)-\sum_t\pi_{P_0}(t)g(t)\right)
\right]
\right|
\leq4\|g\|_\infty^2 C\rho^{j-i}.
\]
Expanding the square and using this covariance bound gives
\[
\begin{aligned}
&\mathbb E_{P_0}\left[
\left(
\frac1n\sum_{i=1}^n g(S_{i-1})
-\sum_s\pi_{P_0}(s)g(s)
\right)^2
\right]\\
&\quad=
\frac1{n^2}\sum_{i=0}^{n-1}\sum_{j=0}^{n-1}
\mathbb E_{P_0}\left[
\left(g(S_i)-\sum_s\pi_{P_0}(s)g(s)\right)
\left(g(S_j)-\sum_s\pi_{P_0}(s)g(s)\right)
\right]\\
&\quad\leq
\frac{4\|g\|_\infty^2 C}{n^2}
\sum_{i=0}^{n-1}\sum_{j=0}^{n-1}\rho^{|j-i|},
\end{aligned}
\]

To count these terms, there are \(n\) pairs with \(i=j\), each contributing
\(\rho^0=1\). For each separation \(k=1,\ldots,n-1\), there are
\(n-k\) pairs with \(j=i+k\) and another \(n-k\) with \(i=j+k\).
Consequently,
\[
\begin{aligned}
\sum_{i=0}^{n-1}\sum_{j=0}^{n-1}\rho^{|j-i|}
&=n+2\sum_{k=1}^{n-1}(n-k)\rho^k\\
&\leq n+2n\sum_{k=1}^{\infty}\rho^k
=n\left(1+\frac{2\rho}{1-\rho}\right).
\end{aligned}
\]
Here we used \(n-k\leq n\) and the geometric-series identity
\(\sum_{k=1}^{\infty}\rho^k=\rho/(1-\rho)\).
Since the bounds hold for every human kernel and initial state with the
same constants, we obtain
\[
\begin{aligned}
&\sup_{S_0\in\mathcal S}\sup_{P_0\in\mathcal P_0}
\mathbb E_{P_0}\left[
\left(
\frac1n\sum_{i=1}^n g(S_{i-1})
-\sum_{s\in\mathcal S}\pi_{P_0}(s)g(s)
\right)^2
\right]\\
&\quad\leq
\frac{4\|g\|_\infty^2 C}{n}
\left(1+\frac{2\rho}{1-\rho}\right).
\end{aligned}
\]
Finally, fix \(\eta>0\). Markov's inequality gives
\[
\begin{aligned}
&\mathbb E_{P_0}\left[
\left(
\frac1n\sum_{i=1}^n g(S_{i-1})
-\sum_{s\in\mathcal S}\pi_{P_0}(s)g(s)
\right)^2
\right]\\
&\quad\geq
\eta^2 P_0^{(n)}\left(
\left|
\frac1n\sum_{i=1}^n g(S_{i-1})
-\sum_{s\in\mathcal S}\pi_{P_0}(s)g(s)
\right|>\eta
\right).
\end{aligned}
\]
Using the preceding second-moment bound yields
\[
\begin{aligned}
&\sup_{S_0\in\mathcal S}\sup_{P_0\in\mathcal P_0}
P_0^{(n)}\left(
\left|
\frac1n\sum_{i=1}^n g(S_{i-1})
-\sum_{s\in\mathcal S}\pi_{P_0}(s)g(s)
\right|>\eta
\right)\\
&\quad\leq
\frac{4\|g\|_\infty^2 C}{n\eta^2}
\left(1+\frac{2\rho}{1-\rho}\right)
\longrightarrow0.
\end{aligned}
\]
For fixed \(g\) and \(\eta\), the right-hand side is a finite constant
times \(1/n\). Since that constant is independent of \(P_0\) and \(S_0\),
this proves \eqref{eq:uniform-empirical-state-average}.
\end{proof}

\section{Proofs for Section~\ref{sec:minimax-exponent}}
\label{app:proof-theorem-3-1}
We prove the two parts of
Theorem~\ref{thm:clean-centered-two-sided-markov}.
The impossibility proof constructs contamination that reproduces one
admissible human source exactly. The possibility proof combines
lower-clipped likelihood-ratio scores for finitely many representative
human kernels.

\subsection{Auxiliary lemmas}

At a fixed history, every contaminated machine distribution has the form
\(q=(1-\epsilon)p_1+\epsilon b\). Since \(b\) is nonnegative, each
coordinate of \(q\) must be at least the retained machine probability.
This coordinatewise lower bound is also sufficient: any remaining mass
can be supplied by the contaminating distribution.

\begin{lemma}
\label{lem:whole-kernel-overlap}
Under Assumptions~\ref{ass:composite-markov-model} and
\ref{ass:common-support-ergodicity}, the minimum defining \(D_0\) is attained,
\(0\le D_0<\infty\), and \(D_0=0\) if and only if \(P_1\in\mathcal P_0\).
For \(0\le\epsilon<1\),
\begin{equation}
\mathcal Q_{1,\epsilon}(s)
=\{Q\in\Delta(\mathcal X):q(x)\ge(1-\epsilon)p_1(x\mid s)
\text{ for every }x\}.
\label{eq:one-sided-contamination-envelope}
\end{equation}
There is a kernel \(P_0\in\mathcal P_0\) whose rows belong to
\(\mathcal Q_{1,\epsilon}(s)\) at every history if and only if
\(D_0\le-\log(1-\epsilon)\).
\end{lemma}

\begin{proof}
The common positive probability bounds make
\(P_0\mapsto\max_{s,x\in\mathcal A(s)}\log[p_1(x\mid s)/p_0(x\mid s)]\)
continuous and finite. Indeed, each log ratio lies between
\(\log m_1\) and \(-\log m_0\), and the maximum is over finitely many
coordinates.

At each history, the two probability vectors sum to one, so at least one
coordinate satisfies \(p_1(x\mid s)\ge p_0(x\mid s)\). Hence \(D_0\ge0\).
If \(D_0=0\), a minimizing kernel satisfies \(p_0(x\mid s)\ge p_1(x\mid s)\)
everywhere. Normalization forces equality at every history, so
\(P_1\in\mathcal P_0\). The converse follows by choosing \(P_0=P_1\).

For \(\epsilon>0\), the mixture representation we describe before this lemma implies the inequalities
in \eqref{eq:one-sided-contamination-envelope}. Conversely, those
inequalities make
\[
b(x\mid s)=\frac{q(x)-(1-\epsilon)p_1(x\mid s)}{\epsilon}
\]
nonnegative. Its total mass is
\[
\sum_x b(x\mid s)
=\frac{1-(1-\epsilon)}{\epsilon}=1,
\]
so it is a valid probability distribution. When \(\epsilon=0\), the inequalities
and normalization imply \(Q=P_1(\cdot\mid s)\).

From definition of $D_0$, for a whole human kernel, the condition is equivalent to
\[
\max_{\substack{s\in\mathcal S\\x\in\mathcal A(s)}}
\log\frac{p_1(x\mid s)}{p_0(x\mid s)}
\le-\log(1-\epsilon).
\]
Taking the attained minimum over \(\mathcal P_0\) proves the last claim.
Off the common support, both clean probabilities are zero, so those
coordinates impose no additional restriction.
\end{proof}

For example, at one history let \(P_1=(0.8,0.2)\) and
\(P_0=(0.4,0.6)\). With \(\epsilon=0.25\), the retained machine mass
is \((0.6,0.15)\); its first coordinate already exceeds \(0.4\), so
adding contamination cannot produce \(P_0\). With \(\epsilon=0.5\),
the retained mass is \((0.4,0.1)\), and choosing \(B=(0,1)\) gives
\(P_0\) exactly.

Fix \(0\le\epsilon<1\). For each human kernel \(P_0\), define
\[
Q_{P_0}^{\star}(\cdot\mid s)
\in
\operatorname*{arg\,min}_{Q\in\mathcal Q_{1,\epsilon}(s)}
D_{\mathrm{KL}}\bigl(P_0(\cdot\mid s)\|Q\bigr).
\]
This distribution is the contaminated machine distribution closest
to the candidate human distribution in the displayed divergence.
The associated likelihood score has the clipping form used in
Section~\ref{sec:minimax-exponent}.

\begin{lemma}[Clipped likelihood scores]
\label{lem:clipped-likelihood-projection}
Under Assumptions~\ref{ass:composite-markov-model}
and~\ref{ass:common-support-ergodicity}, the minimizer is unique and
satisfies
\[
q_{P_0}^{\star}(x\mid s)
=
\max\{(1-\epsilon)p_1(x\mid s),
      c_{P_0}(s)p_0(x\mid s)\},
\qquad x\in\mathcal A(s),
\]
and is zero outside \(\mathcal A(s)\). Here \(c_{P_0}(s)>0\)
is chosen so that the probabilities sum to one. Consequently,
\[
Z_{P_0}(x\mid s)
:=
\log\frac{q_{P_0}^{\star}(x\mid s)}{p_0(x\mid s)}
\]
equals \eqref{eq:clipped-likelihood-score} on the common support.
Set \(Z_{P_0}(x\mid s)=+\infty\) outside that support.

For every \(Q\in\mathcal Q_{1,\epsilon}(s)\),
\begin{equation}
\mathbb E_Q\bigl[e^{-Z_{P_0}(X\mid s)}\bigr]\le1.
\label{eq:clipped-likelihood-exponential}
\end{equation}
In particular, under every admissible contamination strategy,
\[
\mathbb E\left[
e^{-Z_{P_0}(X_i\mid S_{i-1})}
\mid\mathcal F_{i-1}
\right]\le1.
\]
\end{lemma}

\begin{proof}
Fix a history and abbreviate
\[
p(x)=p_0(x\mid s),
\qquad
\ell(x)=(1-\epsilon)p_1(x\mid s),
\qquad
\mathcal A=\mathcal A(s).
\]
First suppose \(\epsilon>0\). The function
\[
c\longmapsto\sum_{x\in\mathcal A}\max\{\ell(x),cp(x)\}
\]
is continuous, equals \(1-\epsilon<1\) at \(c=0\), and is at
least one at \(c=1\). Hence some \(c\in(0,1]\) makes this sum
equal to one. Define
\[
q^\star(x)=\max\{\ell(x),cp(x)\}
\quad(x\in\mathcal A),
\qquad
q^\star(x)=0
\quad(x\notin\mathcal A).
\]
By \eqref{eq:one-sided-contamination-envelope}, this is an
admissible contaminated distribution.

Write \(r(x)=p(x)/q^\star(x)\) on \(\mathcal A\).
Then \(r(x)\le1/c\), with equality whenever
\(q^\star(x)>\ell(x)\). Moreover,
\[
\sum_{x\in\mathcal A}(q^\star(x)-\ell(x))=\epsilon.
\]
For any admissible \(Q\), its mass function satisfies
\(q(x)\ge\ell(x)\) and
\(\sum_{x\in\mathcal A}(q(x)-\ell(x))\le\epsilon\).
Therefore,
\begin{align*}
\sum_{x\in\mathcal A}q(x)r(x)
&=
\sum_{x\in\mathcal A}\ell(x)r(x)
+
\sum_{x\in\mathcal A}(q(x)-\ell(x))r(x)\\
&\le
\sum_{x\in\mathcal A}\ell(x)r(x)+\epsilon/c\\
&=
\sum_{x\in\mathcal A}q^\star(x)r(x)
=1.
\end{align*}
This proves \eqref{eq:clipped-likelihood-exponential}, since
\(e^{-Z_{P_0}}=r\) on the common support and is zero outside it.

We now show that the constructed \(q^\star\) is the unique minimizer using the same inequality. Using
\(-\log u\ge1-u\),
\begin{align*}
D_{\mathrm{KL}}(p\|q)-D_{\mathrm{KL}}(p\|q^\star)
&=
-\sum_{x\in\mathcal A}p(x)
\log\frac{q(x)}{q^\star(x)}\\
&\ge
1-\sum_{x\in\mathcal A}
q(x)\frac{p(x)}{q^\star(x)}
\ge0.
\end{align*}
Equality requires \(q(x)=q^\star(x)\) on every supported
coordinate. These coordinates already have total mass one,
so the minimizer is unique.

It remains to handle \(\epsilon=0\) and establish the conditional bound. When \(\epsilon=0\), the only admissible distribution is \(p_1\), so \(q^\star=p_1\). Choosing
$$
c=\min_{x\in\mathcal A}\frac{p_1(x\mid s)}{p(x)}>0
$$
ensures \(cp(x)\le p_1(x\mid s)\) on \(\mathcal A\), so the formula \(q^\star(x)=\max\{p_1(x\mid s),cp(x)\}\) remains valid. The exponential-moment bound holds with equality because
$$
\sum_{x\in\mathcal A}p_1(x\mid s)\frac{p(x)}{q^\star(x)}
=\sum_{x\in\mathcal A}p(x)=1.
$$

 To obtain the conditional bound, fix the observed history \(\mathcal F_{i-1}\). The next-token distribution belongs to \(\mathcal Q_{1,\epsilon}(S_{i-1})\), and \eqref{eq:clipped-likelihood-exponential} holds for every distribution in this neighborhood. Therefore,
$$
\mathbb E\!\left[
e^{-Z_{P_0}(X_i\mid S_{i-1})}
\mid\mathcal F_{i-1}
\right]\le1,
$$
even when contamination depends on the entire observed history.
\end{proof}

The next lemma reduces the composite human class to finitely
many representative kernels. The representatives and the threshold
may depend on \(\epsilon\), but are fixed before observing the text
and do not depend on its length.

\begin{lemma}[Finite human representatives]
\label{lem:finite-clipped-representatives}
Under the same assumptions, suppose
\(D_0>-\log(1-\epsilon)\). There exist
\(P_0^{(1)},\ldots,P_0^{(N)}\in\mathcal P_0\) and \(t>0\)
such that, for every \(P_0\in\mathcal P_0\), at least one
\(j\in\{1,\ldots,N\}\) satisfies
\begin{equation}
\sum_s\pi_{P_0}(s)
\mathbb E_{P_0(\cdot\mid s)}
[Z_{P_0^{(j)}}(X\mid s)]
\le -2t.
\label{eq:finite-clipped-human-means}
\end{equation}
\end{lemma}

\begin{proof}
For a fixed candidate \(P\in\mathcal P_0\), define
\[
d(P):=
\sum_s\pi_P(s)
D_{\mathrm{KL}}\bigl(
P(\cdot\mid s)\|Q_P^\star(\cdot\mid s)
\bigr).
\]
By Lemma~\ref{lem:whole-kernel-overlap}, \(D_0>-\log(1-\epsilon)\)
implies that at least one row of \(P\) lies outside its
contamination class. The corresponding divergence is positive.
Every state has positive stationary probability, so \(d(P)>0\).
Thus the mean of \(Z_P\) under its own human source is
\[
\sum_s\pi_P(s)\mathbb E_{P(\cdot\mid s)}[Z_P(X\mid s)]
=-d(P)<0.
\]

For fixed \(Z_P\), its stationary mean is continuous in the human kernel \(P_0\) and equals \(-d(P)<0\) at \(P_0=P\). Hence there is a relatively open neighborhood \(U_P\) of \(P\) in \(\mathcal P_0\) such that
$$
\sum_s\pi_{P_0}(s)
\mathbb E_{P_0(\cdot\mid s)}[Z_P(X\mid s)]
<-d(P)/2
\qquad(P_0\in U_P).
$$

These neighborhoods cover the compact class \(\mathcal P_0\).
Choose a finite subcover with centers
\(P_0^{(1)},\ldots,P_0^{(N)}\), and set
\[
t:=\frac14\min_{1\le j\le N}d(P_0^{(j)})>0.
\]
Every human kernel belongs to at least one selected
neighborhood. For its corresponding center, the stationary
mean is less than \(-d(P_0^{(j)})/2\le-2t\), proving
\eqref{eq:finite-clipped-human-means}.
\end{proof}

\subsection{Proof of Theorem~\ref{thm:clean-centered-two-sided-markov}}

\begin{proof}[Proof of Theorem~\ref{thm:clean-centered-two-sided-markov}(1)]
By Lemma~\ref{lem:whole-kernel-overlap}, the condition \(D_0\le-\log(1-\epsilon)\) implies that some \(P_0^\star\in\mathcal P_0\) satisfies
\(P_0^\star(\cdot\mid s)\in\mathcal Q_{1,\epsilon}(s)\) at every history. An admissible contamination strategy can therefore reproduce this human kernel. Since both hypotheses share the initial history, their induced text distributions coincide for every \(n\). Consequently, for every test \(A_n\),
$$
\alpha_n^{\mathrm{rb}}(A_n)+\beta_n^{\mathrm{rb}}(A_n)
\ge (P_0^\star)^{(n)}(A_n^c)+(P_0^\star)^{(n)}(A_n)=1.
$$
Thus, whenever \(\alpha_n^{\mathrm{rb}}(A_n)\to0\), we have \(\beta_n^{\mathrm{rb}}(A_n)\to1\), and hence \(E_{\mathrm{rb}}=0\).
\end{proof}

\begin{proof}[Proof of
Theorem~\ref{thm:clean-centered-two-sided-markov}(2)]
Choose the representatives and \(t>0\) from
Lemma~\ref{lem:finite-clipped-representatives}, and write
\(Z_j:=Z_{P_0^{(j)}}\). Define
\begin{equation}
A_n^\star
:=
\left\{
\min_{1\le j\le N}
\frac1n\sum_{i=1}^n Z_j(X_i\mid S_{i-1})
\le-t
\right\}.
\label{eq:clipped-family-test}
\end{equation}
Thus the test accepts the human hypothesis when at least
one representative provides sufficiently strong human evidence.

\emph{Human rejection probability.}
For each representative, let
\[
m_j(P_0):=
\sum_s\pi_{P_0}(s)
\mathbb E_{P_0(\cdot\mid s)}[Z_j(X\mid s)].
\]
The scores are uniformly bounded on the common support:
\[
\log((1-\epsilon)m_1)
\le Z_j(x\mid s)\le-\log m_0.
\]
For every fixed \(j\) and \(u>0\), we claim that
\begin{equation}
\sup_{P_0\in\mathcal P_0}
P_0^{(n)}
\left\{
\left|
\frac1n\sum_{i=1}^n Z_j(X_i\mid S_{i-1})-m_j(P_0)
\right|>u
\right\}
\longrightarrow0.
\label{eq:clipped-score-human-convergence}
\end{equation}

To verify this, define
\[
g_{j,P_0}(s)
:=
\mathbb E_{P_0(\cdot\mid s)}[Z_j(X\mid s)]
\]
and decompose the centered average as
\begin{align*}
&\frac1n\sum_{i=1}^n Z_j(X_i\mid S_{i-1})-m_j(P_0)\\
&\quad=
\frac1n\sum_{i=1}^n
\bigl(Z_j(X_i\mid S_{i-1})-g_{j,P_0}(S_{i-1})\bigr)\\
&\qquad+
\sum_s g_{j,P_0}(s)
\left(
\frac1n\sum_{i=1}^n\mathbf 1\{S_{i-1}=s\}
-\pi_{P_0}(s)
\right).
\end{align*}
The first term is an average of martingale differences
with a common bound, so Azuma--Hoeffding gives uniform
convergence to zero. The functions \(g_{j,P_0}\) also
have a common bound. Applying
Lemma~\ref{lem:uniform-ergodicity} to each state indicator
therefore makes the second term converge uniformly to zero.
This proves \eqref{eq:clipped-score-human-convergence}.

For every human kernel, at least one representative satisfies
\(m_j(P_0)\le-2t\). If the test rejects, that representative's
average exceeds \(-t\), and hence exceeds its stationary
mean by more than \(t\). Consequently,
\[
\alpha_n^{\mathrm{rb}}(A_n^\star)
\le
\sum_{j=1}^N
\sup_{P_0\in\mathcal P_0}
P_0^{(n)}
\left\{
\frac1n\sum_{i=1}^n Z_j(X_i\mid S_{i-1})
-m_j(P_0)>t
\right\}
\longrightarrow0.
\]

\emph{Contaminated-machine acceptance probability.}
For each \(j\), define
\[
L_{j,n}
:=
\exp\left\{
-\sum_{i=1}^n Z_j(X_i\mid S_{i-1})
\right\},
\qquad L_{j,0}=1.
\]
An off-support observation sets \(L_{j,n}\) to zero.
Lemma~\ref{lem:clipped-likelihood-projection} gives
\[
\mathbb E[L_{j,n}\mid\mathcal F_{n-1}]
\le L_{j,n-1}
\]
under every admissible contamination strategy. Iterating
this inequality yields \(\mathbb E[L_{j,n}]\le1\).
Markov's inequality therefore gives
\[
Q_{1,\mathbf B}^{(n)}
\left\{
\frac1n\sum_{i=1}^n Z_j(X_i\mid S_{i-1})\le-t
\right\}
=
Q_{1,\mathbf B}^{(n)}\{L_{j,n}\ge e^{nt}\}
\le e^{-nt}.
\]
Acceptance requires this event for at least one representative.
A union bound gives
\[
\beta_n^{\mathrm{rb}}(A_n^\star)\le N e^{-nt}.
\]
Taking \(c=t/2\), the right side is at most \(e^{-cn}\)
for all sufficiently large \(n\). Together with the human
error bound, this proves part~(2) and
\(E_{\mathrm{rb}}\ge c>0\).
\end{proof}

\section{Proofs and simulations for Section~\ref{sec:general-clipped-tests}}
\label{app:score-tail-applicability}

\subsection{Proof of Theorem~\ref{thm:one-sided-clipping-improvement}}
\label{app:auxiliary-clipping}
\label{app:proof-clipping}

We first establish two auxiliary lemmas, then prove the main result.
Throughout, fix a clipping width \(r\) in the range specified in
Section~\ref{sec:general-clipped-tests} and let \(\rho_0\) be defined by
\eqref{eq:human-clipping-mass}.
In fact, \(\rho_0<1\) follows from the source assumptions rather than
requiring an additional assumption. Choose a pair \((s_*,x_*)\) with
\(\phi(x_*\mid s_*)=\overline\phi\). Since
\(\underline\phi+r<\overline\phi\), this pair is not clipped.
Appendix~\ref{app:uniform-ergodicity} gives a common integer \(L\)
such that \(\pi_{P_0}(s_*)\ge m_0^L\), while
\(p_0(x_*\mid s_*)\ge m_0\). Every human source therefore assigns
at least \(m_0^{L+1}\) stationary probability to an unclipped pair,
so \(\rho_0\le1-m_0^{L+1}<1\).

The constant \(\mu_1\) is the center of the clean conditional-mean bound,
not a stationary target mean. For a fixed \(\epsilon\in[0,1)\), define
\[
\begin{aligned}
M_0(a)&:=\sup_{P_0\in\mathcal P_0}\sum_s\pi_{P_0}(s)
 \mathbb E_{P_0(\cdot\mid s)}[\phi^{(a)}(X\mid s)],\\
M_{1,\epsilon}(a)&:=(1-\epsilon)\inf_s
 \mathbb E_{P_1(\cdot\mid s)}[\phi^{(a)}(X\mid s)]+\epsilon a,\\
\gamma(a)&:=M_{1,\epsilon}(a)-M_0(a),
\qquad \underline\phi\le a<\overline\phi.
\end{aligned}
\]
Here \(M_0(a)\) is the largest stationary human mean after clipping,
while \(M_{1,\epsilon}(a)\) is a lower bound on the conditional mean
under any admissible contamination strategy. Their difference
\(\gamma(a)\) measures the guaranteed separation: when it is positive,
a threshold between the two bounds separates human and contaminated
machine text asymptotically. At \(a=\underline\phi\), no clipping is
applied and \(M_0(\underline\phi)=\mu_0\).

\begin{lemma}
\label{lem:human-score-averages}
Under Assumptions~\ref{ass:composite-markov-model}
and~\ref{ass:common-support-ergodicity}, let \(f(x\mid s)\) be any fixed
bounded score on the common support. For every \(u>0\),
\[
\sup_{S_0\in\mathcal S}\sup_{P_0\in\mathcal P_0}
P_0^{(n)}
\left(
\left|T_n(f)-\sum_s\pi_{P_0}(s)
\mathbb E_{P_0(\cdot\mid s)}[f(X\mid s)]\right|>u
\right)\longrightarrow0.
\]
\end{lemma}

\begin{proof}
Put \(g_{P_0}(s)=\mathbb E_{P_0(\cdot\mid s)}[f(X\mid s)]\).
Decompose the centered score average as
\[
\begin{aligned}
&T_n(f)-\sum_s\pi_{P_0}(s)g_{P_0}(s)\\
&\quad=\frac1n\sum_{i=1}^n
 \bigl[f(X_i\mid S_{i-1})-g_{P_0}(S_{i-1})\bigr]\\
&\qquad\quad+\sum_s g_{P_0}(s)
 \left[\frac1n\sum_{i=1}^n\mathbf 1\{S_{i-1}=s\}-\pi_{P_0}(s)\right].
\end{aligned}
\]
The first term measures the fluctuation of each token score around its
conditional expectation. Choose \(C>0\) such that \(|f(x\mid s)|\le C\)
on the common support. Then \(|g_{P_0}(s)|\le C\) for every kernel and
state. Each difference \(f(X_i\mid S_{i-1})-g_{P_0}(S_{i-1})\) has
conditional mean zero given \(\mathcal F_{i-1}\) and absolute value at
most \(2C\). Azuma's inequality gives, for every \(u>0\),
\[
P_0^{(n)}\left\{
\left|\frac1n\sum_{i=1}^n
[f(X_i\mid S_{i-1})-g_{P_0}(S_{i-1})]\right|>u
\right\}
\le 2\exp\left\{-\frac{nu^2}{8C^2}\right\}.
\]
This bound is independent of the human kernel and initial history, so
the first term converges to zero uniformly over both.

The second term measures the effect of fluctuations in the frequencies
of the history states. Write
\(\widehat\pi_n(s)=n^{-1}\sum_{i=1}^n\mathbf 1\{S_{i-1}=s\}\).
Although the weights \(g_{P_0}(s)\) depend on the kernel, their common
bound gives
\[
\left|\sum_s g_{P_0}(s)
[\widehat\pi_n(s)-\pi_{P_0}(s)]\right|
\le C|\mathcal S|\max_s
|\widehat\pi_n(s)-\pi_{P_0}(s)|.
\] For each fixed state \(s\),
Lemma~\ref{lem:uniform-ergodicity}, applied to the indicator of that
state, gives uniform convergence of \(\widehat\pi_n(s)\) to
\(\pi_{P_0}(s)\). Since \(\mathcal S\) is finite, for every \(v>0\),
\[
\begin{aligned}
&\sup_{S_0\in\mathcal S}\sup_{P_0\in\mathcal P_0}
P_0^{(n)}\left\{\max_s
|\widehat\pi_n(s)-\pi_{P_0}(s)|>v\right\}\\
&\quad\le\sum_{s\in\mathcal S}
\sup_{S_0\in\mathcal S}\sup_{P_0\in\mathcal P_0}
P_0^{(n)}\left\{
|\widehat\pi_n(s)-\pi_{P_0}(s)|>v\right\}
\longrightarrow0.
\end{aligned}
\]
Taking \(v=u/(C|\mathcal S|)\) proves uniform convergence of the second
term. Applying the two bounds proves the claimed
convergence of their sum. Both arguments allow an arbitrary initial
history; a stationary initial distribution is not required.

\end{proof}

\begin{lemma}
\label{lem:clipped-consistency}
Under Assumptions~\ref{ass:composite-markov-model},
\ref{ass:common-support-ergodicity},
\ref{ass:score-variation-tail}, and~\ref{ass:human-score-separation},
for every \(\underline\phi\le a\le \underline\phi+r\),
\[
M_0(a)\le\mu_0+\rho_0(a-\underline\phi),
\]
and
\begin{equation}
\gamma(a)-\gamma(\underline\phi)\ge(\epsilon-\rho_0)(a-\underline\phi),
\qquad
\gamma(\underline\phi)\ge(1-\epsilon)(\mu_1-\delta)+\epsilon \underline\phi-\mu_0.
\label{eq:clipping-gain}
\end{equation}
\end{lemma}

\begin{proof}
For \(\underline\phi\le a\le\underline\phi+r\), the definition of
clipping gives, on the common support,
\[
0\le\phi^{(a)}(x\mid s)-\phi(x\mid s)
\le(a-\underline\phi)
\mathbf 1\{\phi(x\mid s)<\underline\phi+r\}.
\]
Taking stationary human expectations and using the definition of
\(\rho_0\) in \eqref{eq:human-clipping-mass}, for every \(P_0\in\mathcal P_0\),
\[
\begin{aligned}
&\sum_s\pi_{P_0}(s)
\mathbb E_{P_0}[\phi^{(a)}-\phi\mid s]\\
&\le(a-\underline\phi)\sum_s\pi_{P_0}(s)
P_0\{\phi<\underline\phi+r\mid s\}\\
&\le \rho_0(a-\underline\phi).
\end{aligned}
\]
Adding the unclipped human mean and then taking the supremum over
\(P_0\) gives
\[
M_0(a)\le\mu_0+\rho_0(a-\underline\phi).
\]

Since \(\phi^{(a)}\ge\phi\) on the support,
\(M_{1,\epsilon}(a)-M_{1,\epsilon}(\underline\phi)\ge\epsilon(a-\underline\phi)\).
Combining this with the human mean bound above gives the
first inequality in \eqref{eq:clipping-gain}. The second inequality is a immediate result from the condition in Assumption~\ref{ass:score-variation-tail}.

\end{proof}

\begin{proof}[Proof of Theorem~\ref{thm:one-sided-clipping-improvement}(1)]
Fix \(0\le\epsilon<\epsilon_+\) and set \(a=\underline\phi+r\).
Lemma~\ref{lem:clipped-consistency} gives
\[
\begin{aligned}
\gamma(\underline\phi+r)
&\ge (1-\epsilon)(\mu_1-\delta)+\epsilon(\underline\phi+r)-\mu_0-\rho_0r\\
&\ge(\mu_1-\delta-\underline\phi-r)(\epsilon_+-\epsilon)>0.
\end{aligned}
\]
Therefore $M_{1,\epsilon}(a)>M_0(a)$. Choose \(\theta_a=\{M_0(a)+M_{1,\epsilon}(a)\}/2\), and let
\(A_n^{(a)}=\{T_n(\phi^{(a)})\le\theta_a\}\).

Under the human hypothesis, every stationary clipped mean
is at most \(M_0(a)<\theta_a\). Applying
Lemma~\ref{lem:human-score-averages} to the clipped score shows that
\(\alpha_n^{\mathrm{rb}}(A_n^{(a)})\to0\).

To bound \(\beta_n^{\mathrm{rb}}(A_n^{(a)})\), extend the clipped score to a finite surrogate
\(\widetilde\phi^{(a)}\) by assigning value \(\overline\phi\) off the support.
Under every adaptive contamination strategy, with
\(Y_i=\widetilde\phi^{(a)}(X_i\mid S_{i-1})\),
\[
\mathbb E[Y_i\mid\mathcal F_{i-1}]
\ge(1-\epsilon)\mathbb E_{P_1(\cdot\mid S_{i-1})}
 [\phi^{(a)}(X\mid S_{i-1})]+\epsilon a
\ge M_{1,\epsilon}(a).
\]
The martingale differences
\(Y_i-\mathbb E[Y_i\mid\mathcal F_{i-1}]\) have absolute value at most
\(\overline\phi-a\). Since
\(M_{1,\epsilon}(a)-\theta_a=\gamma(a)/2\), Azuma's inequality gives
\[
Q_{1,\mathbf B}^{(n)}
\left\{\frac1n\sum_{i=1}^nY_i\le\theta_a\right\}
\le \exp\left\{-\frac{n\gamma(a)^2}{8(\overline\phi-a)^2}\right\}.
\]
Since \(\widetilde\phi^{(a)}\le\phi^{(a)}\),
\[
A_n^{(a)}
=\{T_n(\phi^{(a)})\le\theta_a\}
\subseteq
\left\{\frac1n\sum_{i=1}^nY_i\le\theta_a\right\}.
\]
Taking the supremum over contamination strategies therefore gives
\[
\beta_n^{\mathrm{rb}}(A_n^{(a)})
\le \exp\left\{-\frac{n\gamma(a)^2}{8(\overline\phi-a)^2}\right\}
\longrightarrow0.
\]
\end{proof}

\begin{proof}[Proof of Theorem~\ref{thm:one-sided-clipping-improvement}(2)]
Fix \(\epsilon_-<\epsilon<1\). By the definition of \(\epsilon_-\),
\[
R_\epsilon:=(1-\epsilon)(\mu_1+\delta)+\epsilon(\underline\phi+\eta)<\mu_0.
\]

Consider any deterministic raw upper-threshold test
\(A_n^{\mathrm{raw}}=\{T_n(\phi)\le\theta_n\}\) satisfying
\(\alpha_n^{\mathrm{rb}}(A_n^{\mathrm{raw}})\to0\).

We first obtain a lower bound on \(\theta_n\).

Since \(R_\epsilon<\mu_0\), the definition of \(\mu_0\) gives a human
kernel \(P_{0,*}\in\mathcal P_0\) with stationary mean
\[
m_*:=\sum_s\pi_{P_{0,*}}(s)
\mathbb E_{P_{0,*}}[\phi(X\mid s)\mid s]>R_\epsilon.
\]
Set \(u=(m_*-R_\epsilon)/3>0\). By
Lemma~\ref{lem:human-score-averages},
\[
P_{0,*}^{(n)}\{T_n(\phi)>m_*-u\}\to1.
\]
If \(\theta_n\le m_*-u\) for infinitely many \(n\), then along those indices,
\[
\alpha_n^{\mathrm{rb}}(A_n^{\mathrm{raw}})
\ge P_{0,*}^{(n)}\{T_n(\phi)>\theta_n\}
\ge P_{0,*}^{(n)}\{T_n(\phi)>m_*-u\}\to1,
\]
contradicting \(\alpha_n^{\mathrm{rb}}(A_n^{\mathrm{raw}})\to0\).
Consequently, for all sufficiently large \(n\),
\[
\theta_n>m_*-u=R_\epsilon+2u.
\]

We now construct a contamination strategy whose conditional score mean is
at most \(R_\epsilon\). At each state choose
\(x_*(s)\in\operatorname*{arg\,min}_{x\in\mathcal A(s)}\phi(x\mid s)\).
Let \(\mathbf B_*\) place all contaminating mass on \(x_*(S_{i-1})\).
This is an admissible strategy supported on \(\mathcal A(S_{i-1})\).
Under this strategy, Assumption~\ref{ass:score-variation-tail} gives
\[
\begin{aligned}
\mathbb E[\phi(X_i\mid S_{i-1})\mid\mathcal F_{i-1}]
&=(1-\epsilon)\mathbb E_{P_1}[\phi(X\mid S_{i-1})\mid S_{i-1}]
  +\epsilon\phi(x_*(S_{i-1})\mid S_{i-1})\\
&\le(1-\epsilon)(\mu_1+\delta)
  +\epsilon(\underline\phi+\eta)
=R_\epsilon.
\end{aligned}
\]
Summing the conditional-mean bounds gives
\[
\begin{aligned}
&\sum_{i=1}^n\left(
\phi(X_i\mid S_{i-1})
-\mathbb E[\phi(X_i\mid S_{i-1})\mid\mathcal F_{i-1}]
\right)\\
&\qquad\ge\sum_{i=1}^n\phi(X_i\mid S_{i-1})-nR_\epsilon
=n\bigl(T_n(\phi)-R_\epsilon\bigr).
\end{aligned}
\]
Consequently,
\[
\{T_n(\phi)>R_\epsilon+u\}
\subseteq
\left\{
\sum_{i=1}^n\left(
\phi(X_i\mid S_{i-1})
-\mathbb E[\phi(X_i\mid S_{i-1})\mid\mathcal F_{i-1}]
\right)>nu
\right\}.
\]
The centered increments on the right are
martingale differences bounded in absolute value by
\(\overline\phi-\underline\phi\). Azuma's inequality gives
\[
Q_{1,\mathbf B_*}^{(n)}\{T_n(\phi)>R_\epsilon+u\}
\le\exp\left\{-\frac{nu^2}{2(\overline\phi-\underline\phi)^2}\right\}.
\]

Combining this concentration bound with \(\theta_n>R_\epsilon+2u\),
for all sufficiently large \(n\), gives
\[
\begin{aligned}
1\ge\beta_n^{\mathrm{rb}}(A_n^{\mathrm{raw}})
&\ge Q_{1,\mathbf B_*}^{(n)}\{T_n(\phi)\le\theta_n\}\\
&\ge1-\exp\left\{-\frac{nu^2}{2(\overline\phi-\underline\phi)^2}\right\}
\longrightarrow1.
\end{aligned}
\]
\end{proof}

\subsection{Discussion of the contamination interval}
\label{app:conditions}

We discuss when the two detection regimes overlap and the available clipping
levels.

\paragraph{A nonempty contamination interval.}
The following proposition gives a sufficient condition for the two regimes
in Theorem~\ref{thm:one-sided-clipping-improvement} to overlap. Intuitively,
this is favored by a clear score gap between clean machine and human text,
limited variation across histories, and clipping that affects few human tokens.

\begin{proposition}
\label{prop:nonempty-clipping}
Under Assumptions~\ref{ass:composite-markov-model},
\ref{ass:common-support-ergodicity},
\ref{ass:score-variation-tail}, and~\ref{ass:human-score-separation}, if
\begin{equation}
\frac{\mu_1-\mu_0}{\mu_1-\underline\phi}
>\rho_0+\frac{2\delta+\eta}{r},
\label{eq:clipping-nonempty}
\end{equation}
then the endpoints are well defined and
\(0<\epsilon_-<\epsilon_+\). In particular, the two regimes in
Theorem~\ref{thm:one-sided-clipping-improvement} have nonempty overlap.
\end{proposition}

\begin{proof}
Assumption~\ref{ass:human-score-separation} and the chosen range of \(r\)
give positive endpoint denominators and
\(0<\epsilon_-<1\). Moreover,
\[
\epsilon_--\frac{\mu_1-\mu_0}{\mu_1-\underline\phi}
=\frac{\delta(\mu_0-\underline\phi)+\eta(\mu_1-\mu_0)}
 {(\mu_1+\delta-\underline\phi-\eta)(\mu_1-\underline\phi)}\ge0.
\]
Combining this with \eqref{eq:clipping-nonempty} yields
\[
(\epsilon_--\rho_0)r>2\delta+\eta
\ge2(1-\epsilon_-)\delta+\epsilon_-\eta.
\]
Direct subtraction gives the identity
\begin{equation}
\begin{aligned}
&\mu_1-\delta-\mu_0-\rho_0r
-\epsilon_-(\mu_1-\delta-\underline\phi-r)\\
&\qquad=(\epsilon_--\rho_0)r
-2(1-\epsilon_-)\delta-\epsilon_-\eta.
\end{aligned}
\label{eq:endpoint-difference}
\end{equation}
Its right-hand side is positive by the preceding inequality. Dividing by
\(\mu_1-\delta-\underline\phi-r>0\) shows that the fraction defining
\(\epsilon_+\) exceeds \(\epsilon_-\). Since \(0<\epsilon_-<1\),
capping that fraction at one preserves the strict inequality. Thus
\(0<\epsilon_-<\epsilon_+\).
\end{proof}

Appendix~\ref{app:simulation} examines the condition in
Proposition~\ref{prop:nonempty-clipping} through finite-state simulations.
The condition holds for all four scores in the mild profile, illustrating
that it can be satisfied by history-dependent sources.

\paragraph{Other admissible clipping levels.}
Fix \(0\le\epsilon<\epsilon_+\). The clipping floor need not equal
\(\underline\phi+r\): a lower floor can also give consistency.
The proof of Theorem~\ref{thm:one-sided-clipping-improvement}(1) applies whenever
\(\underline\phi\le a\le \underline\phi+r\) and
\[
(1-\epsilon)(\mu_1-\delta)+\epsilon a
>\mu_0+\rho_0(a-\underline\phi).
\]
The left side bounds the contaminated score mean from below, and the right
side bounds the human score mean from above. Their difference is positive
at \(a=\underline\phi+r\), and its slope as \(a\) increases is
\(\epsilon-\rho_0\). This gives two cases.
\begin{enumerate}
\item \emph{\(\epsilon\le\rho_0\)}
Lowering the floor does not reduce this guaranteed separation. Hence any
\(a\in[\underline\phi,\underline\phi+r]\) satisfies the inequality,
including \(a=\underline\phi\), which leaves the score unclipped.
Thus, in this case, the raw test is already consistent, and clipping
anywhere in this band preserves consistency.

\item \emph{\(\epsilon>\rho_0\)}
Raising the floor increases the guaranteed contribution from contamination
more than the upper bound on the human mean. Solving the inequality gives
\[
a\in[\underline\phi,\underline\phi+r],
\qquad
a>\underline\phi+
\frac{\mu_0-(1-\epsilon)(\mu_1-\delta)-\epsilon\underline\phi}
{\epsilon-\rho_0}.
\]
The displayed lower bound is strictly below \(\underline\phi+r\), so
there is always an interval of admissible floors near the upper end of the
band. A lower floor gives less protection and is covered by this guarantee
only if it exceeds the displayed bound. The no-clipping choice is included
precisely when
\[
(1-\epsilon)(\mu_1-\delta)+\epsilon\underline\phi>\mu_0.
\]
\end{enumerate}
For each qualifying floor, the proof of part (1) supplies a fixed threshold
with \(\alpha_n^{\mathrm{rb}}(A_n^{(a)})\to0\) and
\(\beta_n^{\mathrm{rb}}(A_n^{(a)})\to0\). This also covers the interior clipping floors used in
Appendix~\ref{app:simulation}.

\subsection{Finite-state simulations}\label{app:simulation}

We use a finite-state simulation to show how the conditions in
Theorem~\ref{thm:one-sided-clipping-improvement} and
Proposition~\ref{prop:nonempty-clipping} can hold for history-dependent
sources and how clipping can improve detection. The accompanying code and
documentation describe the construction and implementation in detail.

\paragraph{Setup.}
We model human and machine text as order-\(K\) Markov chains on a vocabulary of \(10{,}000\) tokens, with \(K\in\{10,100\}\). At each history, both sources assign positive probability to the same \(M\in\{500,1000\}\) possible next tokens, but with different probabilities. This set changes with the most recent token, and all vocabulary tokens are reachable over time. The machine switches between two next-token probability distributions according to the preceding \(K\) tokens.

We consider three settings: mild heterogeneity, larger heterogeneity, and a tail-heavy human control. Within each conditional support, we divide the token space into 3 parts: common, intermediate, and rare token, who make up \(4\%\), \(46\%\), and \(50\%\) of the tokens respectively. The machine assigns these groups probability masses \((0.70,0.20,0.10)\) in regime zero. In regime one, the masses are \((0.695,0.203,0.102)\) under mild history variation and \((0.62,0.248,0.132)\) under larger variation. Both profiles use human masses \((0.15,0.70,0.15)\) in either regime. A third, tail-heavy human profile retains the mild machine probabilities but changes the human masses to \((0.45,0.10,0.45)\). These profiles let us examine the effects of history variation and human probability on rare tokens. Probabilities vary slightly within each group, and each profile specifies a single human kernel.

We evaluate log-likelihood, rank, log-rank, and entropy gap using the machine’s conditional probabilities. For each score, we compare its raw document average with the average after lower clipping. All four scores are oriented so that larger values favor machine text.

\paragraph{Checking the conditions.}
The known human and machine probabilities let us calculate the quantities
in the theorem exactly. For each detector, we compare a fixed set of clipping
bands and choose the one that gives the widest interval
$(\epsilon_-,\epsilon_+)$. When this interval is nonempty, we use its
midpoint as the contamination level and choose a clipping floor from the
range given in Appendix~\ref{app:conditions}. All choices are made from the
clean-source probabilities, before evaluating detection performance.

Table~\ref{tab:simulation-conditions} reports whether the chosen bands
satisfy Proposition~\ref{prop:nonempty-clipping} and whether
$\epsilon_-<\epsilon_+$. The interval is nonempty in $40$ of the $48$
configurations. With larger history variation, the proposition's condition
is not satisfied, yet the interval remains nonempty for all four scores.

\begin{table}[t]
\centering
\begin{tabular}{lcc}
\hline
Profile & Simpler sufficient inequality & Nonempty certified interval \\
\hline
Mild & $4/4$ & $4/4$ \\
Larger heterogeneity & $0/4$ & $4/4$ \\
Tail-heavy human control & $2/4$ & $2/4$ \\
\hline
\end{tabular}
\caption{Number of detectors satisfying each sufficient check, out of four,
for every $K\in\{10,100\}$ and $M\in\{500,1000\}$.}
\label{tab:simulation-conditions}
\end{table}

\paragraph{Power results.}
We test three ways of contaminating machine text. Two insert tokens with
the lowest detector score: one follows a fixed rule, while the other uses
the history to choose a token that makes regime zero occur next. The third
samples tokens from the human source. After clipping, the lower bound on the machine mean exceeds the clipped human mean in all $40$ configurations.

We generate $2{,}000$ test sequences and $4{,}000$ independent human sequences
for calibration, all starting from the all-zero history. We evaluate sequence
lengths $n\in\{128,512,2048,4096\}$. Each raw and clipped test has its own
threshold, calibrated to target a $5\%$ false-positive rate, and both tests
are evaluated on the same sequences.

Figure~\ref{fig:simulation-power} shows results for the mild profile under
the adaptive minimum-score attack. With $K=100$, $M=1000$, and $n=4096$,
the clipped tests detect between $99.85\%$ and $100\%$ of the contaminated
machine sequences, while the raw tests detect none. The observed
false-positive rates range from $4.05\%$ to $4.70\%$. Full numerical
results are included with the code.

\begin{figure}[t]
\centering
\includegraphics[width=\linewidth]{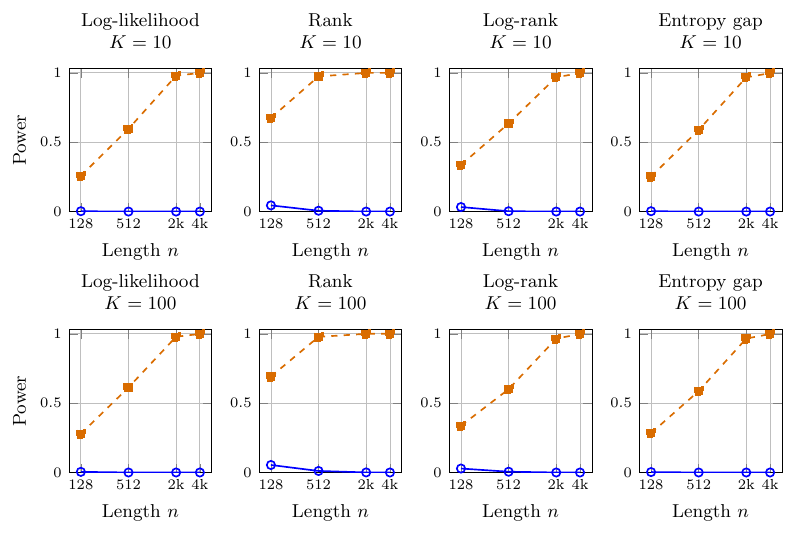}
\caption{Estimated detection power as sequence length increases, for the
mild profile with $M=1000$ under the adaptive minimum-score attack.
Blue solid curves show raw tests; orange dashed curves show clipped tests.
Each detector uses its selected contamination level, and each test is
calibrated to target a $5\%$ false-positive rate. The top and bottom rows
use $K=10$ and $K=100$, respectively.}
\label{fig:simulation-power}
\end{figure}

\section{Statistical methods}
\label{app:protocol}
This appendix explains how we fit and evaluate the detectors.
Appendix ~\ref{app:stat-splitting} describes how related texts are
kept in the same data split and how score directions are chosen.
Appendix~\ref{app:stat-clipping} gives the candidate clipping bounds
and the objectives used to select them in each study.
Appendix~\ref{app:stat-calibration} explains how human calibration
texts determine the detection thresholds.
Finally, Appendix~\ref{app:stat-bootstrap} describes how paired
source-level resampling gives confidence intervals for TPR, FPR,
and their changes after clipping.

\subsection{Data splitting and score direction}
\label{app:stat-splitting}
When splitting the dataset into tuning, calibration and testing, our basic data unit contains three things: a human document, its machine-generated counterpart, and all modified versions of that counterpart. We keep the whole group in the same split to prevent related texts from appearing in both fitting and testing. The tuning set determines the clipping bounds and score directions; the calibration set determines the thresholds; the test set measures performance. Modified machine texts are still labeled machine-generated.

For detectors whose direction is learned, we compare their mean scores on clean human and machine tuning texts. We use \(d=+1\) if the machine mean is at least as large, and \(d=-1\) otherwise. LRR and Binoculars use fixed directions \(+1\) and \(-1\), following their original designs \citep{su-etal-2023-detectllm,hans2024binoculars}. Directions do not change with the contamination rate.

As shown in Table~\ref{tab:primary-methods}, all detectors except entropy have the same direction across datasets. Entropy measures uncertainty in the model's next-token predictions,
rather than the probability assigned to the observed token.
There is therefore no general requirement that machine text have
higher or lower entropy than human text. In our clean tuning data, entropy's learned direction is positive
in some settings and negative in others
(Appendix~\ref{app:selected-directions}).

\subsection{Choosing clipping bounds}
\label{app:stat-clipping}

We choose candidate clipping bounds from token scores in the clean
human and machine tuning texts, using the quantile levels
$$
\mathcal Q=\{0.80,0.85,0.90,0.95,0.975,0.99,0.995\}.
$$
For additive scores, we clip the lower tail of the oriented token
scores, using the $(1-q)$-quantile as the bound.
For Binoculars, we cap $-\log p_i$ at its $q$-quantile.
LRR caps both $-\log p_i$ and $\log r_i$; seven choices for each
give 49 candidate pairs.

We also allow the scores to remain unclipped and choose this option
when it performs equally well. We exclude bounds that give every
clean tuning document the same score, and require LRR's denominator
cap to be positive.

Let \(\AUC(H,M;c)\) be the document-level AUROC of candidate \(c\), evaluated against human texts \(H\) and machine texts \(M\). In the controlled study, we maximize
$$
0.8\AUC(H,M_{\mathrm{mix}};c)
+0.2\AUC(H,M_{\mathrm{clean}};c).
$$
The mixture contains tuning texts at 10, 20, 30, 40, and 50\% replacement under both random and tail constructions. The 5\% condition is tested but not used in this objective. Since each source has three random variants and one tail variant per rate, random texts receive three times the total weight of tail texts. The selected bounds are then used at every rate and both FPR targets within that dataset--generator combination. 

For RAID, we select bounds using detection at a target 5\% FPR:
$$
0.8\,\widehat{\mathrm{TPR}}_{\mathrm{attacked,CV}}(c)
+0.2\,\widehat{\mathrm{TPR}}_{\mathrm{clean,CV}}(c).
$$
We divide tuning sources into five folds, balanced within each
domain, using a fixed random seed. All versions of a
source stay in the same fold. For each candidate, human texts in
four folds determine a separate threshold for each domain. We use
these thresholds to classify clean and attacked machine texts in
the remaining fold, and repeat until every fold has been held out.
We add the detected-text counts across folds and divide by the
corresponding total number of texts to obtain each cross-fitted TPR.
Every attacked text in the fitting set has the same weight; we do
not give equal weight to attack types with different sample sizes.

The candidate bounds and score directions are determined from
all clean tuning texts before this comparison. Thus, only threshold
estimation is cross-fitted. Neither the independent calibration set
nor the test set enters the candidate objective. There is no
additional FPR penalty or hard limit on the loss in clean AUROC.
The clean TPR term provides the clean-performance component of the
objective; it does not guarantee that clean test TPR is preserved.

Rate-specific fitting uses one of the intervals
\((0,0.05]\), \((0.05,0.10]\), \((0.10,0.20]\), and
\((0.20,0.50]\). Each interval must contain at least 250 attacked
tuning texts; otherwise fitting stops, without substituting a
universal bound. Attack-specific fitting uses all tuning texts
from one attack, including its measured-zero and greater-than-50\%
modifications. We fit one specification per detector and interval,
or per detector and attack, not an attack-by-rate grid.

For the secondary comparisons, full-universal fitting pools all
eleven attacks. Eligible-universal fitting pools attacked texts
with \(0<\rho\leq0.5\). Both use the same clean tuning references,
candidate bounds, folds, and objective as the specific fits.
The two universal rules are compared on the same test populations.
With seven detectors, the four scopes give 28 rate-specific,
77 attack-specific, seven full-universal, and seven eligible-universal
specifications.

\subsection{Setting detection thresholds}
\label{app:stat-calibration}
Thresholds are chosen using human calibration texts only. For target FPR \(\alpha\), sort their \(m\) scores increasingly and set \(k=\lfloor\alpha m\rfloor\). A text is classified as machine-generated if its score exceeds the \((m-k)\)-th score. This allows at most \(k\) false positives in calibration; ties can make the rule more conservative. Raw and clipped scores are calibrated separately.

The controlled study uses targets of 1\% and 5\% within each dataset--generator combination. RAID uses a 5\% target and calibrates separately in each domain. Overall RAID TPR and FPR are computed from all test decisions, so domains contribute according to their sample sizes. Test FPR may differ from the calibration target. We therefore report achieved FPR alongside TPR changes. 

\subsection{Confidence intervals}
\label{app:stat-bootstrap}
We use 2,000 paired source-cluster bootstrap resamples. Each resample draws test sources with replacement and keeps all texts from a drawn source together. Raw and clipped results use the same draw, and their difference is calculated before taking the 2.5th and 97.5th percentiles. The controlled study resamples within each dataset--generator combination; RAID resamples within domains, preserving domain sample sizes.

Directions, clipping bounds, and thresholds are held fixed. The intervals therefore describe variation in the test sample, not uncertainty from fitting the method. They are pointwise intervals without adjustment for multiple comparisons. Means across the nine controlled settings are descriptive summaries; we do not attach newly pooled confidence intervals to them.

For RAID, the attacked TPR pools the test texts belonging to the
reported attack, interval, or universal comparison population.
Across all eleven attacks, this also equals an equal-attack mean,
because each source has one version of every attack. Within a rate
interval or the eligible population, attack counts can differ, so
the pooled TPR need not equal an equal-attack mean. Each bootstrap
draw gives the same source weight to all of its included versions.
Human FPR counts each source's human document once.

Intervals for changes pair raw and clipped decisions within the
same reported comparison. They are not intervals for differences
between two fitted scopes or two selected specifications.
Each fit is evaluated on the same clean test population, and its
human FPR is repeated across that fit's attacked test conditions.
These repeated entries do not constitute additional human samples.
The reported FPRs describe the separate fitted rules, not a combined
system that automatically chooses a rate interval or attack label.

\section{Experimental details}
\label{app:data}
This appendix describes the texts and modifications used in the
experiments. Appendix~\ref{app:exp-datasets} introduces the datasets
and explains how human texts are selected and grouped.
Appendix~\ref{app:exp-models} specifies the generators, prompting
and decoding settings, and the text windows used for scoring.
Appendix~\ref{app:exp-contamination} explains how human donor spans
replace randomly located windows or the end of a machine
continuation in the controlled study.
Appendix~\ref{app:exp-raid} describes the selection of complete RAID
families, the eleven attack mechanisms, and the token-level edit
rate used to measure their realized modifications.

\subsection{Datasets and sample construction}
\label{app:exp-datasets}
The controlled study uses three English datasets with contrasting writing styles. XSum supplies news articles, SQuAD supplies Wikipedia contexts, and WritingPrompts supplies creative stories. Each is crossed with three generators. 

\paragraph{Dataset fields and filtering.}
We use news articles from the XSum validation set
\citep{narayan2018xsum}, Wikipedia passages from the SQuAD training
set \citep{rajpurkar2016squad}, and stories from the WritingPrompts
validation set \citep{fan2018writingprompts}.
For SQuAD, we remove repeated passages and combine short passages,
separated by blank lines, until they meet the minimum length
requirement. Each passage is used only once. Across all three
datasets, the resulting source texts must contain at least
260 words; shorter XSum articles and WritingPrompts stories
are excluded.

We assign 500 pairs to tuning, 500 to calibration, and 2,000 to
testing. Each pair and all its contaminated versions remain
together in the same split.

\paragraph{RAID domains.}
The labeled English RAID release supplies arXiv abstracts, book plot summaries, BBC news, PoemHunter poems, recipes, Reddit posts, IMDb reviews, and Wikipedia introductions \citep[Appendix~E.1]{dugan2024raid}. Unlike the controlled study, RAID supplies existing machine outputs and attacked versions; we do not regenerate them. Its family selection and final sample sizes are specified below.

\subsection{Models, generation, and scoring}
\label{app:exp-models}

The controlled study uses three base generators:
\begin{center}\small
\begin{tabular}{@{}ll@{}}
\toprule
Generator & Model identifier \\
\midrule
Granite 8B & \path{ibm-granite/granite-3.3-8b-base} \\
Mistral 24B & \path{mistralai/Mistral-Small-24B-Base-2501} \\
Qwen 32B & \path{Qwen/Qwen2.5-32B} \\
\bottomrule
\end{tabular}
\end{center}
The target tokenizer encodes the source without added special tokens. Its first nominal 30 tokens supply the prompt, and the following 220 supply the available human continuation. No chat template or task instruction is added. Decode/re-encode checks mean that the nominal prompt slice need not remain exactly 30 IDs. Generation uses BF16 with temperature 0.8, top-\(p\) 0.95, and a requested minimum of 210 and maximum of 220 new tokens. Other generation defaults are not inferred from these settings. EOS is removed, and human/machine arrays are cut to their common available length, at most 220, before contamination.

\paragraph{Scorers and context.}
The six single-model detectors in each controlled-study cell use that cell's generator distribution, before temperature and top-\(p\) transformations. Binoculars always uses Falcon-7B as observer and Falcon-7B-Instruct as performer \citep{hans2024binoculars}. Prompt and continuation are tokenized separately without added special tokens and concatenated. The prompt supplies context but contributes no score terms. Sequences are limited to 512 IDs by right-truncating the continuation; 220 generator tokens can have a different length under the Falcon tokenizer.

In RAID, we score each document without an additional prompt.
Falcon-7B provides the six single-model scores, and the Falcon-7B
observer and Falcon-7B-Instruct performer provide the Binoculars
score. We retain up to 512 tokens. The Binoculars numerator
averages next-token losses over up to 511 predicted tokens,
while its denominator averages cross-entropy over all non-padding
input positions. When EOS and padding share a token ID, EOS
positions are also excluded from the denominator.
In the controlled study, the prompt provides context, and both
components are averaged over the continuation tokens.

\subsection{Controlled contamination construction}
\label{app:exp-contamination}

We now explain how we replace parts of a machine continuation
with passages from its matched human continuation. We use two
constructions: random replacement changes passages at randomly
chosen locations, while tail replacement changes the end of the text.

We divide the human continuation into passages at sentence-ending
punctuation and line breaks, then tokenize each passage.
For a machine continuation of \(n\) tokens and replacement rate \(r\),
we replace approximately \(nr\) tokens, rounded to the nearest
integer. We use rates of 5, 10, 20, 30, 40, and 50\%.
The original machine continuation provides the 0\% baseline.
The replacements keep the token-array length unchanged, although
decoding and re-tokenizing the resulting text can change its length.

\paragraph{Random replacement.}
At each rate, we create three versions. For each version, we shuffle
the human passages and take enough tokens to fill the replacement
budget. We then place these passages in randomly located,
non-overlapping windows of the machine continuation. The last passage
is shortened if necessary, and human passages can be reused if more
tokens are needed.

\paragraph{Tail replacement.}
We use the generating model to score each human passage, given the
original prompt. Passages with larger mean \(-\log p_i\) are more
surprising to the model. We select these passages first and use them
to replace the final \(nr\) tokens of the machine continuation.
The passage ranking is computed once and reused across replacement
rates.

Each source contributes one clean human text, one clean machine text,
and four modified versions at each of six rates. This gives
\(2+4\times6=26\) texts per source and 78,000 texts per
dataset--generator combination. In the test set, each positive rate
has 6,000 random-replacement texts and 2,000 tail-replacement texts,
all derived from the same 2,000 sources.

We do not revise the text to repair grammar after replacement.
The recorded replacement rate is the number of donor tokens divided
by the final re-tokenized length. All versions of a source are kept
together when splitting the data and computing confidence intervals.

\subsection{RAID selection, attacks, and modification rates}
\label{app:exp-raid}

Now we explain our RAID experiments.  We use the machine outputs and attacked versions
already provided by the benchmark.

\paragraph{Sample selection.}
Our labeled English release contains 13,371 sources. We exclude
the 500 sources used in pilot experiments, leaving 12,871.
For each source, we require a human reference and at least one
complete machine family: a clean machine output together with
all eleven attacked versions. We randomly select one complete
family, giving each eligible family the same selection probability.
The attacked versions must belong to the selected output and
match its generator and generation settings.

We split sources within each domain, assigning 5,148 to tuning,
2,574 to calibration, and 5,149 to testing.
The human reference, clean machine output, and eleven attacked
versions remain in the same split. These thirteen texts per source
give 167,323 texts in total. Human calibration texts determine
the detection thresholds separately for each domain.
Appendix~\ref{app:raid} reports the sample composition by domain
and generator.

\paragraph{Attack mechanisms.}
Table~\ref{tab:raid-attack-mechanisms} describes the eleven attacks
provided by RAID \citep[Appendix~E.4]{dugan2024raid}. They include
changes to spelling, wording, capitalization, and spacing, as well
as sentence rewriting.

\begin{table}[htbp]
\centering
\small
\caption{RAID attacks and the changes they make.}
\label{tab:raid-attack-mechanisms}
\begin{tabular}{@{}p{0.29\linewidth}p{0.65\linewidth}@{}}
\toprule
Attack & Description \\
\midrule
Alternative spelling & Changes American spellings to British spellings. \\
Article deletion & Removes selected occurrences of a, an, or the. \\
Homoglyph & Replaces characters with visually similar Unicode characters. \\
Insert paragraphs & Adds paragraph breaks between selected sentences. \\
Number & Replaces digits in selected numbers with sampled digits. \\
Paraphrase & Rewrites the text using DIPPER-11B. \\
Perplexity misspelling & Introduces dictionary misspellings selected using GPT-2-small likelihood. \\
Synonym & Replaces selected words with contextually similar words. \\
Upper/lower & Changes the capitalization of selected initial letters. \\
Whitespace & Adds spaces between tokens; a location can be selected more than once. \\
Zero-width space & Inserts invisible U+200B characters between characters. \\
\bottomrule
\end{tabular}
\end{table}

\paragraph{Measuring modification.}
We compare each attacked output with its clean machine counterpart
using Falcon tokens. The comparison uses the scored portion of each
text after truncation, excluding the first token, which supplies
context. Write these token sequences as \(T_s(x)\) and
\(T_s(\widetilde x)\). We define
$$
\rho(x,\widetilde x)=
\frac{d_{\mathrm{Lev}}(T_s(x),T_s(\widetilde x))}
{\max\{|T_s(x)|,|T_s(\widetilde x)|\}}.
$$
Here \(d_{\mathrm{Lev}}\) is the smallest number of token insertions,
deletions, and substitutions needed to turn one sequence into the
other. Dividing by the longer sequence length gives a modification
rate between zero and one.

All attacked texts remain in attack-level testing, including those
with \(\rho=0\) or \(\rho>0.5\).

\paragraph{Information used for fitting and testing.}
Our main comparisons consider two kinds of information about the
modification. Rate-specific fitting assumes that a range for
\(\rho\) is available and learns a separate specification in each
of four intervals: \((0,0.05]\), \((0.05,0.10]\),
\((0.10,0.20]\), and \((0.20,0.50]\). Each specification is tested
on attacked texts in its corresponding interval.
Attack-specific fitting assumes that the attack mechanism is
available and learns one specification from each of the eleven
attacks, pooling that attack's entire modification-rate range.
Each is tested on the corresponding attack. We do not fit a
separate specification for every combination of attack and interval.

The benchmark provides the clean machine counterpart and the
attack label. The clean counterpart allows us to calculate \(\rho\);
such a counterpart may be unavailable for a new document in use.
These comparisons therefore assess the value of the specified
information, without estimating the rate interval or identifying
the attack from the observed document.

The secondary full-universal rule uses all attacked tuning texts.
The eligible-universal rule uses only those with
\(0<\rho\leq0.5\), excluding measured-zero and denser modifications
from its attacked fitting term. Each rule has one specification
per detector and is evaluated on the same full, eligible, attack,
and rate test populations. Clean machine texts enter the clean
objective term in all four scopes, and independent human
calibration texts set the final thresholds. The fitting procedure
is described in Appendix~\ref{app:stat-clipping}.

\section{Controlled contamination study: complete results and secondary diagnostics}
\label{app:primary}
% Canonical self-contained appendix: edit prose and tables here.

This appendix presents the detailed controlled-study results for
all seven detectors across the nine dataset--generator combinations.
The tables report TPR and FPR as percentages. To summarize performance
across contamination levels, we also report the normalized area under
the TPR--replacement-rate curve from 0\% to 50\%, scaled to 0--100.
Larger areas indicate better overall detection across these rates.
Random and tail results share the same clean baseline, and summary
averages give equal weight to each dataset--generator combination.

\subsection{Clean operating points and robustness at target 5\% FPR}
The normalized robustness area is \(R=(1/0.5)\int_0^{0.5}\TPR(r)\,dr\), evaluated by trapezoidal interpolation at requested rates 0, 5, 10, 20, 30, 40, and 50\%. It summarizes TPR versus replacement rate.
\begingroup\small\setlength{\tabcolsep}{4pt}
\begin{longtable}{llrrrrrr}
\toprule
Cell & Detector & \multicolumn{2}{c}{Clean TPR} & \multicolumn{2}{c}{Actual FPR} & \multicolumn{2}{c}{Area change} \\
& & Raw & Clip & Raw & Clip & Random & Tail \\
\midrule\endfirsthead
\toprule
Cell & Detector & \multicolumn{2}{c}{Clean TPR} & \multicolumn{2}{c}{Actual FPR} & \multicolumn{2}{c}{Area change} \\
& & Raw & Clip & Raw & Clip & Random & Tail \\
\midrule\endhead
\bottomrule\endfoot
XSum / Granite 8B & Log likelihood & 79.1 & 66.5 & 4.7 & 4.2 & +1.3 & +1.8 \\
 & Rank & 99.8 & 80.9 & 3.8 & 3.8 & +8.7 & +15.3 \\
 & Log rank & 96.0 & 83.5 & 4.7 & 4.0 & +1.6 & +2.1 \\
 & LRR & 99.4 & 90.0 & 5.2 & 4.6 & +22.3 & +21.9 \\
 & Entropy & 8.0 & 8.2 & 4.3 & 4.6 & +0.5 & +0.5 \\
 & Entropy gap & 99.9 & 99.1 & 6.7 & 5.0 & +30.3 & +9.0 \\
 & Binoculars & 97.2 & 97.7 & 5.1 & 5.2 & +43.7 & +24.4 \\
\addlinespace[3pt]
XSum / Mistral 24B & Log likelihood & 75.6 & 67.0 & 4.1 & 4.3 & +2.3 & +3.0 \\
 & Rank & 99.8 & 95.7 & 3.1 & 4.0 & +13.4 & +19.6 \\
 & Log rank & 94.0 & 82.7 & 4.3 & 4.3 & +3.2 & +3.9 \\
 & LRR & 99.5 & 80.5 & 4.7 & 6.3 & +26.0 & +17.9 \\
 & Entropy & 7.2 & 7.4 & 3.9 & 3.9 & +1.5 & +1.4 \\
 & Entropy gap & 99.7 & 98.3 & 3.5 & 5.5 & +38.8 & +20.2 \\
 & Binoculars & 98.8 & 98.1 & 3.9 & 5.6 & +41.5 & +24.3 \\
\addlinespace[3pt]
XSum / Qwen 32B & Log likelihood & 95.0 & 91.1 & 5.8 & 6.0 & +7.0 & +9.1 \\
 & Rank & 100.0 & 96.0 & 3.4 & 4.8 & +21.4 & +31.1 \\
 & Log rank & 99.5 & 96.5 & 4.6 & 5.7 & +7.7 & +10.6 \\
 & LRR & 100.0 & 86.4 & 5.9 & 6.4 & +22.8 & +17.3 \\
 & Entropy & 19.3 & 19.8 & 5.6 & 5.8 & +1.0 & +1.1 \\
 & Entropy gap & 97.1 & 95.9 & 4.3 & 5.5 & +39.9 & +25.2 \\
 & Binoculars & 99.0 & 98.6 & 5.0 & 4.7 & +33.8 & +18.3 \\
\addlinespace[3pt]
SQuAD / Granite 8B & Log likelihood & 2.9 & 2.8 & 5.9 & 5.9 & -0.0 & +0.1 \\
 & Rank & 2.8 & 2.8 & 5.9 & 5.9 & +0.0 & +0.0 \\
 & Log rank & 2.6 & 2.5 & 5.6 & 5.8 & -0.0 & +0.0 \\
 & LRR & 5.7 & 14.1 & 6.0 & 6.9 & +1.6 & +2.2 \\
 & Entropy & 9.2 & 9.2 & 6.4 & 6.4 & -0.0 & +0.0 \\
 & Entropy gap & 87.1 & 79.4 & 5.3 & 4.3 & +29.8 & +13.7 \\
 & Binoculars & 84.2 & 82.1 & 4.5 & 5.2 & +33.2 & +20.0 \\
\addlinespace[3pt]
SQuAD / Mistral 24B & Log likelihood & 3.1 & 3.0 & 4.3 & 4.3 & +0.2 & +0.4 \\
 & Rank & 27.8 & 6.8 & 4.1 & 4.8 & -0.9 & -0.6 \\
 & Log rank & 4.5 & 4.0 & 4.7 & 4.9 & +0.1 & +0.2 \\
 & LRR & 29.4 & 32.9 & 4.2 & 4.5 & +3.2 & +4.1 \\
 & Entropy & 5.3 & 5.3 & 6.9 & 6.9 & +0.0 & +0.0 \\
 & Entropy gap & 96.5 & 93.8 & 5.1 & 6.2 & +39.3 & +26.4 \\
 & Binoculars & 94.9 & 92.8 & 3.6 & 4.7 & +28.4 & +20.6 \\
\addlinespace[3pt]
SQuAD / Qwen 32B & Log likelihood & 7.3 & 7.1 & 7.0 & 7.5 & -0.0 & +0.5 \\
 & Rank & 6.4 & 6.4 & 6.7 & 6.7 & +0.0 & +0.0 \\
 & Log rank & 6.4 & 6.4 & 6.7 & 6.7 & +0.0 & +0.0 \\
 & LRR & 6.9 & 6.6 & 6.9 & 7.0 & +0.0 & +0.3 \\
 & Entropy & 8.9 & 8.9 & 7.0 & 7.0 & -0.0 & +0.0 \\
 & Entropy gap & 46.9 & 30.7 & 6.4 & 8.5 & +31.1 & +18.3 \\
 & Binoculars & 55.1 & 64.7 & 2.9 & 6.7 & +48.4 & +29.6 \\
\addlinespace[3pt]
WritingPrompts / Granite 8B & Log likelihood & 100.0 & 99.8 & 5.7 & 5.9 & +5.6 & +7.5 \\
 & Rank & 99.9 & 100.0 & 4.3 & 5.2 & +33.4 & +49.1 \\
 & Log rank & 100.0 & 100.0 & 6.0 & 5.8 & +6.5 & +9.4 \\
 & LRR & 100.0 & 99.2 & 3.5 & 3.9 & +27.9 & +27.6 \\
 & Entropy & 34.3 & 36.5 & 4.7 & 6.2 & +1.6 & +2.7 \\
 & Entropy gap & 99.9 & 98.8 & 5.0 & 4.8 & +27.7 & +8.7 \\
 & Binoculars & 99.8 & 99.8 & 5.6 & 4.2 & +22.5 & +11.4 \\
\addlinespace[3pt]
WritingPrompts / Mistral 24B & Log likelihood & 100.0 & 99.7 & 5.1 & 5.0 & +6.1 & +7.7 \\
 & Rank & 99.9 & 99.9 & 4.3 & 4.2 & +35.0 & +46.5 \\
 & Log rank & 100.0 & 99.9 & 4.6 & 4.7 & +7.6 & +10.3 \\
 & LRR & 100.0 & 99.4 & 4.1 & 4.1 & +29.5 & +26.8 \\
 & Entropy & 44.2 & 48.4 & 4.0 & 4.4 & +2.1 & +3.4 \\
 & Entropy gap & 99.8 & 97.5 & 5.5 & 4.7 & +27.1 & +13.6 \\
 & Binoculars & 99.9 & 99.9 & 5.0 & 3.6 & +18.6 & +10.8 \\
\addlinespace[3pt]
WritingPrompts / Qwen 32B & Log likelihood & 100.0 & 100.0 & 4.2 & 4.6 & +9.5 & +11.6 \\
 & Rank & 100.0 & 100.0 & 4.8 & 3.6 & +42.8 & +61.9 \\
 & Log rank & 100.0 & 100.0 & 4.7 & 3.9 & +11.0 & +15.4 \\
 & LRR & 100.0 & 100.0 & 3.5 & 4.1 & +31.5 & +30.8 \\
 & Entropy & 74.2 & 74.9 & 3.8 & 4.3 & +3.1 & +5.6 \\
 & Entropy gap & 97.2 & 94.5 & 5.1 & 5.2 & +29.3 & +20.1 \\
 & Binoculars & 99.5 & 99.2 & 4.8 & 3.5 & +21.5 & +12.8 \\
\end{longtable}

\endgroup
\Needspace{16\baselineskip}

\subsection{Secondary operating points at target 1\% FPR}
The next table shows detection performance at a target FPR of
1\%. With 500 human calibration texts, this allows at most five
false positives during calibration. We report achieved test FPR
alongside TPR because it can differ from the calibration target.
\begingroup\small\setlength{\tabcolsep}{4pt}
\begin{longtable}{llrrrrrr}
\toprule
Cell & Detector & \multicolumn{2}{c}{Clean TPR} & \multicolumn{2}{c}{Actual FPR} & \multicolumn{2}{c}{Area change} \\
& & Raw & Clip & Raw & Clip & Random & Tail \\
\midrule\endfirsthead
\toprule
Cell & Detector & \multicolumn{2}{c}{Clean TPR} & \multicolumn{2}{c}{Actual FPR} & \multicolumn{2}{c}{Area change} \\
& & Raw & Clip & Raw & Clip & Random & Tail \\
\midrule\endhead
\bottomrule\endfoot
XSum / Granite 8B & Log likelihood & 44.5 & 36.3 & 0.9 & 1.1 & +0.5 & +0.8 \\
 & Rank & 99.6 & 51.3 & 0.5 & 0.9 & +0.7 & +2.9 \\
 & Log rank & 74.9 & 51.9 & 1.1 & 0.9 & -0.5 & -0.6 \\
 & LRR & 94.1 & 74.1 & 0.8 & 0.7 & +12.5 & +14.9 \\
 & Entropy & 1.4 & 1.4 & 0.7 & 0.7 & +0.3 & +0.3 \\
 & Entropy gap & 99.7 & 98.7 & 2.4 & 1.8 & +36.6 & +14.3 \\
 & Binoculars & 94.5 & 93.3 & 1.0 & 2.0 & +41.8 & +25.9 \\
\addlinespace[3pt]
XSum / Mistral 24B & Log likelihood & 29.5 & 24.6 & 0.8 & 0.6 & +0.4 & +0.6 \\
 & Rank & 98.9 & 79.5 & 0.4 & 1.0 & +4.2 & +6.8 \\
 & Log rank & 56.0 & 45.1 & 0.8 & 0.9 & +1.2 & +1.8 \\
 & LRR & 98.2 & 64.8 & 1.8 & 1.9 & +12.1 & +7.3 \\
 & Entropy & 2.6 & 2.7 & 0.9 & 0.9 & -0.1 & -0.0 \\
 & Entropy gap & 99.2 & 97.4 & 0.9 & 1.9 & +44.3 & +27.1 \\
 & Binoculars & 97.5 & 95.0 & 1.0 & 1.6 & +42.5 & +28.0 \\
\addlinespace[3pt]
XSum / Qwen 32B & Log likelihood & 54.8 & 49.5 & 1.1 & 1.2 & +2.4 & +3.3 \\
 & Rank & 100.0 & 64.2 & 0.5 & 0.8 & +4.2 & +7.6 \\
 & Log rank & 84.2 & 66.5 & 0.9 & 0.9 & +1.5 & +2.4 \\
 & LRR & 99.0 & 75.3 & 1.2 & 2.1 & +18.3 & +19.0 \\
 & Entropy & 11.6 & 11.7 & 1.1 & 1.1 & +0.3 & +0.4 \\
 & Entropy gap & 95.3 & 93.5 & 1.2 & 1.5 & +37.3 & +27.7 \\
 & Binoculars & 97.6 & 96.8 & 1.1 & 1.3 & +32.1 & +20.2 \\
\addlinespace[3pt]
SQuAD / Granite 8B & Log likelihood & 0.4 & 0.4 & 0.9 & 0.9 & +0.0 & +0.0 \\
 & Rank & 0.7 & 0.7 & 1.4 & 1.4 & +0.0 & +0.0 \\
 & Log rank & 0.7 & 0.7 & 1.4 & 1.4 & +0.0 & +0.0 \\
 & LRR & 0.5 & 0.6 & 1.2 & 1.2 & +0.0 & +0.0 \\
 & Entropy & 1.7 & 1.7 & 1.1 & 1.1 & +0.0 & +0.0 \\
 & Entropy gap & 81.2 & 73.1 & 1.2 & 1.2 & +30.6 & +17.4 \\
 & Binoculars & 76.2 & 71.7 & 1.1 & 0.9 & +24.1 & +15.0 \\
\addlinespace[3pt]
SQuAD / Mistral 24B & Log likelihood & 1.1 & 1.1 & 1.5 & 1.5 & +0.0 & +0.1 \\
 & Rank & 1.1 & 1.1 & 1.0 & 1.1 & +0.0 & +0.0 \\
 & Log rank & 1.1 & 1.1 & 1.1 & 1.2 & +0.0 & +0.0 \\
 & LRR & 1.3 & 1.9 & 0.5 & 1.1 & +0.1 & +0.1 \\
 & Entropy & 1.6 & 1.6 & 2.5 & 2.5 & +0.0 & +0.0 \\
 & Entropy gap & 93.8 & 89.2 & 1.6 & 1.1 & +35.7 & +23.5 \\
 & Binoculars & 90.7 & 88.8 & 1.4 & 1.5 & +30.7 & +23.0 \\
\addlinespace[3pt]
SQuAD / Qwen 32B & Log likelihood & 2.2 & 2.2 & 2.1 & 2.1 & +0.0 & +0.0 \\
 & Rank & 0.0 & 0.0 & 0.0 & 0.0 & +0.0 & +0.0 \\
 & Log rank & 0.0 & 0.0 & 0.0 & 0.0 & +0.0 & +0.0 \\
 & LRR & 1.2 & 1.2 & 1.5 & 1.5 & +0.0 & +0.0 \\
 & Entropy & 2.8 & 2.8 & 1.9 & 1.9 & +0.0 & +0.0 \\
 & Entropy gap & 37.8 & 18.6 & 1.5 & 2.4 & +15.8 & +8.7 \\
 & Binoculars & 43.8 & 46.9 & 1.1 & 0.9 & +29.2 & +15.6 \\
\addlinespace[3pt]
WritingPrompts / Granite 8B & Log likelihood & 94.2 & 78.5 & 0.3 & 0.2 & -0.4 & -0.5 \\
 & Rank & 99.8 & 99.0 & 0.9 & 0.3 & +14.7 & +22.5 \\
 & Log rank & 99.9 & 96.4 & 0.6 & 0.6 & +1.3 & +2.7 \\
 & LRR & 99.9 & 97.5 & 0.8 & 0.9 & +20.8 & +23.6 \\
 & Entropy & 17.4 & 17.6 & 0.9 & 1.0 & +0.5 & +1.0 \\
 & Entropy gap & 99.6 & 98.5 & 1.1 & 1.6 & +35.8 & +16.7 \\
 & Binoculars & 99.5 & 99.5 & 0.8 & 1.2 & +26.0 & +17.2 \\
\addlinespace[3pt]
WritingPrompts / Mistral 24B & Log likelihood & 96.5 & 95.1 & 0.4 & 0.7 & +6.0 & +8.1 \\
 & Rank & 99.8 & 97.5 & 1.0 & 0.6 & +20.9 & +27.3 \\
 & Log rank & 99.7 & 97.2 & 0.4 & 0.6 & +5.3 & +7.5 \\
 & LRR & 100.0 & 98.2 & 0.9 & 1.5 & +26.5 & +26.4 \\
 & Entropy & 30.0 & 29.8 & 1.1 & 0.8 & +0.5 & +1.0 \\
 & Entropy gap & 98.7 & 95.5 & 0.9 & 0.9 & +35.3 & +21.2 \\
 & Binoculars & 99.8 & 99.5 & 0.4 & 0.7 & +25.4 & +18.3 \\
\addlinespace[3pt]
WritingPrompts / Qwen 32B & Log likelihood & 99.8 & 99.5 & 0.4 & 0.5 & +6.6 & +9.7 \\
 & Rank & 100.0 & 99.8 & 0.8 & 0.9 & +31.3 & +46.7 \\
 & Log rank & 100.0 & 99.2 & 0.4 & 1.0 & +11.5 & +18.8 \\
 & LRR & 100.0 & 99.7 & 0.7 & 1.2 & +29.0 & +30.9 \\
 & Entropy & 50.7 & 52.5 & 0.5 & 0.5 & +2.1 & +3.9 \\
 & Entropy gap & 95.1 & 91.2 & 1.1 & 1.4 & +29.7 & +21.1 \\
 & Binoculars & 98.6 & 97.5 & 0.5 & 0.9 & +26.3 & +19.9 \\
\end{longtable}

\endgroup
\Needspace{16\baselineskip}

\subsection{Detection across different thresholds}

The preceding results evaluate detection at thresholds calibrated
to a chosen FPR. Here we examine performance across a range of
thresholds using AUROC and partial AUROC.

AUROC summarizes the full ROC curve, which plots TPR against FPR as the threshold changes. Larger values indicate better separation between human and machine texts. The first table reports AUROC multiplied by 100. Clean results are averaged equally across the nine dataset--generator combinations. Random and tail results are also averaged equally across the six positive replacement rates, from 5\% to 50\%.
\par\medskip
\begingroup\small\centering
\begin{tabular}{lrrrrrr}
\toprule
& \multicolumn{2}{c}{Clean AUROC} & \multicolumn{2}{c}{Random AUROC} & \multicolumn{2}{c}{Tail AUROC} \\
Detector & Raw & Clip & Raw & Clip & Raw & Clip \\
\midrule
Log likelihood & 87.1 & 85.2 & 42.7 & 50.7 & 59.1 & 64.9 \\
Rank & 94.9 & 88.4 & 6.8 & 55.1 & 10.4 & 68.7 \\
Log rank & 89.3 & 87.3 & 46.7 & 54.9 & 62.1 & 68.2 \\
LRR & 95.2 & 86.9 & 59.9 & 81.1 & 69.0 & 82.0 \\
Entropy & 65.0 & 64.8 & 60.0 & 60.6 & 61.3 & 62.2 \\
Entropy gap & 97.4 & 94.6 & 71.0 & 94.0 & 82.2 & 92.5 \\
Binoculars & 97.6 & 97.5 & 78.8 & 95.1 & 84.2 & 92.6 \\
\bottomrule
\end{tabular}
\par
\endgroup
\medskip
Partial AUROC focuses on FPRs between 0\% and 5\%.
We divide the area over this range by 0.05 and multiply by 100,
so the reported value represents average TPR over the low-FPR
range. The averaging across settings and replacement rates is
the same as in the first table.

For Binoculars, clipping slightly reduces clean performance over
this range, from 90.2 to 87.3, but improves performance under
random replacement from 39.9 to 67.5 and under tail replacement
from 45.3 to 61.4. These results show that the contamination
benefit extends beyond a single detection threshold.
\par\medskip
\begingroup\small\centering
\begin{tabular}{lrrrrrr}
\toprule
& \multicolumn{2}{c}{Clean pAUROC} & \multicolumn{2}{c}{Random pAUROC} & \multicolumn{2}{c}{Tail pAUROC} \\
Detector & Raw & Clip & Raw & Clip & Raw & Clip \\
\midrule
Log likelihood & 53.9 & 50.9 & 10.4 & 13.0 & 14.9 & 18.2 \\
Rank & 67.7 & 59.1 & 0.5 & 16.7 & 0.6 & 22.6 \\
Log rank & 61.2 & 56.5 & 12.8 & 15.9 & 17.4 & 21.7 \\
LRR & 67.4 & 58.8 & 18.9 & 31.3 & 22.6 & 34.2 \\
Entropy & 18.0 & 18.1 & 17.1 & 17.6 & 12.7 & 13.6 \\
Entropy gap & 89.6 & 84.5 & 41.8 & 73.2 & 53.2 & 71.2 \\
Binoculars & 90.2 & 87.3 & 39.9 & 67.5 & 45.3 & 61.4 \\
\bottomrule
\end{tabular}
\par
\endgroup

\subsection{Clean detection and false-positive rates}

This table shows how clipping affects detection on unmodified text.
Raw and clipped thresholds are calibrated separately to target
5\% FPR. We report their achieved FPRs on human test texts and
the change in TPR on clean machine texts, with pointwise 95\%
confidence intervals from 2,000 paired source-cluster bootstrap
resamples.

FPR values are percentages. Clean \(\Delta\)TPR is clipped minus
raw TPR, measured in percentage points: positive values indicate
improved detection, and negative values indicate a loss.

\begingroup\small\setlength{\tabcolsep}{4pt}
\begin{longtable}{lrrr}
\toprule
Detector & Raw FPR (95\% CI) & Clip FPR (95\% CI) & Clean $\Delta$TPR (95\% CI) \\
\midrule\endfirsthead
\toprule
Detector & Raw FPR (95\% CI) & Clip FPR (95\% CI) & Clean $\Delta$TPR (95\% CI) \\
\midrule\endhead
\bottomrule\endfoot
\multicolumn{4}{l}{\textbf{XSum / Granite 8B}} \\*
Log likelihood & 4.7 (3.8, 5.6) & 4.2 (3.4, 5.1) & -12.7 (-14.2, -11.2) \\*
Rank & 3.8 (2.9, 4.7) & 3.8 (2.9, 4.6) & -18.9 (-20.7, -17.2) \\*
Log rank & 4.7 (3.8, 5.5) & 4.0 (3.2, 4.9) & -12.5 (-14.0, -11.1) \\*
LRR & 5.2 (4.2, 6.2) & 4.6 (3.7, 5.5) & -9.4 (-10.6, -8.1) \\*
Entropy & 4.3 (3.5, 5.2) & 4.6 (3.7, 5.5) & +0.2 (+0.1, +0.4) \\*
Entropy gap & 6.7 (5.5, 7.8) & 5.0 (4.0, 6.0) & -0.8 (-1.2, -0.4) \\*
Binoculars & 5.1 (4.2, 6.1) & 5.2 (4.2, 6.2) & +0.5 (-0.4, +1.3) \\
\addlinespace[3pt]
\multicolumn{4}{l}{\textbf{XSum / Mistral 24B}} \\*
Log likelihood & 4.1 (3.2, 5.0) & 4.3 (3.5, 5.2) & -8.6 (-9.8, -7.4) \\*
Rank & 3.1 (2.4, 3.9) & 4.0 (3.1, 4.8) & -4.1 (-5.0, -3.2) \\*
Log rank & 4.3 (3.5, 5.2) & 4.3 (3.5, 5.2) & -11.4 (-12.9, -10.0) \\*
LRR & 4.7 (3.8, 5.6) & 6.3 (5.2, 7.3) & -18.9 (-20.7, -17.2) \\*
Entropy & 3.9 (3.0, 4.7) & 3.9 (3.0, 4.8) & +0.3 (+0.0, +0.5) \\*
Entropy gap & 3.5 (2.6, 4.3) & 5.5 (4.5, 6.5) & -1.4 (-1.9, -0.9) \\*
Binoculars & 3.9 (3.1, 4.8) & 5.6 (4.6, 6.6) & -0.7 (-1.4, +0.1) \\
\addlinespace[3pt]
\multicolumn{4}{l}{\textbf{XSum / Qwen 32B}} \\*
Log likelihood & 5.8 (4.8, 6.8) & 6.0 (5.0, 7.0) & -3.8 (-4.7, -3.0) \\*
Rank & 3.4 (2.6, 4.2) & 4.8 (3.9, 5.8) & -3.9 (-4.7, -3.1) \\*
Log rank & 4.6 (3.8, 5.5) & 5.7 (4.7, 6.7) & -2.9 (-3.7, -2.2) \\*
LRR & 5.9 (5.0, 7.0) & 6.4 (5.3, 7.4) & -13.6 (-15.1, -12.1) \\*
Entropy & 5.6 (4.7, 6.7) & 5.8 (4.8, 6.8) & +0.5 (-0.1, +1.1) \\*
Entropy gap & 4.3 (3.5, 5.2) & 5.5 (4.5, 6.5) & -1.2 (-1.8, -0.7) \\*
Binoculars & 5.0 (4.0, 5.9) & 4.7 (3.8, 5.7) & -0.4 (-0.9, +0.2) \\
\addlinespace[3pt]
\multicolumn{4}{l}{\textbf{SQuAD / Granite 8B}} \\*
Log likelihood & 5.9 (4.9, 7.0) & 5.9 (4.9, 7.0) & -0.1 (-0.2, +0.0) \\*
Rank & 5.9 (4.8, 6.9) & 5.9 (4.8, 6.9) & +0.0 (+0.0, +0.0) \\*
Log rank & 5.6 (4.6, 6.6) & 5.8 (4.7, 6.8) & -0.1 (-0.2, +0.0) \\*
LRR & 6.0 (5.1, 7.1) & 6.9 (5.8, 8.1) & +8.4 (+7.0, +9.8) \\*
Entropy & 6.4 (5.3, 7.4) & 6.4 (5.3, 7.4) & +0.0 (+0.0, +0.0) \\*
Entropy gap & 5.3 (4.3, 6.4) & 4.3 (3.5, 5.2) & -7.7 (-8.9, -6.5) \\*
Binoculars & 4.5 (3.7, 5.5) & 5.2 (4.2, 6.2) & -2.2 (-3.3, -0.9) \\
\addlinespace[3pt]
\multicolumn{4}{l}{\textbf{SQuAD / Mistral 24B}} \\*
Log likelihood & 4.3 (3.5, 5.2) & 4.3 (3.5, 5.2) & -0.1 (-0.3, +0.0) \\*
Rank & 4.1 (3.2, 5.0) & 4.8 (3.9, 5.7) & -21.1 (-22.9, -19.3) \\*
Log rank & 4.7 (3.8, 5.7) & 4.9 (3.9, 5.8) & -0.4 (-0.8, -0.2) \\*
LRR & 4.2 (3.4, 5.1) & 4.5 (3.6, 5.5) & +3.5 (+1.6, +5.4) \\*
Entropy & 6.9 (5.8, 8.0) & 6.9 (5.8, 8.0) & +0.0 (+0.0, +0.0) \\*
Entropy gap & 5.1 (4.2, 6.1) & 6.2 (5.1, 7.2) & -2.8 (-3.5, -2.1) \\*
Binoculars & 3.6 (2.8, 4.5) & 4.7 (3.8, 5.7) & -2.0 (-2.8, -1.3) \\
\addlinespace[3pt]
\multicolumn{4}{l}{\textbf{SQuAD / Qwen 32B}} \\*
Log likelihood & 7.0 (5.9, 8.1) & 7.5 (6.5, 8.7) & -0.2 (-0.5, +0.1) \\*
Rank & 6.7 (5.7, 7.8) & 6.7 (5.7, 7.8) & +0.0 (+0.0, +0.0) \\*
Log rank & 6.7 (5.6, 7.8) & 6.7 (5.6, 7.8) & +0.0 (+0.0, +0.0) \\*
LRR & 6.9 (5.8, 7.9) & 7.0 (5.9, 8.0) & -0.3 (-0.6, -0.0) \\*
Entropy & 7.0 (5.9, 8.2) & 7.0 (5.9, 8.2) & +0.0 (+0.0, +0.0) \\*
Entropy gap & 6.4 (5.3, 7.5) & 8.5 (7.2, 9.7) & -16.2 (-17.9, -14.7) \\*
Binoculars & 2.9 (2.1, 3.6) & 6.7 (5.6, 7.8) & +9.6 (+7.7, +11.6) \\
\addlinespace[3pt]
\multicolumn{4}{l}{\textbf{WritingPrompts / Granite 8B}} \\*
Log likelihood & 5.7 (4.8, 6.7) & 5.9 (4.9, 6.9) & -0.2 (-0.3, +0.0) \\*
Rank & 4.3 (3.4, 5.1) & 5.2 (4.2, 6.2) & +0.1 (-0.1, +0.3) \\*
Log rank & 6.0 (5.0, 7.1) & 5.8 (4.8, 6.8) & -0.0 (-0.1, +0.0) \\*
LRR & 3.5 (2.8, 4.3) & 3.9 (3.0, 4.7) & -0.8 (-1.1, -0.4) \\*
Entropy & 4.7 (3.8, 5.6) & 6.2 (5.1, 7.3) & +2.2 (+1.4, +3.1) \\*
Entropy gap & 5.0 (4.0, 5.9) & 4.8 (3.9, 5.7) & -1.1 (-1.5, -0.6) \\*
Binoculars & 5.6 (4.7, 6.6) & 4.2 (3.4, 5.1) & -0.0 (-0.2, +0.1) \\
\addlinespace[3pt]
\multicolumn{4}{l}{\textbf{WritingPrompts / Mistral 24B}} \\*
Log likelihood & 5.1 (4.2, 6.0) & 5.0 (4.0, 5.9) & -0.3 (-0.6, -0.1) \\*
Rank & 4.3 (3.5, 5.2) & 4.2 (3.3, 5.1) & -0.0 (-0.3, +0.2) \\*
Log rank & 4.6 (3.7, 5.5) & 4.7 (3.8, 5.6) & -0.1 (-0.3, +0.0) \\*
LRR & 4.1 (3.2, 5.0) & 4.1 (3.2, 5.0) & -0.6 (-0.9, -0.2) \\*
Entropy & 4.0 (3.1, 4.8) & 4.4 (3.5, 5.3) & +4.2 (+3.3, +5.2) \\*
Entropy gap & 5.5 (4.5, 6.6) & 4.7 (3.7, 5.6) & -2.3 (-3.0, -1.7) \\*
Binoculars & 5.0 (4.0, 5.9) & 3.6 (2.9, 4.5) & -0.0 (-0.2, +0.0) \\
\addlinespace[3pt]
\multicolumn{4}{l}{\textbf{WritingPrompts / Qwen 32B}} \\*
Log likelihood & 4.2 (3.3, 5.0) & 4.6 (3.7, 5.5) & +0.0 (+0.0, +0.0) \\*
Rank & 4.8 (4.0, 5.8) & 3.6 (2.8, 4.5) & +0.0 (+0.0, +0.0) \\*
Log rank & 4.7 (3.7, 5.5) & 3.9 (3.0, 4.8) & +0.0 (+0.0, +0.0) \\*
LRR & 3.5 (2.6, 4.3) & 4.1 (3.2, 5.0) & -0.0 (-0.1, +0.0) \\*
Entropy & 3.8 (2.9, 4.7) & 4.3 (3.5, 5.2) & +0.7 (-0.1, +1.5) \\*
Entropy gap & 5.1 (4.1, 6.0) & 5.2 (4.2, 6.2) & -2.8 (-3.6, -2.1) \\*
Binoculars & 4.8 (4.0, 5.8) & 3.5 (2.6, 4.2) & -0.3 (-0.7, +0.0) \\
\end{longtable}

\endgroup

\subsection{Selected directions and clipping specifications}
\label{app:selected-directions}
The table lists the score direction and selected clipping bounds
for each detector and dataset--generator combination.
For additive methods, \(L\) is the lower bound applied to token
scores after multiplying by the direction \(d\).
LRR uses separate upper bounds for \(-\log p_i\) and \(\log r_i\),
while Binoculars uses an upper bound for \(-\log p_i\) only.
``None'' means that no clipping was selected.
These bounds control the token scores; the document-level
detection thresholds are determined separately during calibration.
\begingroup\fontsize{8.5}{9}\selectfont\setlength{\tabcolsep}{4pt}
\begin{longtable}{lllrl}
\toprule
Dataset & Model & Detector & $d$ & Selected bound(s) \\
\midrule\endfirsthead
\toprule
Dataset & Model & Detector & $d$ & Selected bound(s) \\
\midrule\endhead
\bottomrule\endfoot
XSum & Granite 8B & Log likelihood & 1 & $L=-6.183$ \\
XSum & Mistral 24B & Log likelihood & 1 & $L=-6.466$ \\
XSum & Qwen 32B & Log likelihood & 1 & $L=-6.739$ \\
SQuAD & Granite 8B & Log likelihood & 1 & $L=-7.076$ \\
SQuAD & Mistral 24B & Log likelihood & 1 & $L=-6.036$ \\
SQuAD & Qwen 32B & Log likelihood & 1 & $L=-0.523$ \\
WritingPrompts & Granite 8B & Log likelihood & 1 & $L=-6.038$ \\
WritingPrompts & Mistral 24B & Log likelihood & 1 & $L=-6.366$ \\
WritingPrompts & Qwen 32B & Log likelihood & 1 & $L=-6.664$ \\
XSum & Granite 8B & Rank & -1 & $L=-7$ \\
XSum & Mistral 24B & Rank & -1 & $L=-19$ \\
XSum & Qwen 32B & Rank & -1 & $L=-8$ \\
SQuAD & Granite 8B & Rank & -1 & $L=-21$ \\
SQuAD & Mistral 24B & Rank & -1 & $L=-14$ \\
SQuAD & Qwen 32B & Rank & -1 & $L=-2$ \\
WritingPrompts & Granite 8B & Rank & -1 & $L=-13$ \\
WritingPrompts & Mistral 24B & Rank & -1 & $L=-8$ \\
WritingPrompts & Qwen 32B & Rank & -1 & $L=-5$ \\
XSum & Granite 8B & Log rank & -1 & $L=-2.833$ \\
XSum & Mistral 24B & Log rank & -1 & $L=-2.944$ \\
XSum & Qwen 32B & Log rank & -1 & $L=-2.952$ \\
SQuAD & Granite 8B & Log rank & -1 & $L=-3.045$ \\
SQuAD & Mistral 24B & Log rank & -1 & $L=-3.497$ \\
SQuAD & Qwen 32B & Log rank & -1 & $L=-0.6931$ \\
WritingPrompts & Granite 8B & Log rank & -1 & $L=-2.565$ \\
WritingPrompts & Mistral 24B & Log rank & -1 & $L=-2.708$ \\
WritingPrompts & Qwen 32B & Log rank & -1 & $L=-1.609$ \\
XSum & Granite 8B & LRR & 1 & $U_{-\log p}=3.088,\ U_{\log r}=1.099$ \\
XSum & Mistral 24B & LRR & 1 & $U_{-\log p}=9.524,\ U_{\log r}=2.944$ \\
XSum & Qwen 32B & LRR & 1 & $U_{-\log p}=10.03,\ U_{\log r}=2.952$ \\
SQuAD & Granite 8B & LRR & 1 & $U_{-\log p}=2.309,\ U_{\log r}=0.6931$ \\
SQuAD & Mistral 24B & LRR & 1 & $U_{-\log p}=2.828,\ U_{\log r}=1.099$ \\
SQuAD & Qwen 32B & LRR & 1 & $U_{-\log p}=8.15,\ U_{\log r}=0.6931$ \\
WritingPrompts & Granite 8B & LRR & 1 & $U_{-\log p}=10.48,\ U_{\log r}=3.664$ \\
WritingPrompts & Mistral 24B & LRR & 1 & $U_{-\log p}=11.25,\ U_{\log r}=3.829$ \\
WritingPrompts & Qwen 32B & LRR & 1 & $U_{-\log p}=11.67,\ U_{\log r}=4.007$ \\
XSum & Granite 8B & Entropy & 1 & $L=0.1271$ \\
XSum & Mistral 24B & Entropy & 1 & $L=0.5413$ \\
XSum & Qwen 32B & Entropy & -1 & $L=-3.148$ \\
SQuAD & Granite 8B & Entropy & 1 & $L=0.02071$ \\
SQuAD & Mistral 24B & Entropy & 1 & None \\
SQuAD & Qwen 32B & Entropy & 1 & $L=0.00164$ \\
WritingPrompts & Granite 8B & Entropy & -1 & $L=-3.562$ \\
WritingPrompts & Mistral 24B & Entropy & -1 & $L=-3.712$ \\
WritingPrompts & Qwen 32B & Entropy & -1 & $L=-3.586$ \\
XSum & Granite 8B & Entropy gap & -1 & $L=-1.993$ \\
XSum & Mistral 24B & Entropy gap & -1 & $L=-2.031$ \\
XSum & Qwen 32B & Entropy gap & -1 & $L=-2.39$ \\
SQuAD & Granite 8B & Entropy gap & -1 & $L=-2.385$ \\
SQuAD & Mistral 24B & Entropy gap & -1 & $L=-3.093$ \\
SQuAD & Qwen 32B & Entropy gap & -1 & $L=-0.5195$ \\
WritingPrompts & Granite 8B & Entropy gap & -1 & $L=-2.679$ \\
WritingPrompts & Mistral 24B & Entropy gap & -1 & $L=-2.831$ \\
WritingPrompts & Qwen 32B & Entropy gap & -1 & $L=-3.203$ \\
XSum & Granite 8B & Binoculars & -1 & $U=4.265$ \\
XSum & Mistral 24B & Binoculars & -1 & $U=4.184$ \\
XSum & Qwen 32B & Binoculars & -1 & $U=4.972$ \\
SQuAD & Granite 8B & Binoculars & -1 & $U=4.537$ \\
SQuAD & Mistral 24B & Binoculars & -1 & $U=5.308$ \\
SQuAD & Qwen 32B & Binoculars & -1 & $U=4.12$ \\
WritingPrompts & Granite 8B & Binoculars & -1 & $U=5.631$ \\
WritingPrompts & Mistral 24B & Binoculars & -1 & $U=5.5$ \\
WritingPrompts & Qwen 32B & Binoculars & -1 & $U=5.656$ \\
\end{longtable}

\endgroup

\clearpage
\subsection{Full dose-response atlas}
Each page shows results for all seven detectors on one dataset
at a target FPR of either 1\% or 5\%. Columns correspond to the
three generators, and rows correspond to the detectors.
Raw and clipped methods are evaluated on the same texts.
The horizontal axis shows the requested replacement rate,
and the vertical axis shows TPR. All panels use the same
vertical scale from 0 to 1.
\begin{figure}[H]\centering
\includegraphics[width=\linewidth,height=0.79\textheight,keepaspectratio]{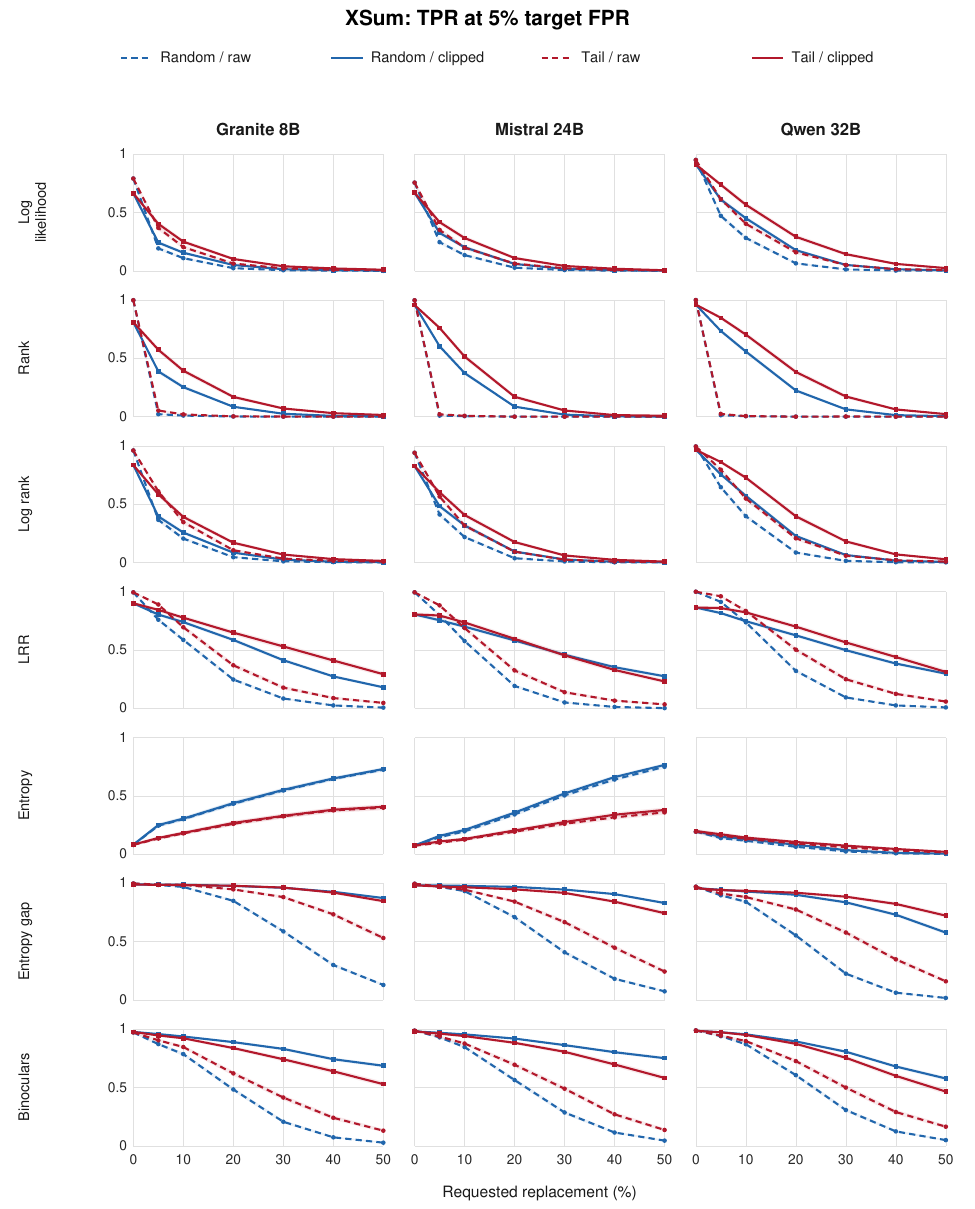}
\caption{XSum: all seven methods at target 5\% FPR.}\label{fig:primary-xsum5}
\end{figure}
\clearpage
\begin{figure}[H]\centering
\includegraphics[width=\linewidth,height=0.92\textheight,keepaspectratio]{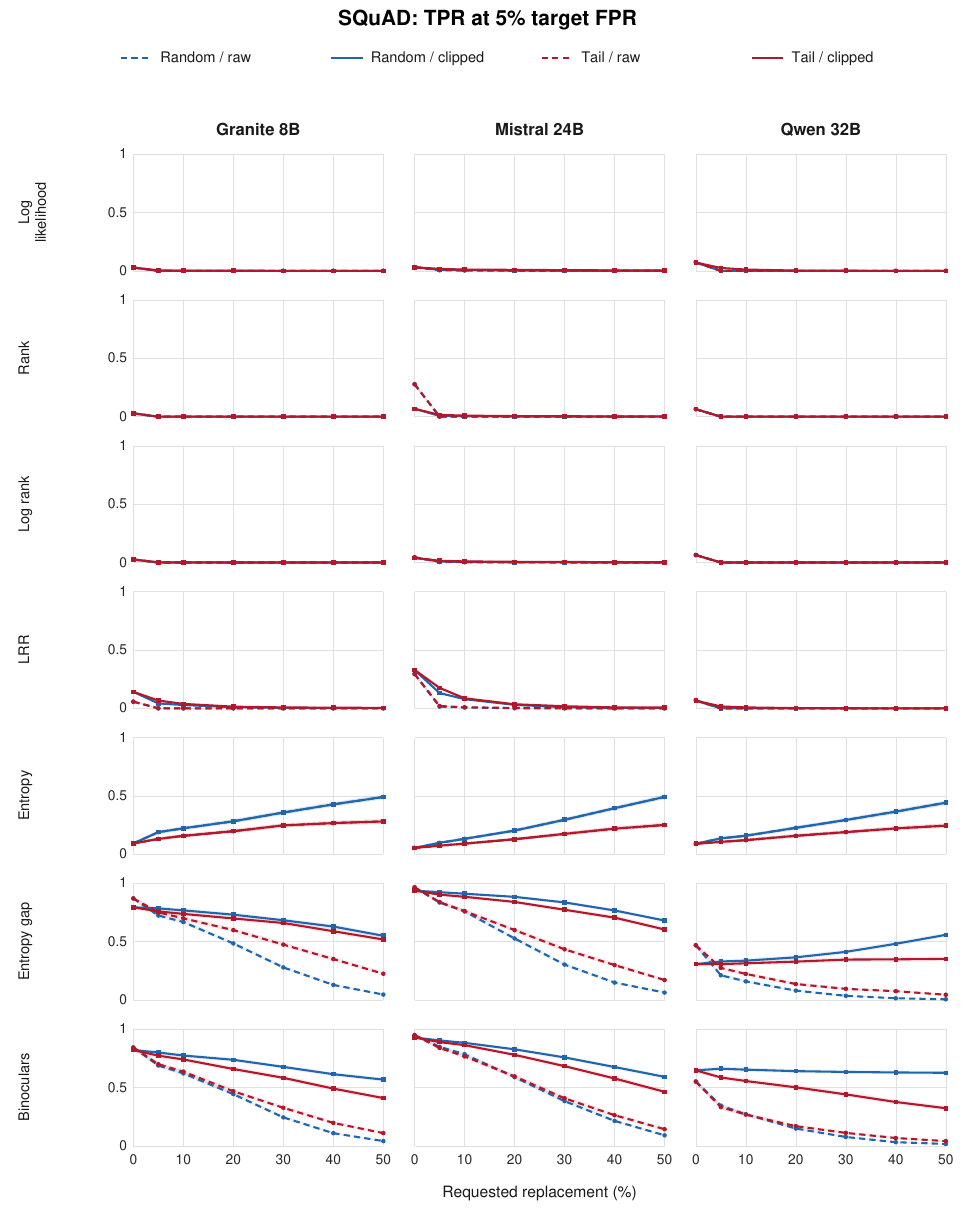}
\caption{SQuAD: all seven methods at target 5\% FPR. Low values are retained; no test-selected direction reversal is applied.}\label{fig:primary-squad5}
\end{figure}
\clearpage
\begin{figure}[H]\centering
\includegraphics[width=\linewidth,height=0.92\textheight,keepaspectratio]{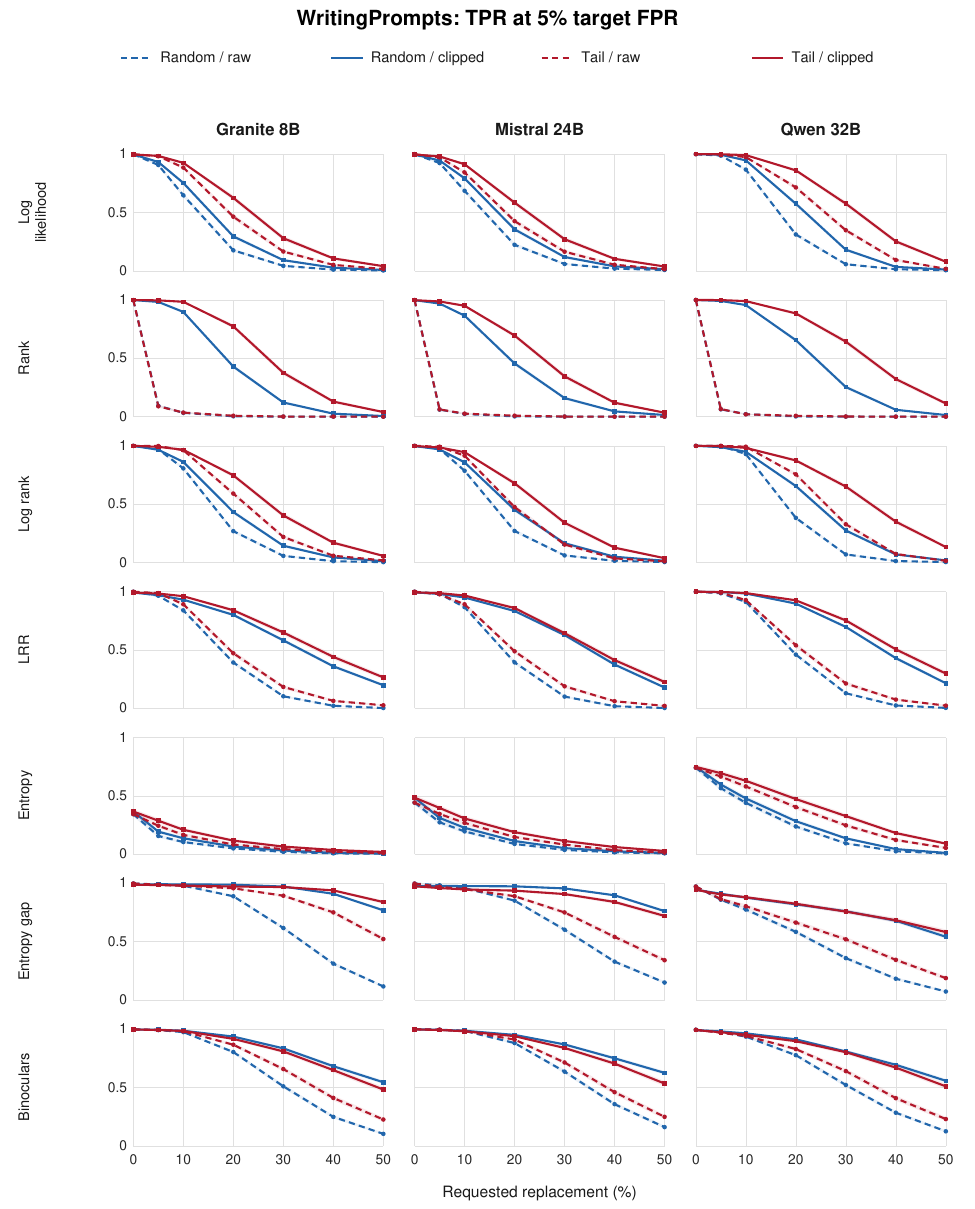}
\caption{WritingPrompts: all seven methods at target 5\% FPR.}\label{fig:primary-wp5}
\end{figure}
\clearpage
\begin{figure}[H]\centering
\includegraphics[width=\linewidth,height=0.92\textheight,keepaspectratio]{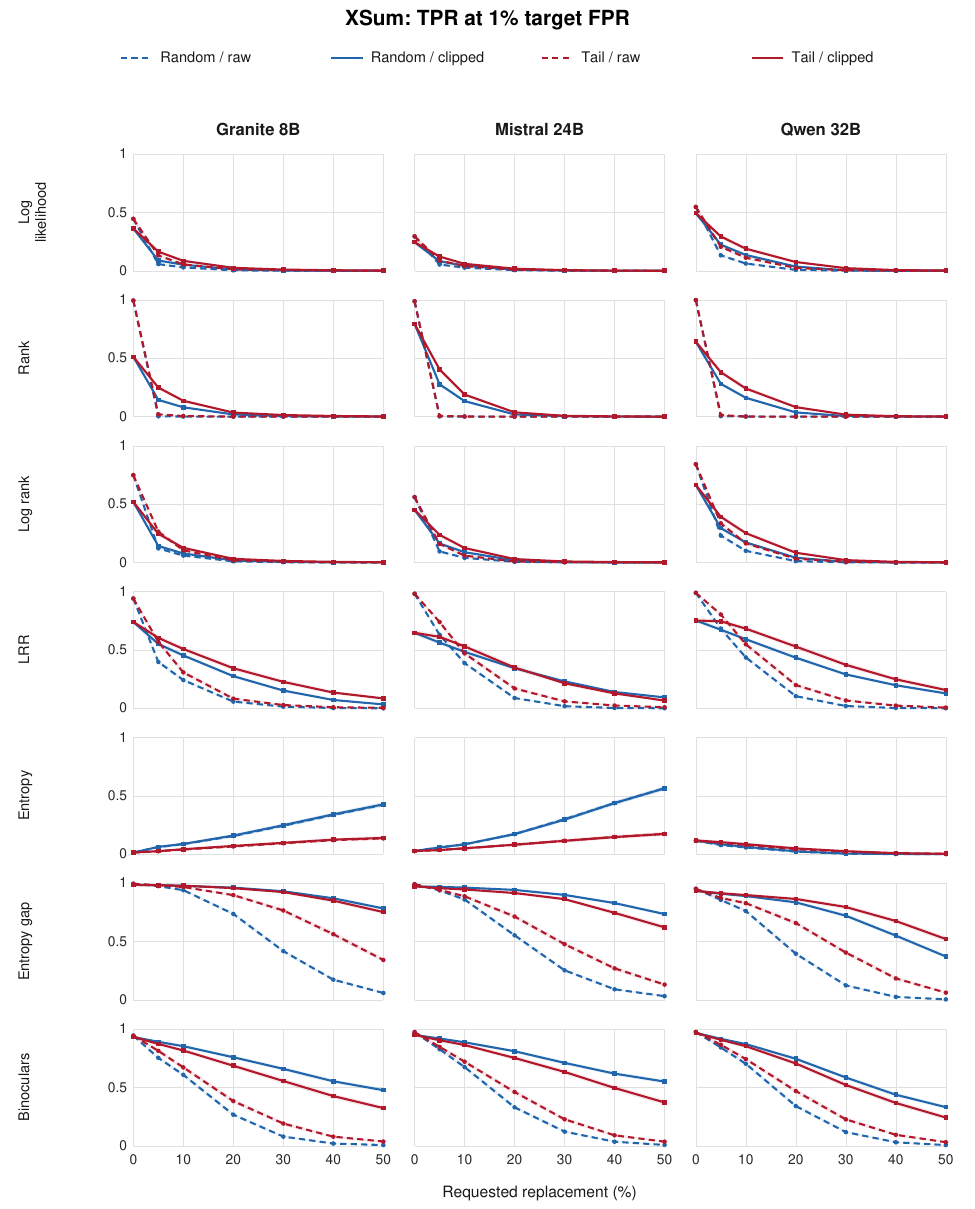}
\caption{XSum: secondary target 1\% FPR.}\label{fig:primary-xsum1}
\end{figure}
\clearpage
\begin{figure}[H]\centering
\includegraphics[width=\linewidth,height=0.92\textheight,keepaspectratio]{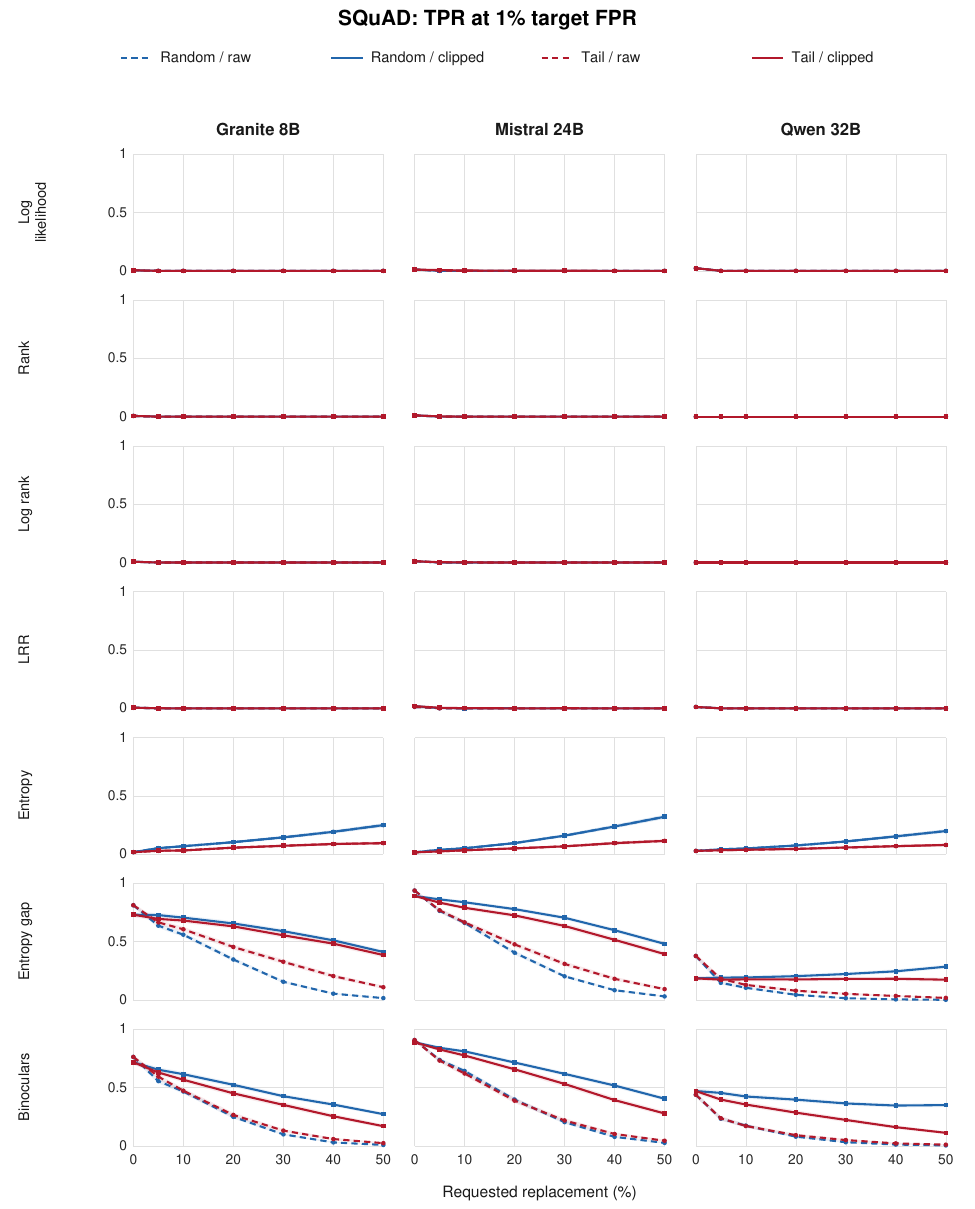}
\caption{SQuAD: secondary target 1\% FPR.}\label{fig:primary-squad1}
\end{figure}
\clearpage
\begin{figure}[H]\centering
\includegraphics[width=\linewidth,height=0.92\textheight,keepaspectratio]{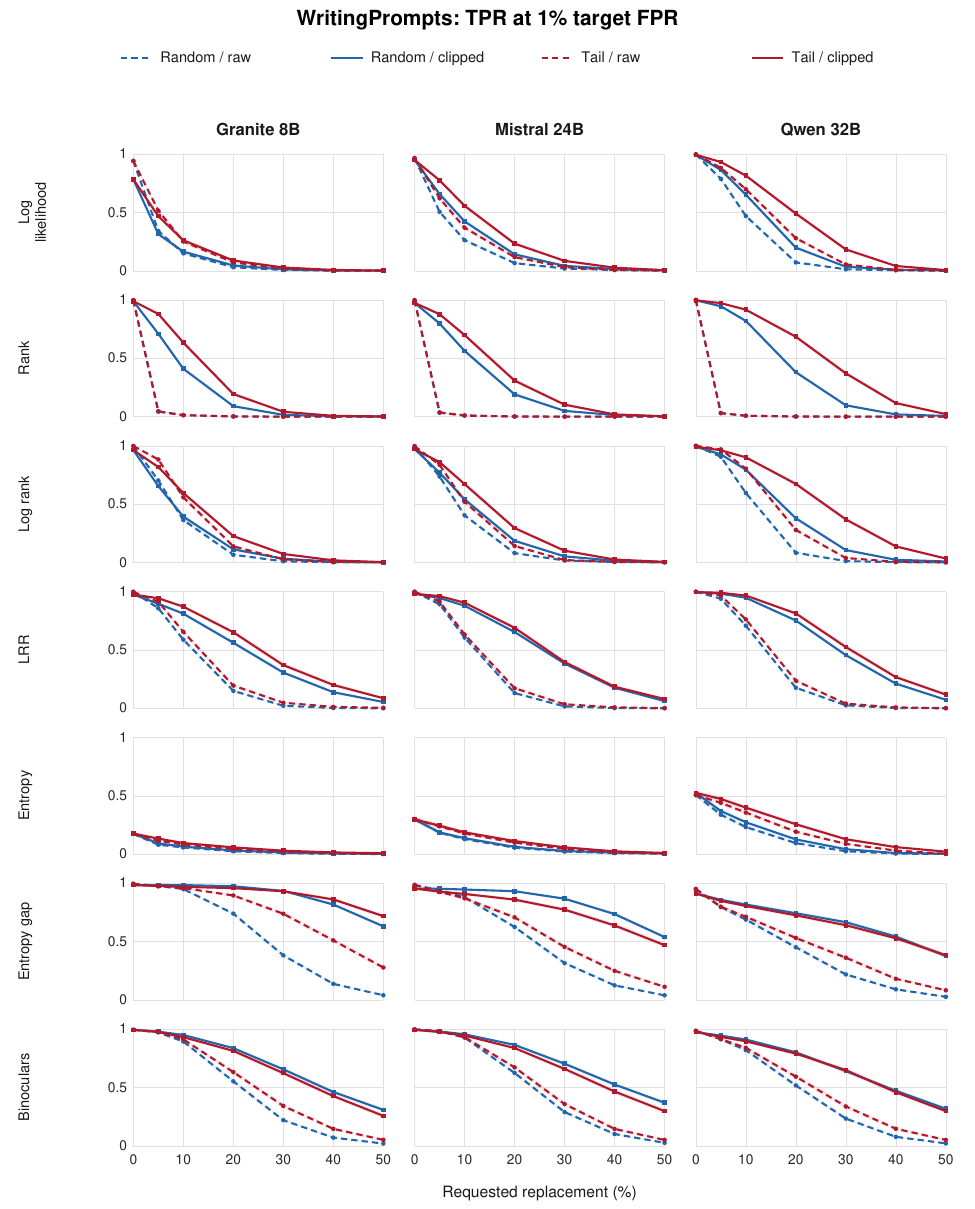}
\caption{WritingPrompts: secondary target 1\% FPR.}\label{fig:primary-wp1}
\end{figure}

\section{RAID: complete numerical results}
\label{app:raid}
% Shared appendix for the manuscript's theory variants.
% Tables generated by paper/scripts/raid_tables/build_tables.py.
This appendix reports the full RAID analysis.
We first check the raw Binoculars baseline, then give rate-specific and
attack-specific TPRs, followed by their clean TPR and human FPR.
The secondary universal comparisons, selected bounds, and sample details
complete the appendix.
TPR and FPR are percentages; changes are clipped minus raw in percentage
points. Brackets give paired 95\% source-cluster bootstrap intervals
from 2,000 resamples with fitted rules fixed.
All comparisons target 5\% FPR and draw on the same held-out pool of
5,149 source groups.
The intervals are pointwise, without correction for multiple comparisons.
The fitting procedure and its exploratory status are described in
Appendix~\ref{app:stat-clipping}.

\subsection{Comparison with the published RAID results}\label{app:raid-baseline}
The next table compares raw Binoculars with the seven conditions in RAID's main attack table. Point estimates and 95\% intervals are retained from the earlier baseline audit; current raw point estimates match. Here, \(\Delta\) is our TPR minus the published TPR.
\begin{table}[H]\centering
\caption{Raw Binoculars TPR (\%) compared with the published RAID results, at target 5\% FPR.}
\label{tab:raid-sanity}
\begingroup\centering
\begin{tabular}{lrrrr}
\toprule
Condition & RAID paper & Our raw & 95\% CI & $\Delta$ (pp) \\
\midrule
None & 79.6 & 80.00 & [78.91, 81.12] & +0.40 \\
Paraphrase & 80.3 & 80.83 & [79.76, 81.92] & +0.53 \\
Synonym & 43.5 & 43.58 & [42.26, 44.86] & +0.08 \\
Misspelling & 78.0 & 78.40 & [77.28, 79.49] & +0.40 \\
Homoglyph & 37.7 & 37.56 & [36.30, 38.82] & -0.14 \\
Whitespace & 70.1 & 70.09 & [68.83, 71.26] & -0.01 \\
Article deletion & 74.3 & 74.66 & [73.49, 75.90] & +0.36 \\
\bottomrule
\end{tabular}

\par\endgroup
\end{table}
Our raw Binoculars TPRs differ from the published values by at most 0.53 percentage points across these seven conditions. The achieved human FPR is 4.62\%. The comparison uses our selected sample and calibrated threshold.

\subsection{Rate-specific detection}\label{app:raid-rates}
Each table evaluates the bound fitted for that interval.
The same bound is used for all attacks represented within the interval.
The four test populations contain 19,894, 4,682, 7,137, and 3,490 attacked
texts. These counts can exceed the number of sources because multiple
attacks from one source can fall in the same interval; the bootstrap keeps
those texts together. Clean texts and attacks with measured zero change
are excluded from these attacked populations.

\begin{table}[htbp]\centering
\caption{Rate-specific detection for \(\rho\in(0,0.05]\), with 19894 attacked test texts.}\label{tab:raid-rate-1}
\setlength{\tabcolsep}{3pt}
\begin{tabular}{@{}lrrr@{}}
\toprule
Detector & Raw TPR & Clipped TPR & Change [95\% CI] \\
\midrule
Log likelihood & 49.17 & 49.49 & +0.32 [0.21, 0.44] \\
Rank & 29.48 & 54.71 & +25.23 [24.09, 26.41] \\
Log rank & 53.24 & 53.98 & +0.74 [0.54, 0.94] \\
LRR & 59.51 & 61.79 & +2.28 [1.72, 2.83] \\
Entropy & 26.11 & 26.20 & +0.09 [-0.20, 0.36] \\
Entropy gap & 66.55 & 68.15 & +1.60 [1.37, 1.84] \\
Binoculars & 74.27 & 74.49 & +0.22 [0.07, 0.38] \\
\bottomrule
\end{tabular}
\end{table}
\begin{table}[htbp]\centering
\caption{Rate-specific detection for \(\rho\in(0.05,0.1]\), with 4682 attacked test texts.}\label{tab:raid-rate-2}
\setlength{\tabcolsep}{3pt}
\begin{tabular}{@{}lrrr@{}}
\toprule
Detector & Raw TPR & Clipped TPR & Change [95\% CI] \\
\midrule
Log likelihood & 55.62 & 55.62 & +0.00 [0.00, 0.00] \\
Rank & 35.05 & 61.45 & +26.40 [24.83, 27.98] \\
Log rank & 59.38 & 60.06 & +0.68 [0.32, 1.08] \\
LRR & 64.48 & 68.69 & +4.21 [3.36, 5.01] \\
Entropy & 35.07 & 35.50 & +0.43 [0.02, 0.85] \\
Entropy gap & 71.64 & 72.77 & +1.13 [0.75, 1.52] \\
Binoculars & 82.79 & 83.06 & +0.28 [-0.13, 0.70] \\
\bottomrule
\end{tabular}
\end{table}
\begin{table}[htbp]\centering
\caption{Rate-specific detection for \(\rho\in(0.1,0.2]\), with 7137 attacked test texts.}\label{tab:raid-rate-3}
\setlength{\tabcolsep}{3pt}
\begin{tabular}{@{}lrrr@{}}
\toprule
Detector & Raw TPR & Clipped TPR & Change [95\% CI] \\
\midrule
Log likelihood & 30.47 & 30.47 & +0.00 [0.00, 0.00] \\
Rank & 24.87 & 40.83 & +15.96 [15.00, 16.98] \\
Log rank & 35.79 & 36.07 & +0.28 [0.10, 0.48] \\
LRR & 47.60 & 48.33 & +0.73 [0.10, 1.33] \\
Entropy & 18.50 & 18.51 & +0.01 [-0.08, 0.11] \\
Entropy gap & 46.77 & 47.79 & +1.02 [0.56, 1.52] \\
Binoculars & 60.88 & 61.75 & +0.87 [0.47, 1.26] \\
\bottomrule
\end{tabular}
\end{table}
\begin{table}[htbp]\centering
\caption{Rate-specific detection for \(\rho\in(0.2,0.5]\), with 3490 attacked test texts.}\label{tab:raid-rate-4}
\setlength{\tabcolsep}{3pt}
\begin{tabular}{@{}lrrr@{}}
\toprule
Detector & Raw TPR & Clipped TPR & Change [95\% CI] \\
\midrule
Log likelihood & 27.59 & 27.59 & +0.00 [0.00, 0.00] \\
Rank & 25.62 & 33.27 & +7.65 [6.32, 9.05] \\
Log rank & 31.49 & 31.40 & -0.09 [-0.34, 0.17] \\
LRR & 37.36 & 39.28 & +1.92 [1.07, 2.82] \\
Entropy & 16.28 & 16.48 & +0.20 [-0.06, 0.46] \\
Entropy gap & 44.58 & 45.53 & +0.95 [0.29, 1.59] \\
Binoculars & 57.88 & 57.88 & +0.00 [0.00, 0.00] \\
\bottomrule
\end{tabular}
\end{table}

\FloatBarrier
\subsection{Attack-specific detection}\label{app:raid-attacks}
The following tables give all eleven attacks for every detector.
Each row uses a bound fitted for that attack, without restricting its
modification rate, and evaluates 5,149 attacked test texts.
An attack-specific bound is not guaranteed to outperform a bound fitted
on a broader population: the selection objective is estimated from tuning
data and also includes clean TPR.

\begingroup\setlength{\tabcolsep}{3pt}
\begin{longtable}{@{}lrrr@{}}
\caption{Log likelihood: attack-specific detection, with 5,149 test texts per attack.}\label{tab:raid-attacks-log-likelihood}\\
\toprule
Attack & Raw TPR & Clipped TPR & Change [95\% CI] \\
\midrule\endfirsthead
\multicolumn{4}{l}{\textit{Continued}}\\
\toprule
Attack & Raw TPR & Clipped TPR & Change [95\% CI] \\
\midrule\endhead
\bottomrule\endfoot
Alternative spelling & 59.22 & 59.22 & +0.00 [0.00, 0.00] \\
Article deletion & 49.99 & 49.99 & +0.00 [0.00, 0.00] \\
Homoglyph & 91.42 & 91.42 & +0.00 [0.00, 0.00] \\
Paragraph insertion & 56.69 & 56.69 & +0.00 [0.00, 0.00] \\
Number change & 52.34 & 53.06 & +0.72 [0.49, 0.97] \\
Paraphrase & 27.91 & 27.91 & +0.00 [0.00, 0.00] \\
Misspelling & 58.13 & 58.40 & +0.27 [0.08, 0.47] \\
Synonym & 13.69 & 13.69 & +0.00 [0.00, 0.00] \\
Case change & 44.96 & 46.11 & +1.15 [0.78, 1.50] \\
Whitespace & 37.70 & 37.70 & +0.00 [0.00, 0.00] \\
Zero-width space & 100.00 & 100.00 & +0.00 [0.00, 0.00] \\
\end{longtable}
\endgroup
\begingroup\setlength{\tabcolsep}{3pt}
\begin{longtable}{@{}lrrr@{}}
\caption{Rank: attack-specific detection, with 5,149 test texts per attack.}\label{tab:raid-attacks-rank}\\
\toprule
Attack & Raw TPR & Clipped TPR & Change [95\% CI] \\
\midrule\endfirsthead
\multicolumn{4}{l}{\textit{Continued}}\\
\toprule
Attack & Raw TPR & Clipped TPR & Change [95\% CI] \\
\midrule\endhead
\bottomrule\endfoot
Alternative spelling & 39.50 & 63.33 & +23.83 [22.57, 25.15] \\
Article deletion & 36.63 & 56.09 & +19.46 [18.14, 20.78] \\
Homoglyph & 67.18 & 87.84 & +20.66 [19.48, 21.81] \\
Paragraph insertion & 38.98 & 60.61 & +21.64 [20.31, 22.90] \\
Number change & 37.19 & 56.79 & +19.60 [18.31, 20.88] \\
Paraphrase & 27.89 & 35.68 & +7.79 [6.62, 8.93] \\
Misspelling & 37.08 & 61.43 & +24.35 [23.05, 25.62] \\
Synonym & 11.98 & 15.63 & +3.65 [2.68, 4.62] \\
Case change & 22.02 & 53.39 & +31.37 [30.01, 32.72] \\
Whitespace & 31.89 & 55.82 & +23.93 [22.61, 25.23] \\
Zero-width space & 66.58 & 98.50 & +31.93 [30.86, 32.98] \\
\end{longtable}
\endgroup
\begingroup\setlength{\tabcolsep}{3pt}
\begin{longtable}{@{}lrrr@{}}
\caption{Log rank: attack-specific detection, with 5,149 test texts per attack.}\label{tab:raid-attacks-log-rank}\\
\toprule
Attack & Raw TPR & Clipped TPR & Change [95\% CI] \\
\midrule\endfirsthead
\multicolumn{4}{l}{\textit{Continued}}\\
\toprule
Attack & Raw TPR & Clipped TPR & Change [95\% CI] \\
\midrule\endhead
\bottomrule\endfoot
Alternative spelling & 62.11 & 62.30 & +0.19 [-0.02, 0.43] \\
Article deletion & 54.44 & 54.67 & +0.23 [0.00, 0.47] \\
Homoglyph & 90.15 & 90.29 & +0.14 [0.04, 0.25] \\
Paragraph insertion & 59.68 & 60.01 & +0.33 [0.14, 0.52] \\
Number change & 55.64 & 56.24 & +0.60 [0.21, 0.95] \\
Paraphrase & 31.87 & 31.77 & -0.10 [-0.25, 0.04] \\
Misspelling & 61.06 & 61.39 & +0.33 [0.12, 0.54] \\
Synonym & 15.15 & 15.23 & +0.08 [-0.10, 0.25] \\
Case change & 50.15 & 52.36 & +2.21 [1.77, 2.66] \\
Whitespace & 45.81 & 46.24 & +0.43 [0.17, 0.68] \\
Zero-width space & 99.73 & 99.71 & -0.02 [-0.06, 0.00] \\
\end{longtable}
\endgroup
\begingroup\setlength{\tabcolsep}{3pt}
\begin{longtable}{@{}lrrr@{}}
\caption{LRR: attack-specific detection, with 5,149 test texts per attack.}\label{tab:raid-attacks-lrr}\\
\toprule
Attack & Raw TPR & Clipped TPR & Change [95\% CI] \\
\midrule\endfirsthead
\multicolumn{4}{l}{\textit{Continued}}\\
\toprule
Attack & Raw TPR & Clipped TPR & Change [95\% CI] \\
\midrule\endhead
\bottomrule\endfoot
Alternative spelling & 65.04 & 69.92 & +4.87 [4.21, 5.57] \\
Article deletion & 61.70 & 65.00 & +3.30 [2.68, 3.94] \\
Homoglyph & 76.42 & 81.98 & +5.55 [4.89, 6.25] \\
Paragraph insertion & 63.80 & 67.94 & +4.14 [3.46, 4.82] \\
Number change & 60.87 & 62.92 & +2.06 [1.34, 2.84] \\
Paraphrase & 38.80 & 48.11 & +9.30 [8.47, 10.10] \\
Misspelling & 64.50 & 68.42 & +3.92 [3.28, 4.62] \\
Synonym & 19.97 & 24.39 & +4.43 [3.55, 5.26] \\
Case change & 59.16 & 62.67 & +3.52 [2.66, 4.35] \\
Whitespace & 63.10 & 67.37 & +4.27 [3.55, 4.99] \\
Zero-width space & 78.89 & 86.35 & +7.46 [6.82, 8.16] \\
\end{longtable}
\endgroup
\begingroup\setlength{\tabcolsep}{3pt}
\begin{longtable}{@{}lrrr@{}}
\caption{Entropy: attack-specific detection, with 5,149 test texts per attack.}\label{tab:raid-attacks-entropy}\\
\toprule
Attack & Raw TPR & Clipped TPR & Change [95\% CI] \\
\midrule\endfirsthead
\multicolumn{4}{l}{\textit{Continued}}\\
\toprule
Attack & Raw TPR & Clipped TPR & Change [95\% CI] \\
\midrule\endhead
\bottomrule\endfoot
Alternative spelling & 33.52 & 34.06 & +0.54 [0.12, 1.01] \\
Article deletion & 24.92 & 24.68 & -0.23 [-0.66, 0.17] \\
Homoglyph & 88.35 & 88.60 & +0.25 [-0.02, 0.52] \\
Paragraph insertion & 33.89 & 34.45 & +0.56 [0.12, 0.97] \\
Number change & 30.67 & 31.07 & +0.41 [-0.04, 0.85] \\
Paraphrase & 11.87 & 11.96 & +0.10 [-0.21, 0.41] \\
Misspelling & 32.05 & 32.82 & +0.78 [0.35, 1.22] \\
Synonym & 11.34 & 11.89 & +0.54 [0.27, 0.82] \\
Case change & 24.20 & 24.63 & +0.43 [0.10, 0.76] \\
Whitespace & 21.91 & 21.93 & +0.02 [-0.08, 0.12] \\
Zero-width space & 99.98 & 100.00 & +0.02 [0.00, 0.06] \\
\end{longtable}
\endgroup
\begingroup\setlength{\tabcolsep}{3pt}
\begin{longtable}{@{}lrrr@{}}
\caption{Entropy gap: attack-specific detection, with 5,149 test texts per attack.}\label{tab:raid-attacks-entropy-gap}\\
\toprule
Attack & Raw TPR & Clipped TPR & Change [95\% CI] \\
\midrule\endfirsthead
\multicolumn{4}{l}{\textit{Continued}}\\
\toprule
Attack & Raw TPR & Clipped TPR & Change [95\% CI] \\
\midrule\endhead
\bottomrule\endfoot
Alternative spelling & 72.54 & 73.39 & +0.85 [0.54, 1.17] \\
Article deletion & 68.89 & 68.89 & +0.00 [0.00, 0.00] \\
Homoglyph & 40.63 & 40.63 & +0.00 [0.00, 0.00] \\
Paragraph insertion & 70.40 & 70.40 & +0.00 [0.00, 0.00] \\
Number change & 67.92 & 70.36 & +2.45 [1.96, 2.99] \\
Paraphrase & 75.78 & 75.78 & +0.00 [0.00, 0.00] \\
Misspelling & 72.34 & 73.00 & +0.66 [0.37, 0.95] \\
Synonym & 28.06 & 32.22 & +4.16 [3.38, 4.99] \\
Case change & 65.43 & 68.85 & +3.42 [2.85, 4.00] \\
Whitespace & 55.78 & 56.71 & +0.93 [0.56, 1.30] \\
Zero-width space & 45.64 & 45.64 & +0.00 [0.00, 0.00] \\
\end{longtable}
\endgroup
\begingroup\setlength{\tabcolsep}{3pt}
\begin{longtable}{@{}lrrr@{}}
\caption{Binoculars: attack-specific detection, with 5,149 test texts per attack.}\label{tab:raid-attacks-binocular-origin}\\
\toprule
Attack & Raw TPR & Clipped TPR & Change [95\% CI] \\
\midrule\endfirsthead
\multicolumn{4}{l}{\textit{Continued}}\\
\toprule
Attack & Raw TPR & Clipped TPR & Change [95\% CI] \\
\midrule\endhead
\bottomrule\endfoot
Alternative spelling & 78.71 & 79.03 & +0.31 [0.12, 0.50] \\
Article deletion & 74.66 & 74.62 & -0.04 [-0.27, 0.21] \\
Homoglyph & 37.56 & 37.56 & +0.00 [0.00, 0.00] \\
Paragraph insertion & 79.30 & 79.49 & +0.19 [0.00, 0.39] \\
Number change & 77.35 & 78.02 & +0.66 [0.41, 0.93] \\
Paraphrase & 80.83 & 80.83 & +0.00 [0.00, 0.00] \\
Misspelling & 78.40 & 78.68 & +0.27 [0.12, 0.45] \\
Synonym & 43.58 & 49.85 & +6.27 [5.22, 7.34] \\
Case change & 74.07 & 74.50 & +0.43 [0.06, 0.80] \\
Whitespace & 70.09 & 70.73 & +0.64 [0.33, 0.93] \\
Zero-width space & 99.07 & 99.07 & +0.00 [0.00, 0.00] \\
\end{longtable}
\endgroup

\subsection{Clean detection and false positives}\label{app:raid-clean}
Each row below evaluates one fitted rule on clean machine texts and human
texts. Raw/Clip lists the two TPRs or FPRs; the adjacent column gives their
paired change and interval. Every fit uses the same clean test population.
The repeated human results are not independent samples and do not describe
an automatic system that chooses the fitting group for each input.

Clean performance need not decrease: several rank and LRR fits improve it.
There are also costs. For Binoculars fitted to synonyms, attacked TPR rises
from 43.58\% to 49.85\%, but clean TPR falls from 80.00\% to 70.69\%.
Human FPR rises from 4.62\% to 5.59\%, a 0.97-point change with interval
[0.27, 1.67]. The weighted fitting objective permits this trade-off;
it imposes no hard clean-performance or test-FPR constraint.

\begingroup\setlength{\tabcolsep}{3pt}
\begin{longtable}{@{}lrrrr@{}}
\caption{Log likelihood: clean detection and false positives for each specific fit. Each entry Raw/Clip gives both operating points.}\label{tab:raid-costs-log-likelihood}\\
\toprule
Fitting group & Clean TPR & Change [95\% CI] & Human FPR & Change [95\% CI] \\
\midrule\endfirsthead
\multicolumn{5}{l}{\textit{Continued}}\\
\toprule
Fitting group & Clean TPR & Change [95\% CI] & Human FPR & Change [95\% CI] \\
\midrule\endhead
\bottomrule\endfoot
\((0,0.05]\) & 61.29/61.45 & +0.16 [0.02, 0.29] & 5.85/5.90 & +0.06 [-0.06, 0.17] \\
\((0.05,0.1]\) & 61.29/61.29 & +0.00 [0.00, 0.00] & 5.85/5.85 & +0.00 [0.00, 0.00] \\
\((0.1,0.2]\) & 61.29/61.29 & +0.00 [0.00, 0.00] & 5.85/5.85 & +0.00 [0.00, 0.00] \\
\((0.2,0.5]\) & 61.29/61.29 & +0.00 [0.00, 0.00] & 5.85/5.85 & +0.00 [0.00, 0.00] \\
Alternative spelling & 61.29/61.29 & +0.00 [0.00, 0.00] & 5.85/5.85 & +0.00 [0.00, 0.00] \\
Article deletion & 61.29/61.29 & +0.00 [0.00, 0.00] & 5.85/5.85 & +0.00 [0.00, 0.00] \\
Homoglyph & 61.29/61.29 & +0.00 [0.00, 0.00] & 5.85/5.85 & +0.00 [0.00, 0.00] \\
Paragraph insertion & 61.29/61.29 & +0.00 [0.00, 0.00] & 5.85/5.85 & +0.00 [0.00, 0.00] \\
Number change & 61.29/61.45 & +0.16 [0.02, 0.31] & 5.85/5.90 & +0.06 [-0.08, 0.18] \\
Paraphrase & 61.29/61.29 & +0.00 [0.00, 0.00] & 5.85/5.85 & +0.00 [0.00, 0.00] \\
Misspelling & 61.29/61.45 & +0.16 [0.02, 0.31] & 5.85/5.90 & +0.06 [-0.08, 0.17] \\
Synonym & 61.29/61.29 & +0.00 [0.00, 0.00] & 5.85/5.85 & +0.00 [0.00, 0.00] \\
Case change & 61.29/60.98 & -0.31 [-0.60, -0.02] & 5.85/6.08 & +0.23 [0.00, 0.47] \\
Whitespace & 61.29/61.29 & +0.00 [0.00, 0.00] & 5.85/5.85 & +0.00 [0.00, 0.00] \\
Zero-width space & 61.29/61.29 & +0.00 [0.00, 0.00] & 5.85/5.85 & +0.00 [0.00, 0.00] \\
\end{longtable}
\endgroup
\begingroup\setlength{\tabcolsep}{3pt}
\begin{longtable}{@{}lrrrr@{}}
\caption{Rank: clean detection and false positives for each specific fit. Each entry Raw/Clip gives both operating points.}\label{tab:raid-costs-rank}\\
\toprule
Fitting group & Clean TPR & Change [95\% CI] & Human FPR & Change [95\% CI] \\
\midrule\endfirsthead
\multicolumn{5}{l}{\textit{Continued}}\\
\toprule
Fitting group & Clean TPR & Change [95\% CI] & Human FPR & Change [95\% CI] \\
\midrule\endhead
\bottomrule\endfoot
\((0,0.05]\) & 47.91/64.71 & +16.80 [15.42, 18.18] & 5.54/5.92 & +0.39 [-0.29, 1.13] \\
\((0.05,0.1]\) & 47.91/64.71 & +16.80 [15.38, 18.20] & 5.54/5.92 & +0.39 [-0.33, 1.13] \\
\((0.1,0.2]\) & 47.91/65.35 & +17.44 [16.14, 18.80] & 5.54/5.57 & +0.04 [-0.58, 0.66] \\
\((0.2,0.5]\) & 47.91/65.35 & +17.44 [16.20, 18.70] & 5.54/5.57 & +0.04 [-0.58, 0.64] \\
Alternative spelling & 47.91/65.35 & +17.44 [16.12, 18.74] & 5.54/5.57 & +0.04 [-0.60, 0.68] \\
Article deletion & 47.91/65.35 & +17.44 [16.12, 18.70] & 5.54/5.57 & +0.04 [-0.62, 0.68] \\
Homoglyph & 47.91/61.41 & +13.50 [12.06, 14.93] & 5.54/6.20 & +0.66 [-0.08, 1.42] \\
Paragraph insertion & 47.91/65.35 & +17.44 [16.16, 18.68] & 5.54/5.57 & +0.04 [-0.62, 0.66] \\
Number change & 47.91/65.35 & +17.44 [16.22, 18.80] & 5.54/5.57 & +0.04 [-0.58, 0.66] \\
Paraphrase & 47.91/65.35 & +17.44 [16.08, 18.80] & 5.54/5.57 & +0.04 [-0.60, 0.70] \\
Misspelling & 47.91/65.35 & +17.44 [16.18, 18.78] & 5.54/5.57 & +0.04 [-0.58, 0.64] \\
Synonym & 47.91/63.24 & +15.32 [13.91, 16.76] & 5.54/5.90 & +0.37 [-0.35, 1.09] \\
Case change & 47.91/63.24 & +15.32 [14.00, 16.72] & 5.54/5.90 & +0.37 [-0.43, 1.13] \\
Whitespace & 47.91/65.35 & +17.44 [16.12, 18.72] & 5.54/5.57 & +0.04 [-0.62, 0.68] \\
Zero-width space & 47.91/61.41 & +13.50 [12.10, 14.97] & 5.54/6.20 & +0.66 [-0.06, 1.40] \\
\end{longtable}
\endgroup
\begingroup\setlength{\tabcolsep}{3pt}
\begin{longtable}{@{}lrrrr@{}}
\caption{Log rank: clean detection and false positives for each specific fit. Each entry Raw/Clip gives both operating points.}\label{tab:raid-costs-log-rank}\\
\toprule
Fitting group & Clean TPR & Change [95\% CI] & Human FPR & Change [95\% CI] \\
\midrule\endfirsthead
\multicolumn{5}{l}{\textit{Continued}}\\
\toprule
Fitting group & Clean TPR & Change [95\% CI] & Human FPR & Change [95\% CI] \\
\midrule\endhead
\bottomrule\endfoot
\((0,0.05]\) & 63.76/63.68 & -0.08 [-0.37, 0.19] & 5.98/5.85 & -0.14 [-0.37, 0.08] \\
\((0.05,0.1]\) & 63.76/63.68 & -0.08 [-0.37, 0.23] & 5.98/5.85 & -0.14 [-0.37, 0.10] \\
\((0.1,0.2]\) & 63.76/63.82 & +0.06 [-0.12, 0.23] & 5.98/5.83 & -0.16 [-0.33, 0.02] \\
\((0.2,0.5]\) & 63.76/63.82 & +0.06 [-0.10, 0.21] & 5.98/5.83 & -0.16 [-0.33, 0.02] \\
Alternative spelling & 63.76/63.82 & +0.06 [-0.12, 0.23] & 5.98/5.83 & -0.16 [-0.33, 0.02] \\
Article deletion & 63.76/63.82 & +0.06 [-0.10, 0.23] & 5.98/5.83 & -0.16 [-0.33, 0.00] \\
Homoglyph & 63.76/63.84 & +0.08 [-0.06, 0.21] & 5.98/5.98 & +0.00 [-0.14, 0.14] \\
Paragraph insertion & 63.76/63.82 & +0.06 [-0.10, 0.23] & 5.98/5.83 & -0.16 [-0.33, 0.02] \\
Number change & 63.76/63.68 & -0.08 [-0.37, 0.21] & 5.98/5.85 & -0.14 [-0.37, 0.10] \\
Paraphrase & 63.76/63.84 & +0.08 [-0.04, 0.21] & 5.98/5.98 & +0.00 [-0.12, 0.14] \\
Misspelling & 63.76/63.82 & +0.06 [-0.10, 0.23] & 5.98/5.83 & -0.16 [-0.33, 0.02] \\
Synonym & 63.76/63.82 & +0.06 [-0.12, 0.23] & 5.98/5.83 & -0.16 [-0.33, 0.00] \\
Case change & 63.76/63.86 & +0.10 [-0.31, 0.50] & 5.98/6.23 & +0.25 [-0.06, 0.58] \\
Whitespace & 63.76/63.82 & +0.06 [-0.10, 0.21] & 5.98/5.83 & -0.16 [-0.33, 0.00] \\
Zero-width space & 63.76/63.68 & -0.08 [-0.35, 0.19] & 5.98/5.85 & -0.14 [-0.37, 0.10] \\
\end{longtable}
\endgroup
\begingroup\setlength{\tabcolsep}{3pt}
\begin{longtable}{@{}lrrrr@{}}
\caption{LRR: clean detection and false positives for each specific fit. Each entry Raw/Clip gives both operating points.}\label{tab:raid-costs-lrr}\\
\toprule
Fitting group & Clean TPR & Change [95\% CI] & Human FPR & Change [95\% CI] \\
\midrule\endfirsthead
\multicolumn{5}{l}{\textit{Continued}}\\
\toprule
Fitting group & Clean TPR & Change [95\% CI] & Human FPR & Change [95\% CI] \\
\midrule\endhead
\bottomrule\endfoot
\((0,0.05]\) & 66.05/69.86 & +3.81 [3.11, 4.58] & 5.55/5.54 & -0.02 [-0.49, 0.49] \\
\((0.05,0.1]\) & 66.05/69.86 & +3.81 [3.13, 4.55] & 5.55/5.54 & -0.02 [-0.50, 0.49] \\
\((0.1,0.2]\) & 66.05/70.91 & +4.86 [4.19, 5.50] & 5.55/5.63 & +0.08 [-0.39, 0.52] \\
\((0.2,0.5]\) & 66.05/69.86 & +3.81 [3.09, 4.49] & 5.55/5.54 & -0.02 [-0.52, 0.50] \\
Alternative spelling & 66.05/70.91 & +4.86 [4.19, 5.52] & 5.55/5.63 & +0.08 [-0.39, 0.52] \\
Article deletion & 66.05/70.91 & +4.86 [4.21, 5.54] & 5.55/5.63 & +0.08 [-0.37, 0.51] \\
Homoglyph & 66.05/70.98 & +4.93 [4.25, 5.59] & 5.55/5.44 & -0.12 [-0.60, 0.37] \\
Paragraph insertion & 66.05/70.91 & +4.86 [4.18, 5.52] & 5.55/5.63 & +0.08 [-0.37, 0.52] \\
Number change & 66.05/69.45 & +3.40 [2.72, 4.08] & 5.55/5.48 & -0.08 [-0.54, 0.41] \\
Paraphrase & 66.05/70.91 & +4.86 [4.21, 5.54] & 5.55/5.63 & +0.08 [-0.37, 0.50] \\
Misspelling & 66.05/70.91 & +4.86 [4.19, 5.48] & 5.55/5.63 & +0.08 [-0.37, 0.58] \\
Synonym & 66.05/64.60 & -1.46 [-2.35, -0.56] & 5.55/5.17 & -0.39 [-1.05, 0.23] \\
Case change & 66.05/68.25 & +2.19 [1.42, 2.97] & 5.55/5.38 & -0.17 [-0.76, 0.39] \\
Whitespace & 66.05/71.20 & +5.15 [4.49, 5.87] & 5.55/5.77 & +0.21 [-0.25, 0.68] \\
Zero-width space & 66.05/71.20 & +5.15 [4.49, 5.81] & 5.55/5.77 & +0.21 [-0.25, 0.68] \\
\end{longtable}
\endgroup
\begingroup\setlength{\tabcolsep}{3pt}
\begin{longtable}{@{}lrrrr@{}}
\caption{Entropy: clean detection and false positives for each specific fit. Each entry Raw/Clip gives both operating points.}\label{tab:raid-costs-entropy}\\
\toprule
Fitting group & Clean TPR & Change [95\% CI] & Human FPR & Change [95\% CI] \\
\midrule\endfirsthead
\multicolumn{5}{l}{\textit{Continued}}\\
\toprule
Fitting group & Clean TPR & Change [95\% CI] & Human FPR & Change [95\% CI] \\
\midrule\endhead
\bottomrule\endfoot
\((0,0.05]\) & 34.65/35.42 & +0.78 [0.35, 1.26] & 5.75/5.79 & +0.04 [-0.25, 0.33] \\
\((0.05,0.1]\) & 34.65/35.42 & +0.78 [0.33, 1.22] & 5.75/5.79 & +0.04 [-0.25, 0.31] \\
\((0.1,0.2]\) & 34.65/34.90 & +0.25 [0.10, 0.43] & 5.75/5.77 & +0.02 [-0.10, 0.14] \\
\((0.2,0.5]\) & 34.65/35.60 & +0.95 [0.66, 1.26] & 5.75/5.96 & +0.21 [0.00, 0.45] \\
Alternative spelling & 34.65/35.42 & +0.78 [0.33, 1.24] & 5.75/5.79 & +0.04 [-0.23, 0.31] \\
Article deletion & 34.65/35.42 & +0.78 [0.33, 1.22] & 5.75/5.79 & +0.04 [-0.23, 0.33] \\
Homoglyph & 34.65/35.42 & +0.78 [0.31, 1.24] & 5.75/5.79 & +0.04 [-0.23, 0.33] \\
Paragraph insertion & 34.65/35.42 & +0.78 [0.33, 1.22] & 5.75/5.79 & +0.04 [-0.23, 0.31] \\
Number change & 34.65/35.42 & +0.78 [0.31, 1.24] & 5.75/5.79 & +0.04 [-0.25, 0.31] \\
Paraphrase & 34.65/35.42 & +0.78 [0.35, 1.24] & 5.75/5.79 & +0.04 [-0.25, 0.33] \\
Misspelling & 34.65/35.42 & +0.78 [0.33, 1.22] & 5.75/5.79 & +0.04 [-0.23, 0.33] \\
Synonym & 34.65/35.42 & +0.78 [0.31, 1.24] & 5.75/5.79 & +0.04 [-0.25, 0.31] \\
Case change & 34.65/35.68 & +1.03 [0.66, 1.38] & 5.75/5.85 & +0.10 [-0.14, 0.35] \\
Whitespace & 34.65/34.73 & +0.08 [-0.06, 0.21] & 5.75/5.75 & +0.00 [-0.10, 0.10] \\
Zero-width space & 34.65/35.42 & +0.78 [0.35, 1.24] & 5.75/5.79 & +0.04 [-0.23, 0.33] \\
\end{longtable}
\endgroup
\begingroup\setlength{\tabcolsep}{3pt}
\begin{longtable}{@{}lrrrr@{}}
\caption{Entropy gap: clean detection and false positives for each specific fit. Each entry Raw/Clip gives both operating points.}\label{tab:raid-costs-entropy-gap}\\
\toprule
Fitting group & Clean TPR & Change [95\% CI] & Human FPR & Change [95\% CI] \\
\midrule\endfirsthead
\multicolumn{5}{l}{\textit{Continued}}\\
\toprule
Fitting group & Clean TPR & Change [95\% CI] & Human FPR & Change [95\% CI] \\
\midrule\endhead
\bottomrule\endfoot
\((0,0.05]\) & 75.06/74.93 & -0.14 [-0.45, 0.17] & 5.07/5.09 & +0.02 [-0.23, 0.29] \\
\((0.05,0.1]\) & 75.06/75.30 & +0.23 [-0.02, 0.47] & 5.07/4.97 & -0.10 [-0.29, 0.10] \\
\((0.1,0.2]\) & 75.06/73.90 & -1.17 [-1.63, -0.74] & 5.07/5.36 & +0.29 [-0.08, 0.66] \\
\((0.2,0.5]\) & 75.06/73.90 & -1.17 [-1.59, -0.74] & 5.07/5.36 & +0.29 [-0.08, 0.66] \\
Alternative spelling & 75.06/75.30 & +0.23 [0.00, 0.47] & 5.07/4.97 & -0.10 [-0.29, 0.10] \\
Article deletion & 75.06/75.06 & +0.00 [0.00, 0.00] & 5.07/5.07 & +0.00 [0.00, 0.00] \\
Homoglyph & 75.06/75.06 & +0.00 [0.00, 0.00] & 5.07/5.07 & +0.00 [0.00, 0.00] \\
Paragraph insertion & 75.06/75.06 & +0.00 [0.00, 0.00] & 5.07/5.07 & +0.00 [0.00, 0.00] \\
Number change & 75.06/74.93 & -0.14 [-0.45, 0.17] & 5.07/5.09 & +0.02 [-0.23, 0.29] \\
Paraphrase & 75.06/75.06 & +0.00 [0.00, 0.00] & 5.07/5.07 & +0.00 [0.00, 0.00] \\
Misspelling & 75.06/75.30 & +0.23 [0.00, 0.47] & 5.07/4.97 & -0.10 [-0.31, 0.10] \\
Synonym & 75.06/70.09 & -4.97 [-5.63, -4.31] & 5.07/5.19 & +0.12 [-0.37, 0.58] \\
Case change & 75.06/73.90 & -1.17 [-1.63, -0.72] & 5.07/5.36 & +0.29 [-0.08, 0.66] \\
Whitespace & 75.06/75.30 & +0.23 [0.00, 0.47] & 5.07/4.97 & -0.10 [-0.31, 0.10] \\
Zero-width space & 75.06/75.06 & +0.00 [0.00, 0.00] & 5.07/5.07 & +0.00 [0.00, 0.00] \\
\end{longtable}
\endgroup
\begingroup\setlength{\tabcolsep}{3pt}
\begin{longtable}{@{}lrrrr@{}}
\caption{Binoculars: clean detection and false positives for each specific fit. Each entry Raw/Clip gives both operating points.}\label{tab:raid-costs-binocular-origin}\\
\toprule
Fitting group & Clean TPR & Change [95\% CI] & Human FPR & Change [95\% CI] \\
\midrule\endfirsthead
\multicolumn{5}{l}{\textit{Continued}}\\
\toprule
Fitting group & Clean TPR & Change [95\% CI] & Human FPR & Change [95\% CI] \\
\midrule\endhead
\bottomrule\endfoot
\((0,0.05]\) & 80.00/79.92 & -0.08 [-0.29, 0.14] & 4.62/4.53 & -0.10 [-0.33, 0.12] \\
\((0.05,0.1]\) & 80.00/79.72 & -0.27 [-0.58, 0.02] & 4.62/4.66 & +0.04 [-0.25, 0.33] \\
\((0.1,0.2]\) & 80.00/79.72 & -0.27 [-0.58, 0.02] & 4.62/4.66 & +0.04 [-0.27, 0.33] \\
\((0.2,0.5]\) & 80.00/80.00 & +0.00 [0.00, 0.00] & 4.62/4.62 & +0.00 [0.00, 0.00] \\
Alternative spelling & 80.00/80.09 & +0.10 [-0.10, 0.29] & 4.62/4.64 & +0.02 [-0.16, 0.19] \\
Article deletion & 80.00/79.92 & -0.08 [-0.29, 0.14] & 4.62/4.53 & -0.10 [-0.31, 0.12] \\
Homoglyph & 80.00/80.00 & +0.00 [0.00, 0.00] & 4.62/4.62 & +0.00 [0.00, 0.00] \\
Paragraph insertion & 80.00/80.09 & +0.10 [-0.08, 0.29] & 4.62/4.64 & +0.02 [-0.16, 0.19] \\
Number change & 80.00/80.09 & +0.10 [-0.10, 0.29] & 4.62/4.64 & +0.02 [-0.16, 0.19] \\
Paraphrase & 80.00/80.00 & +0.00 [0.00, 0.00] & 4.62/4.62 & +0.00 [0.00, 0.00] \\
Misspelling & 80.00/80.09 & +0.10 [-0.08, 0.29] & 4.62/4.64 & +0.02 [-0.16, 0.19] \\
Synonym & 80.00/70.69 & -9.30 [-10.27, -8.39] & 4.62/5.59 & +0.97 [0.27, 1.67] \\
Case change & 80.00/79.72 & -0.27 [-0.60, 0.04] & 4.62/4.66 & +0.04 [-0.25, 0.33] \\
Whitespace & 80.00/79.92 & -0.08 [-0.29, 0.14] & 4.62/4.53 & -0.10 [-0.33, 0.12] \\
Zero-width space & 80.00/80.00 & +0.00 [0.00, 0.00] & 4.62/4.62 & +0.00 [0.00, 0.00] \\
\end{longtable}
\endgroup

\subsection{Secondary universal comparisons}\label{app:raid-universal}
Full-universal fitting uses all attacked tuning texts; eligible-universal
fitting uses those with \(0<\rho\leq0.5\).
Both use the same cross-fitted TPR objective as the main analysis.
Within each row, the two fitted rules are evaluated on exactly the same test
population. Full and Eligible give clipped TPRs; changes compare each with
raw, not with each other. These comparisons examine whether a shared bound
is useful when no specific attack or narrow rate interval is supplied.

A narrower fitting population does not guarantee better test performance.
For example, eligible-universal entropy gap and Binoculars have negative
all-attack TPR changes, whereas full-universal entropy gap is unchanged and
full-universal Binoculars improves slightly. Each fit's clean TPR and FPR
are reported directly after its attacked results.

\begingroup\setlength{\tabcolsep}{3pt}
\begin{longtable}{@{}lrrrrr@{}}
\caption{Log likelihood: secondary universal rules evaluated on matched test populations.}\label{tab:raid-universal-log-likelihood}\\
\toprule
Test population & Raw & Full & Change [95\% CI] & Eligible & Change [95\% CI] \\
\midrule\endfirsthead
\multicolumn{6}{l}{\textit{Continued}}\\
\toprule
Test population & Raw & Full & Change [95\% CI] & Eligible & Change [95\% CI] \\
\midrule\endhead
\bottomrule\endfoot
All attacks & 53.82 & 53.82 & \shortstack{+0.00\\{}[0.00, 0.00]} & 53.99 & \shortstack{+0.17\\{}[0.11, 0.22]} \\[3pt]
\(0<\rho\leq0.5\) & 44.10 & 44.10 & \shortstack{+0.00\\{}[0.00, 0.00]} & 44.34 & \shortstack{+0.24\\{}[0.16, 0.33]} \\[3pt]
\((0,0.05]\) & 49.17 & 49.17 & \shortstack{+0.00\\{}[0.00, 0.00]} & 49.49 & \shortstack{+0.32\\{}[0.21, 0.43]} \\[3pt]
\((0.05,0.1]\) & 55.62 & 55.62 & \shortstack{+0.00\\{}[0.00, 0.00]} & 55.83 & \shortstack{+0.21\\{}[0.02, 0.40]} \\[3pt]
\((0.1,0.2]\) & 30.47 & 30.47 & \shortstack{+0.00\\{}[0.00, 0.00]} & 30.73 & \shortstack{+0.25\\{}[0.08, 0.42]} \\[3pt]
\((0.2,0.5]\) & 27.59 & 27.59 & \shortstack{+0.00\\{}[0.00, 0.00]} & 27.42 & \shortstack{-0.17\\{}[-0.40, 0.06]} \\[3pt]
\end{longtable}
\endgroup
\begingroup\setlength{\tabcolsep}{3pt}
\begin{longtable}{@{}lrrrr@{}}
\caption{Log likelihood: clean TPR and human FPR for the two universal fits (Raw/Clip).}\label{tab:raid-universal-cost-log-likelihood}\\
\toprule
Fit & Clean TPR & Change [95\% CI] & Human FPR & Change [95\% CI] \\
\midrule\endfirsthead
\multicolumn{5}{l}{\textit{Continued}}\\
\toprule
Fit & Clean TPR & Change [95\% CI] & Human FPR & Change [95\% CI] \\
\midrule\endhead
\bottomrule\endfoot
Full & 61.29/61.29 & +0.00 [0.00, 0.00] & 5.85/5.85 & +0.00 [0.00, 0.00] \\
Eligible & 61.29/61.45 & +0.16 [0.02, 0.31] & 5.85/5.90 & +0.06 [-0.06, 0.19] \\
\end{longtable}
\endgroup
\begingroup\setlength{\tabcolsep}{3pt}
\begin{longtable}{@{}lrrrrr@{}}
\caption{Rank: secondary universal rules evaluated on matched test populations.}\label{tab:raid-universal-rank}\\
\toprule
Test population & Raw & Full & Change [95\% CI] & Eligible & Change [95\% CI] \\
\midrule\endfirsthead
\multicolumn{6}{l}{\textit{Continued}}\\
\toprule
Test population & Raw & Full & Change [95\% CI] & Eligible & Change [95\% CI] \\
\midrule\endhead
\bottomrule\endfoot
All attacks & 37.90 & 57.64 & \shortstack{+19.74\\{}[19.00, 20.43]} & 57.75 & \shortstack{+19.84\\{}[19.08, 20.60]} \\[3pt]
\(0<\rho\leq0.5\) & 28.90 & 49.58 & \shortstack{+20.68\\{}[19.90, 21.62]} & 50.26 & \shortstack{+21.36\\{}[20.43, 22.28]} \\[3pt]
\((0,0.05]\) & 29.48 & 53.37 & \shortstack{+23.89\\{}[22.82, 24.97]} & 54.71 & \shortstack{+25.23\\{}[24.10, 26.38]} \\[3pt]
\((0.05,0.1]\) & 35.05 & 59.01 & \shortstack{+23.96\\{}[22.48, 25.43]} & 61.45 & \shortstack{+26.40\\{}[24.84, 28.01]} \\[3pt]
\((0.1,0.2]\) & 24.87 & 40.83 & \shortstack{+15.96\\{}[14.88, 17.00]} & 39.13 & \shortstack{+14.26\\{}[13.18, 15.40]} \\[3pt]
\((0.2,0.5]\) & 25.62 & 33.27 & \shortstack{+7.65\\{}[6.29, 8.99]} & 32.61 & \shortstack{+6.99\\{}[5.55, 8.32]} \\[3pt]
\end{longtable}
\endgroup
\begingroup\setlength{\tabcolsep}{3pt}
\begin{longtable}{@{}lrrrr@{}}
\caption{Rank: clean TPR and human FPR for the two universal fits (Raw/Clip).}\label{tab:raid-universal-cost-rank}\\
\toprule
Fit & Clean TPR & Change [95\% CI] & Human FPR & Change [95\% CI] \\
\midrule\endfirsthead
\multicolumn{5}{l}{\textit{Continued}}\\
\toprule
Fit & Clean TPR & Change [95\% CI] & Human FPR & Change [95\% CI] \\
\midrule\endhead
\bottomrule\endfoot
Full & 47.91/65.35 & +17.44 [16.12, 18.78] & 5.54/5.57 & +0.04 [-0.56, 0.70] \\
Eligible & 47.91/64.71 & +16.80 [15.46, 18.14] & 5.54/5.92 & +0.39 [-0.31, 1.07] \\
\end{longtable}
\endgroup
\begingroup\setlength{\tabcolsep}{3pt}
\begin{longtable}{@{}lrrrrr@{}}
\caption{Log rank: secondary universal rules evaluated on matched test populations.}\label{tab:raid-universal-log-rank}\\
\toprule
Test population & Raw & Full & Change [95\% CI] & Eligible & Change [95\% CI] \\
\midrule\endfirsthead
\multicolumn{6}{l}{\textit{Continued}}\\
\toprule
Test population & Raw & Full & Change [95\% CI] & Eligible & Change [95\% CI] \\
\midrule\endhead
\bottomrule\endfoot
All attacks & 56.89 & 57.10 & \shortstack{+0.21\\{}[0.15, 0.29]} & 57.17 & \shortstack{+0.28\\{}[0.17, 0.40]} \\[3pt]
\(0<\rho\leq0.5\) & 48.36 & 48.69 & \shortstack{+0.34\\{}[0.24, 0.43]} & 48.87 & \shortstack{+0.51\\{}[0.37, 0.66]} \\[3pt]
\((0,0.05]\) & 53.24 & 53.63 & \shortstack{+0.40\\{}[0.27, 0.53]} & 53.98 & \shortstack{+0.74\\{}[0.55, 0.94]} \\[3pt]
\((0.05,0.1]\) & 59.38 & 59.85 & \shortstack{+0.47\\{}[0.23, 0.72]} & 60.06 & \shortstack{+0.68\\{}[0.32, 1.08]} \\[3pt]
\((0.1,0.2]\) & 35.79 & 36.07 & \shortstack{+0.28\\{}[0.11, 0.46]} & 35.88 & \shortstack{+0.10\\{}[-0.15, 0.35]} \\[3pt]
\((0.2,0.5]\) & 31.49 & 31.40 & \shortstack{-0.09\\{}[-0.32, 0.14]} & 31.35 & \shortstack{-0.14\\{}[-0.46, 0.17]} \\[3pt]
\end{longtable}
\endgroup
\begingroup\setlength{\tabcolsep}{3pt}
\begin{longtable}{@{}lrrrr@{}}
\caption{Log rank: clean TPR and human FPR for the two universal fits (Raw/Clip).}\label{tab:raid-universal-cost-log-rank}\\
\toprule
Fit & Clean TPR & Change [95\% CI] & Human FPR & Change [95\% CI] \\
\midrule\endfirsthead
\multicolumn{5}{l}{\textit{Continued}}\\
\toprule
Fit & Clean TPR & Change [95\% CI] & Human FPR & Change [95\% CI] \\
\midrule\endhead
\bottomrule\endfoot
Full & 63.76/63.82 & +0.06 [-0.10, 0.23] & 5.98/5.83 & -0.16 [-0.33, 0.02] \\
Eligible & 63.76/63.68 & -0.08 [-0.37, 0.21] & 5.98/5.85 & -0.14 [-0.35, 0.10] \\
\end{longtable}
\endgroup
\begingroup\setlength{\tabcolsep}{3pt}
\begin{longtable}{@{}lrrrrr@{}}
\caption{LRR: secondary universal rules evaluated on matched test populations.}\label{tab:raid-universal-lrr}\\
\toprule
Test population & Raw & Full & Change [95\% CI] & Eligible & Change [95\% CI] \\
\midrule\endfirsthead
\multicolumn{6}{l}{\textit{Continued}}\\
\toprule
Test population & Raw & Full & Change [95\% CI] & Eligible & Change [95\% CI] \\
\midrule\endhead
\bottomrule\endfoot
All attacks & 59.29 & 62.84 & \shortstack{+3.55\\{}[3.25, 3.87]} & 62.09 & \shortstack{+2.79\\{}[2.43, 3.16]} \\[3pt]
\(0<\rho\leq0.5\) & 55.56 & 57.68 & \shortstack{+2.12\\{}[1.76, 2.49]} & 57.65 & \shortstack{+2.09\\{}[1.66, 2.53]} \\[3pt]
\((0,0.05]\) & 59.51 & 61.85 & \shortstack{+2.35\\{}[1.83, 2.83]} & 61.79 & \shortstack{+2.28\\{}[1.72, 2.85]} \\[3pt]
\((0.05,0.1]\) & 64.48 & 68.11 & \shortstack{+3.63\\{}[2.87, 4.39]} & 68.69 & \shortstack{+4.21\\{}[3.36, 4.98]} \\[3pt]
\((0.1,0.2]\) & 47.60 & 48.33 & \shortstack{+0.73\\{}[0.14, 1.34]} & 47.86 & \shortstack{+0.27\\{}[-0.36, 0.90]} \\[3pt]
\((0.2,0.5]\) & 37.36 & 39.03 & \shortstack{+1.66\\{}[0.80, 2.52]} & 39.28 & \shortstack{+1.92\\{}[1.04, 2.79]} \\[3pt]
\end{longtable}
\endgroup
\begingroup\setlength{\tabcolsep}{3pt}
\begin{longtable}{@{}lrrrr@{}}
\caption{LRR: clean TPR and human FPR for the two universal fits (Raw/Clip).}\label{tab:raid-universal-cost-lrr}\\
\toprule
Fit & Clean TPR & Change [95\% CI] & Human FPR & Change [95\% CI] \\
\midrule\endfirsthead
\multicolumn{5}{l}{\textit{Continued}}\\
\toprule
Fit & Clean TPR & Change [95\% CI] & Human FPR & Change [95\% CI] \\
\midrule\endhead
\bottomrule\endfoot
Full & 66.05/70.91 & +4.86 [4.19, 5.48] & 5.55/5.63 & +0.08 [-0.37, 0.52] \\
Eligible & 66.05/69.86 & +3.81 [3.13, 4.54] & 5.55/5.54 & -0.02 [-0.54, 0.47] \\
\end{longtable}
\endgroup
\begingroup\setlength{\tabcolsep}{3pt}
\begin{longtable}{@{}lrrrrr@{}}
\caption{Entropy: secondary universal rules evaluated on matched test populations.}\label{tab:raid-universal-entropy}\\
\toprule
Test population & Raw & Full & Change [95\% CI] & Eligible & Change [95\% CI] \\
\midrule\endfirsthead
\multicolumn{6}{l}{\textit{Continued}}\\
\toprule
Test population & Raw & Full & Change [95\% CI] & Eligible & Change [95\% CI] \\
\midrule\endhead
\bottomrule\endfoot
All attacks & 37.52 & 37.68 & \shortstack{+0.16\\{}[0.00, 0.32]} & 37.91 & \shortstack{+0.39\\{}[0.29, 0.50]} \\[3pt]
\(0<\rho\leq0.5\) & 24.78 & 24.75 & \shortstack{-0.04\\{}[-0.25, 0.16]} & 25.17 & \shortstack{+0.38\\{}[0.26, 0.50]} \\[3pt]
\((0,0.05]\) & 26.11 & 26.20 & \shortstack{+0.09\\{}[-0.19, 0.35]} & 26.65 & \shortstack{+0.53\\{}[0.38, 0.69]} \\[3pt]
\((0.05,0.1]\) & 35.07 & 35.50 & \shortstack{+0.43\\{}[0.00, 0.87]} & 35.73 & \shortstack{+0.66\\{}[0.37, 0.97]} \\[3pt]
\((0.1,0.2]\) & 18.50 & 17.82 & \shortstack{-0.67\\{}[-1.01, -0.35]} & 18.37 & \shortstack{-0.13\\{}[-0.33, 0.10]} \\[3pt]
\((0.2,0.5]\) & 16.28 & 16.19 & \shortstack{-0.09\\{}[-0.44, 0.26]} & 16.48 & \shortstack{+0.20\\{}[-0.06, 0.48]} \\[3pt]
\end{longtable}
\endgroup
\begingroup\setlength{\tabcolsep}{3pt}
\begin{longtable}{@{}lrrrr@{}}
\caption{Entropy: clean TPR and human FPR for the two universal fits (Raw/Clip).}\label{tab:raid-universal-cost-entropy}\\
\toprule
Fit & Clean TPR & Change [95\% CI] & Human FPR & Change [95\% CI] \\
\midrule\endfirsthead
\multicolumn{5}{l}{\textit{Continued}}\\
\toprule
Fit & Clean TPR & Change [95\% CI] & Human FPR & Change [95\% CI] \\
\midrule\endhead
\bottomrule\endfoot
Full & 34.65/35.42 & +0.78 [0.33, 1.22] & 5.75/5.79 & +0.04 [-0.23, 0.31] \\
Eligible & 34.65/35.60 & +0.95 [0.66, 1.26] & 5.75/5.96 & +0.21 [0.00, 0.43] \\
\end{longtable}
\endgroup
\begingroup\setlength{\tabcolsep}{3pt}
\begin{longtable}{@{}lrrrrr@{}}
\caption{Entropy gap: secondary universal rules evaluated on matched test populations.}\label{tab:raid-universal-entropy-gap}\\
\toprule
Test population & Raw & Full & Change [95\% CI] & Eligible & Change [95\% CI] \\
\midrule\endfirsthead
\multicolumn{6}{l}{\textit{Continued}}\\
\toprule
Test population & Raw & Full & Change [95\% CI] & Eligible & Change [95\% CI] \\
\midrule\endhead
\bottomrule\endfoot
All attacks & 60.31 & 60.31 & \shortstack{+0.00\\{}[0.00, 0.00]} & 59.79 & \shortstack{-0.52\\{}[-0.68, -0.35]} \\[3pt]
\(0<\rho\leq0.5\) & 61.04 & 61.04 & \shortstack{+0.00\\{}[0.00, 0.00]} & 62.23 & \shortstack{+1.19\\{}[1.01, 1.37]} \\[3pt]
\((0,0.05]\) & 66.55 & 66.55 & \shortstack{+0.00\\{}[0.00, 0.00]} & 68.15 & \shortstack{+1.60\\{}[1.37, 1.86]} \\[3pt]
\((0.05,0.1]\) & 71.64 & 71.64 & \shortstack{+0.00\\{}[0.00, 0.00]} & 72.30 & \shortstack{+0.66\\{}[0.19, 1.13]} \\[3pt]
\((0.1,0.2]\) & 46.77 & 46.77 & \shortstack{+0.00\\{}[0.00, 0.00]} & 47.44 & \shortstack{+0.67\\{}[0.31, 1.03]} \\[3pt]
\((0.2,0.5]\) & 44.58 & 44.58 & \shortstack{+0.00\\{}[0.00, 0.00]} & 45.24 & \shortstack{+0.66\\{}[0.20, 1.14]} \\[3pt]
\end{longtable}
\endgroup
\begingroup\setlength{\tabcolsep}{3pt}
\begin{longtable}{@{}lrrrr@{}}
\caption{Entropy gap: clean TPR and human FPR for the two universal fits (Raw/Clip).}\label{tab:raid-universal-cost-entropy-gap}\\
\toprule
Fit & Clean TPR & Change [95\% CI] & Human FPR & Change [95\% CI] \\
\midrule\endfirsthead
\multicolumn{5}{l}{\textit{Continued}}\\
\toprule
Fit & Clean TPR & Change [95\% CI] & Human FPR & Change [95\% CI] \\
\midrule\endhead
\bottomrule\endfoot
Full & 75.06/75.06 & +0.00 [0.00, 0.00] & 5.07/5.07 & +0.00 [0.00, 0.00] \\
Eligible & 75.06/74.93 & -0.14 [-0.47, 0.17] & 5.07/5.09 & +0.02 [-0.23, 0.27] \\
\end{longtable}
\endgroup
\begingroup\setlength{\tabcolsep}{3pt}
\begin{longtable}{@{}lrrrrr@{}}
\caption{Binoculars: secondary universal rules evaluated on matched test populations.}\label{tab:raid-universal-binocular-origin}\\
\toprule
Test population & Raw & Full & Change [95\% CI] & Eligible & Change [95\% CI] \\
\midrule\endfirsthead
\multicolumn{6}{l}{\textit{Continued}}\\
\toprule
Test population & Raw & Full & Change [95\% CI] & Eligible & Change [95\% CI] \\
\midrule\endhead
\bottomrule\endfoot
All attacks & 72.15 & 72.32 & \shortstack{+0.17\\{}[0.09, 0.25]} & 71.49 & \shortstack{-0.66\\{}[-0.82, -0.49]} \\[3pt]
\(0<\rho\leq0.5\) & 71.06 & 71.38 & \shortstack{+0.32\\{}[0.22, 0.41]} & 71.29 & \shortstack{+0.22\\{}[0.04, 0.42]} \\[3pt]
\((0,0.05]\) & 74.27 & 74.58 & \shortstack{+0.31\\{}[0.19, 0.43]} & 74.35 & \shortstack{+0.08\\{}[-0.15, 0.30]} \\[3pt]
\((0.05,0.1]\) & 82.79 & 83.04 & \shortstack{+0.26\\{}[0.04, 0.47]} & 83.06 & \shortstack{+0.28\\{}[-0.13, 0.66]} \\[3pt]
\((0.1,0.2]\) & 60.88 & 61.30 & \shortstack{+0.42\\{}[0.20, 0.65]} & 61.75 & \shortstack{+0.87\\{}[0.49, 1.26]} \\[3pt]
\((0.2,0.5]\) & 57.88 & 58.14 & \shortstack{+0.26\\{}[-0.03, 0.56]} & 57.56 & \shortstack{-0.32\\{}[-0.83, 0.18]} \\[3pt]
\end{longtable}
\endgroup
\begingroup\setlength{\tabcolsep}{3pt}
\begin{longtable}{@{}lrrrr@{}}
\caption{Binoculars: clean TPR and human FPR for the two universal fits (Raw/Clip).}\label{tab:raid-universal-cost-binocular-origin}\\
\toprule
Fit & Clean TPR & Change [95\% CI] & Human FPR & Change [95\% CI] \\
\midrule\endfirsthead
\multicolumn{5}{l}{\textit{Continued}}\\
\toprule
Fit & Clean TPR & Change [95\% CI] & Human FPR & Change [95\% CI] \\
\midrule\endhead
\bottomrule\endfoot
Full & 80.00/80.09 & +0.10 [-0.10, 0.29] & 4.62/4.64 & +0.02 [-0.16, 0.19] \\
Eligible & 80.00/79.72 & -0.27 [-0.62, 0.04] & 4.62/4.66 & +0.04 [-0.25, 0.33] \\
\end{longtable}
\endgroup

\subsection{Selected directions and clipping bounds}\label{app:raid-bounds}
The tables list all 119 fitted specifications.
A floor acts on oriented additive token evidence.
The NLL cap bounds \(-\log p_i\); LRR also caps \(\log r_i\).
Dashes in all three bound columns mean no clipping.
Directions are shared across the fitting groups for a detector.
The displayed values are rounded to six significant digits; evaluation uses
the saved full-precision values. These are clipping bounds, not decision
thresholds. The tuning-text column counts the attacked texts used in the
objective; every fit also uses the clean tuning references.

\begingroup\setlength{\tabcolsep}{3pt}
\begin{longtable}{@{}lrrrr@{}}
\caption{Log likelihood: selected bounds; direction \(d=+1\).}\label{tab:raid-bounds-log-likelihood}\\
\toprule
Fitting group & Tuning texts & Floor & NLL cap & Log-rank cap \\
\midrule\endfirsthead
\multicolumn{5}{l}{\textit{Continued}}\\
\toprule
Fitting group & Tuning texts & Floor & NLL cap & Log-rank cap \\
\midrule\endhead
\bottomrule\endfoot
\((0,0.05]\) & 19989 & -11.1402 & -- & -- \\
\((0.05,0.1]\) & 4672 & -- & -- & -- \\
\((0.1,0.2]\) & 7169 & -- & -- & -- \\
\((0.2,0.5]\) & 3406 & -- & -- & -- \\
Alternative spelling & 5148 & -- & -- & -- \\
Article deletion & 5148 & -- & -- & -- \\
Homoglyph & 5148 & -- & -- & -- \\
Paragraph insertion & 5148 & -- & -- & -- \\
Number change & 5148 & -11.1402 & -- & -- \\
Paraphrase & 5148 & -- & -- & -- \\
Misspelling & 5148 & -11.1402 & -- & -- \\
Synonym & 5148 & -- & -- & -- \\
Case change & 5148 & -8.29309 & -- & -- \\
Whitespace & 5148 & -- & -- & -- \\
Zero-width space & 5148 & -- & -- & -- \\
Full universal & 56628 & -- & -- & -- \\
Eligible universal & 35236 & -11.1402 & -- & -- \\
\end{longtable}
\endgroup
\begingroup\setlength{\tabcolsep}{3pt}
\begin{longtable}{@{}lrrrr@{}}
\caption{Rank: selected bounds; direction \(d=-1\).}\label{tab:raid-bounds-rank}\\
\toprule
Fitting group & Tuning texts & Floor & NLL cap & Log-rank cap \\
\midrule\endfirsthead
\multicolumn{5}{l}{\textit{Continued}}\\
\toprule
Fitting group & Tuning texts & Floor & NLL cap & Log-rank cap \\
\midrule\endhead
\bottomrule\endfoot
\((0,0.05]\) & 19989 & -21 & -- & -- \\
\((0.05,0.1]\) & 4672 & -21 & -- & -- \\
\((0.1,0.2]\) & 7169 & -66 & -- & -- \\
\((0.2,0.5]\) & 3406 & -66 & -- & -- \\
Alternative spelling & 5148 & -66 & -- & -- \\
Article deletion & 5148 & -66 & -- & -- \\
Homoglyph & 5148 & -6 & -- & -- \\
Paragraph insertion & 5148 & -66 & -- & -- \\
Number change & 5148 & -66 & -- & -- \\
Paraphrase & 5148 & -66 & -- & -- \\
Misspelling & 5148 & -66 & -- & -- \\
Synonym & 5148 & -10 & -- & -- \\
Case change & 5148 & -10 & -- & -- \\
Whitespace & 5148 & -66 & -- & -- \\
Zero-width space & 5148 & -6 & -- & -- \\
Full universal & 56628 & -66 & -- & -- \\
Eligible universal & 35236 & -21 & -- & -- \\
\end{longtable}
\endgroup
\begingroup\setlength{\tabcolsep}{3pt}
\begin{longtable}{@{}lrrrr@{}}
\caption{Log rank: selected bounds; direction \(d=-1\).}\label{tab:raid-bounds-log-rank}\\
\toprule
Fitting group & Tuning texts & Floor & NLL cap & Log-rank cap \\
\midrule\endfirsthead
\multicolumn{5}{l}{\textit{Continued}}\\
\toprule
Fitting group & Tuning texts & Floor & NLL cap & Log-rank cap \\
\midrule\endhead
\bottomrule\endfoot
\((0,0.05]\) & 19989 & -5.2575 & -- & -- \\
\((0.05,0.1]\) & 4672 & -5.2575 & -- & -- \\
\((0.1,0.2]\) & 7169 & -6.49979 & -- & -- \\
\((0.2,0.5]\) & 3406 & -6.49979 & -- & -- \\
Alternative spelling & 5148 & -6.49979 & -- & -- \\
Article deletion & 5148 & -6.49979 & -- & -- \\
Homoglyph & 5148 & -7.27309 & -- & -- \\
Paragraph insertion & 5148 & -6.49979 & -- & -- \\
Number change & 5148 & -5.2575 & -- & -- \\
Paraphrase & 5148 & -7.27309 & -- & -- \\
Misspelling & 5148 & -6.49979 & -- & -- \\
Synonym & 5148 & -6.49979 & -- & -- \\
Case change & 5148 & -4.18965 & -- & -- \\
Whitespace & 5148 & -6.49979 & -- & -- \\
Zero-width space & 5148 & -5.2575 & -- & -- \\
Full universal & 56628 & -6.49979 & -- & -- \\
Eligible universal & 35236 & -5.2575 & -- & -- \\
\end{longtable}
\endgroup
\begingroup\setlength{\tabcolsep}{3pt}
\begin{longtable}{@{}lrrrr@{}}
\caption{LRR: selected bounds; direction \(d=+1\).}\label{tab:raid-bounds-lrr}\\
\toprule
Fitting group & Tuning texts & Floor & NLL cap & Log-rank cap \\
\midrule\endfirsthead
\multicolumn{5}{l}{\textit{Continued}}\\
\toprule
Fitting group & Tuning texts & Floor & NLL cap & Log-rank cap \\
\midrule\endhead
\bottomrule\endfoot
\((0,0.05]\) & 19989 & -- & 3.59968 & 3.04452 \\
\((0.05,0.1]\) & 4672 & -- & 3.59968 & 3.04452 \\
\((0.1,0.2]\) & 7169 & -- & 3.59968 & 4.18965 \\
\((0.2,0.5]\) & 3406 & -- & 3.59968 & 3.04452 \\
Alternative spelling & 5148 & -- & 3.59968 & 4.18965 \\
Article deletion & 5148 & -- & 3.59968 & 4.18965 \\
Homoglyph & 5148 & -- & 3.59968 & 6.49979 \\
Paragraph insertion & 5148 & -- & 3.59968 & 4.18965 \\
Number change & 5148 & -- & 4.33537 & 3.04452 \\
Paraphrase & 5148 & -- & 3.59968 & 4.18965 \\
Misspelling & 5148 & -- & 3.59968 & 4.18965 \\
Synonym & 5148 & -- & 3.59968 & 1.79176 \\
Case change & 5148 & -- & 3.59968 & 2.30259 \\
Whitespace & 5148 & -- & 3.59968 & 5.2575 \\
Zero-width space & 5148 & -- & 3.59968 & 5.2575 \\
Full universal & 56628 & -- & 3.59968 & 4.18965 \\
Eligible universal & 35236 & -- & 3.59968 & 3.04452 \\
\end{longtable}
\endgroup
\begingroup\setlength{\tabcolsep}{3pt}
\begin{longtable}{@{}lrrrr@{}}
\caption{Entropy: selected bounds; direction \(d=-1\).}\label{tab:raid-bounds-entropy}\\
\toprule
Fitting group & Tuning texts & Floor & NLL cap & Log-rank cap \\
\midrule\endfirsthead
\multicolumn{5}{l}{\textit{Continued}}\\
\toprule
Fitting group & Tuning texts & Floor & NLL cap & Log-rank cap \\
\midrule\endhead
\bottomrule\endfoot
\((0,0.05]\) & 19989 & -3.59632 & -- & -- \\
\((0.05,0.1]\) & 4672 & -3.59632 & -- & -- \\
\((0.1,0.2]\) & 7169 & -5.63195 & -- & -- \\
\((0.2,0.5]\) & 3406 & -4.41172 & -- & -- \\
Alternative spelling & 5148 & -3.59632 & -- & -- \\
Article deletion & 5148 & -3.59632 & -- & -- \\
Homoglyph & 5148 & -3.59632 & -- & -- \\
Paragraph insertion & 5148 & -3.59632 & -- & -- \\
Number change & 5148 & -3.59632 & -- & -- \\
Paraphrase & 5148 & -3.59632 & -- & -- \\
Misspelling & 5148 & -3.59632 & -- & -- \\
Synonym & 5148 & -3.59632 & -- & -- \\
Case change & 5148 & -3.96486 & -- & -- \\
Whitespace & 5148 & -6.26594 & -- & -- \\
Zero-width space & 5148 & -3.59632 & -- & -- \\
Full universal & 56628 & -3.59632 & -- & -- \\
Eligible universal & 35236 & -4.41172 & -- & -- \\
\end{longtable}
\endgroup
\begingroup\setlength{\tabcolsep}{3pt}
\begin{longtable}{@{}lrrrr@{}}
\caption{Entropy gap: selected bounds; direction \(d=-1\).}\label{tab:raid-bounds-entropy-gap}\\
\toprule
Fitting group & Tuning texts & Floor & NLL cap & Log-rank cap \\
\midrule\endfirsthead
\multicolumn{5}{l}{\textit{Continued}}\\
\toprule
Fitting group & Tuning texts & Floor & NLL cap & Log-rank cap \\
\midrule\endhead
\bottomrule\endfoot
\((0,0.05]\) & 19989 & -6.7337 & -- & -- \\
\((0.05,0.1]\) & 4672 & -8.05884 & -- & -- \\
\((0.1,0.2]\) & 7169 & -4.91388 & -- & -- \\
\((0.2,0.5]\) & 3406 & -4.91388 & -- & -- \\
Alternative spelling & 5148 & -8.05884 & -- & -- \\
Article deletion & 5148 & -- & -- & -- \\
Homoglyph & 5148 & -- & -- & -- \\
Paragraph insertion & 5148 & -- & -- & -- \\
Number change & 5148 & -6.7337 & -- & -- \\
Paraphrase & 5148 & -- & -- & -- \\
Misspelling & 5148 & -8.05884 & -- & -- \\
Synonym & 5148 & -3.47183 & -- & -- \\
Case change & 5148 & -4.91388 & -- & -- \\
Whitespace & 5148 & -8.05884 & -- & -- \\
Zero-width space & 5148 & -- & -- & -- \\
Full universal & 56628 & -- & -- & -- \\
Eligible universal & 35236 & -6.7337 & -- & -- \\
\end{longtable}
\endgroup
\begingroup\setlength{\tabcolsep}{3pt}
\begin{longtable}{@{}lrrrr@{}}
\caption{Binoculars: selected bounds; direction \(d=-1\).}\label{tab:raid-bounds-binocular-origin}\\
\toprule
Fitting group & Tuning texts & Floor & NLL cap & Log-rank cap \\
\midrule\endfirsthead
\multicolumn{5}{l}{\textit{Continued}}\\
\toprule
Fitting group & Tuning texts & Floor & NLL cap & Log-rank cap \\
\midrule\endhead
\bottomrule\endfoot
\((0,0.05]\) & 19989 & -- & 10.25 & -- \\
\((0.05,0.1]\) & 4672 & -- & 8.625 & -- \\
\((0.1,0.2]\) & 7169 & -- & 8.625 & -- \\
\((0.2,0.5]\) & 3406 & -- & -- & -- \\
Alternative spelling & 5148 & -- & 11.3125 & -- \\
Article deletion & 5148 & -- & 10.25 & -- \\
Homoglyph & 5148 & -- & -- & -- \\
Paragraph insertion & 5148 & -- & 11.3125 & -- \\
Number change & 5148 & -- & 11.3125 & -- \\
Paraphrase & 5148 & -- & -- & -- \\
Misspelling & 5148 & -- & 11.3125 & -- \\
Synonym & 5148 & -- & 3.96875 & -- \\
Case change & 5148 & -- & 8.625 & -- \\
Whitespace & 5148 & -- & 10.25 & -- \\
Zero-width space & 5148 & -- & -- & -- \\
Full universal & 56628 & -- & 11.3125 & -- \\
Eligible universal & 35236 & -- & 8.625 & -- \\
\end{longtable}
\endgroup

\subsection{Modification rates and text truncation}
The first table summarizes the edit rates for each attack. Q1--Q3 gives the middle half of the observed rates. The last three columns show the percentages with no measured change, an eligible rate \(0<\rho\leq0.5\), or a dense rate \(\rho>0.5\). Each condition contains 5,149 test texts.

\Needspace{20\baselineskip}
\begingroup
\begin{longtable}{lrrrrr}
\caption{Modification rates by attack, measured on 5,149 test texts per condition.}\label{tab:raid-rho}\\
\toprule
Condition & Median & Q1--Q3 & Zero (\%) & Eligible (\%) & Dense (\%) \\
\midrule\endfirsthead
\multicolumn{6}{l}{\textit{Continued from previous page}}\\
\toprule
Condition & Median & Q1--Q3 & Zero (\%) & Eligible (\%) & Dense (\%) \\
\midrule\endhead
\bottomrule\endfoot
None & 0.000 & 0.000--0.000 & 100.0 & 0.0 & 0.0 \\
Alternative spelling & 0.006 & 0.000--0.015 & 36.1 & 63.9 & 0.0 \\
Article deletion & 0.027 & 0.015--0.037 & 10.4 & 89.6 & 0.0 \\
Homoglyph & 0.963 & 0.933--0.977 & 0.0 & 0.0 & 100.0 \\
Paragraph insertion & 0.051 & 0.034--0.068 & 6.9 & 93.1 & 0.0 \\
Number & 0.000 & 0.000--0.013 & 54.0 & 46.0 & 0.0 \\
Paraphrase & 0.613 & 0.523--0.708 & 0.0 & 20.2 & 79.8 \\
Misspelling & 0.012 & 0.000--0.024 & 28.5 & 71.5 & 0.0 \\
Synonym & 0.183 & 0.144--0.219 & 0.1 & 99.9 & 0.0 \\
Case change & 0.040 & 0.035--0.045 & 0.3 & 99.7 & 0.0 \\
Whitespace & 0.124 & 0.110--0.137 & 0.1 & 99.9 & 0.0 \\
Zero-width space & 0.984 & 0.975--0.990 & 0.0 & 0.0 & 100.0 \\
\end{longtable}
\endgroup

Homoglyph and zero-width-space attacks change many Falcon tokens even when the text looks similar. Their median edit rates are 0.963 and 0.984. The 512-token scoring limit truncates 87.3\% of homoglyph outputs and 99.4\% of zero-width-space outputs, compared with 5.2\% of clean outputs. Their full-token edit distance exceeds the scored-window distance in 87.1\% and 99.4\% of cases, respectively.

\subsection{Sample composition}

The following tables show how the 12,871 selected sources are distributed across generators, decoding methods, and domains in each split. Each source is counted once, using its selected machine output. Generators and decoding methods have different sample sizes because family selection was performed within each source. The selected generations comprise 6,507 greedy outputs and 6,364 sampled outputs.

\Needspace{20\baselineskip}
\begingroup
\begin{longtable}{lrrrr}
\caption{Observed generator identifier: one selected machine family per human-source group.}\label{tab:raid-composition-generator}\\
\toprule
Observed generator identifier & Tuning & Calibration & Test & Total \\
\midrule\endfirsthead
\multicolumn{5}{l}{\textit{Continued from previous page}}\\
\toprule
Observed generator identifier & Tuning & Calibration & Test & Total \\
\midrule\endhead
\bottomrule\endfoot
chatgpt & 337 & 150 & 320 & 807 \\
cohere & 315 & 146 & 292 & 753 \\
cohere-chat & 284 & 147 & 285 & 716 \\
gpt2 & 610 & 268 & 593 & 1,471 \\
gpt3 & 276 & 162 & 317 & 755 \\
gpt4 & 294 & 169 & 329 & 792 \\
llama-chat & 637 & 303 & 597 & 1,537 \\
mistral & 602 & 326 & 600 & 1,528 \\
mistral-chat & 597 & 308 & 601 & 1,506 \\
mpt & 586 & 313 & 617 & 1,516 \\
mpt-chat & 610 & 282 & 598 & 1,490 \\
\textbf{Total} & 5,148 & 2,574 & 5,149 & 12,871 \\
\end{longtable}
\endgroup

\Needspace{17\baselineskip}
\begingroup
\begin{longtable}{lrrrr}
\caption{Source domain: one selected machine family per human-source group.}\label{tab:raid-composition-domain}\\
\toprule
Source domain & Tuning & Calibration & Test & Total \\
\midrule\endfirsthead
\multicolumn{5}{l}{\textit{Continued from previous page}}\\
\toprule
Source domain & Tuning & Calibration & Test & Total \\
\midrule\endhead
\bottomrule\endfoot
abstracts & 681 & 341 & 681 & 1,703 \\
books & 687 & 344 & 687 & 1,718 \\
news & 687 & 343 & 687 & 1,717 \\
poetry & 683 & 342 & 683 & 1,708 \\
recipes & 684 & 342 & 684 & 1,710 \\
reddit & 687 & 343 & 687 & 1,717 \\
reviews & 352 & 176 & 353 & 881 \\
wiki & 687 & 343 & 687 & 1,717 \\
\textbf{Total} & 5,148 & 2,574 & 5,149 & 12,871 \\
\end{longtable}
\endgroup

\end{document}